\documentclass[twoside,11pt]{article}
\usepackage{jmlr2e}

\usepackage{amsmath,verbatim, paralist}
\usepackage{multirow}
\usepackage{fancyhdr,graphicx,color}
\usepackage[utf8]{inputenc}
\usepackage{graphicx}
\usepackage{graphics}

\usepackage[dvipsnames]{xcolor}

\usepackage{appendix}
\usepackage[ruled,vlined]{algorithm2e}

\newtheorem{lem}{Lemma}
\newtheorem{thm}{Theorem}
\newtheorem{pro}{Proposition}

\newtheorem{dfn}{Definition}
\newtheorem{exa}{Example}
\newtheorem{rmk}{Remark}

\newcommand{\argmin}{\mathop{\mathrm{argmin}}}

\newcommand{\1}{{\rm 1}\kern-0.24em{\rm I}}

\newcommand{\B}{\boldsymbol}
\newcommand{\R}{\mathbb{R}}

\newif\ifworkinprogress
\workinprogresstrue

\ShortHeadings{Sparse NMF with Archetypal regularization}{Behdin and Mazumder}
\firstpageno{1}

\usepackage{lastpage}

\begin{document}

\jmlrheading{25}{2024}{1-\pageref{LastPage}}{3/21; Revised
11/23}{1/24}{21-0233}{Kayhan Behdin and Rahul Mazumder}
\ShortHeadings{Sparse NMF with Archetypal Regularization}{Behdin and Mazumder}

\title{Sparse NMF with Archetypal Regularization: Computational and Robustness Properties }

\author{\name Kayhan Behdin \email behdink@mit.edu \\
\addr Operations Research Center \\
Massachusetts Institute of Technology \\
Cambridge, MA 02139, USA
\AND \name Rahul Mazumder \email rahulmaz@mit.edu  \\
\addr MIT Sloan School of Management and Operations Research Center \\
Massachusetts Institute of Technology \\
Cambridge, MA 02139, USA}

\editor{John Cunningham}

\maketitle

\begin{abstract}%
    We  consider  the  problem  of  sparse  nonnegative matrix factorization (NMF)  using  archetypal  regularization. The goal is to represent a collection of data points as nonnegative linear combinations of  a  few  nonnegative  sparse  factors with appealing geometric properties, arising from the use of archetypal regularization.  We generalize the notion of robustness studied in  Javadi and Montanari (2019) (without sparsity) to the notions of (a) strong robustness that implies each estimated archetype is close to the underlying archetypes and (b) weak robustness that  implies  there  exists  at  least  one  recovered  archetype  that  is  close  to  the  underlying archetypes.  Our theoretical results on robustness guarantees hold under minimal  assumptions on the underlying data, and applies to settings where the underlying archetypes need not be sparse. We present theoretical results and illustrative examples to strengthen the insights underlying the notions of robustness. 
    We propose new algorithms for our optimization problem;
    and present
    numerical experiments on synthetic and real data sets that shed further insights into our proposed framework and theoretical developments.
\end{abstract}

\medskip


\section{Introduction}
Nonnegative Matrix Factorization (NMF) \citep{lee1999learning} is a well-known dimensionality reduction method where we represent a collection of data points as nonnegative linear combinations of a few nonnegative latent factors. Nonnegative factors are desirable from an interpretability standpoint in applications such as computational biology~\citep{kotliar2019identifying,brunet2004metagenes}, image processing \citep{kalayeh2014nmf,liu2011constrained}, text mining \citep{berry2005email}, and chemometrics \citep{lawton1971self}, among others.  Mathematically, given a $m \times n$ data matrix with nonnegative entries $\boldsymbol{X}\in\mathbb{R}_{\geq0}^{m\times n}$ (the rows of $\B{X}$ correspond to the $m$ samples and the columns the dimensions), NMF computes nonnegative lower-dimensional latent factors $\boldsymbol{W}\in\mathbb{R}_{\geq0}^{m\times k},\boldsymbol{H}\in\mathbb{R}_{\geq0}^{k\times n}$. Here, $k$ denotes the number of latent factors with $k<m,n$ and we desire the factors to lead to a good approximation of the underlying data matrix: $\boldsymbol{X}\approx \boldsymbol{WH}$, where rows of $\boldsymbol{H}$ are representatives of the data and rows of $\boldsymbol{W}$ denote the coefficient weights. NMF can be formulated \citep{6795860} as the following nonconvex optimization problem
\begin{align}\label{basic1}
     \min_{\boldsymbol{H},\B{W}} \quad & \|\B{X}-\B{WH}\|_F^2 ~~~\text{s.t.}~~~\boldsymbol{H}\in\R_{\geq 0}^{k\times n},\boldsymbol{W}\in\R_{\geq 0}^{m\times k}
\end{align}
where, $\| \cdot \|_{F}$ denotes the Frobenius norm of a matrix. The NMF problem is agnostic to scaling: multiplying $\B W$ with a positive scalar $a>0$ and $\B H$ with $a^{-1}$ does not change the objective and feasible set. Therefore, to remove such ambiguity, one needs to impose some form of regularization.

\smallskip 

\noindent \textbf{Archetypal Analysis.} Another classic dimension reduction method that is closely related to NMF is Archetypal Analysis (AA) due to \citet{Breiman}.  AA seeks to find an exact representation of the form $\B{X}=\B{WH}$ for the data. However, in contrast to NMF, representatives of the data (aka the {\emph {archetypes}}), given by the rows of $\boldsymbol{H}$, have an appealing geometric interpretation.  The archetypes are estimated so that they belong to the convex hull of the data and the data is contained within their convex hull. Such constraints on the archetypes can help make the factorization of the data unique, resulting in more interpretable solutions. In practice however, it may not be possible to find such archetypes due to exact factorization requirement, especially in applications pertaining to NMF~\citep{javadi2019nonnegative}. 
To this end, \citet{javadi2019nonnegative}, drawing ideas from AA, propose a regularization scheme for NMF: Among all possible sets of candidate representatives that contain the data in their convex hull with acceptable accuracy (i.e. sufficiently high data fidelity), we select the one that is the closest to the convex hull of the data. Figure~\ref{fig:0} (a), (b) illustrate the exact AA of \citet{Breiman} and regularized AA with a toy example\footnote{See Appendix \ref{toyexa} for details of the example.}. As shown by~\citet{javadi2019nonnegative}, the archetypal-regularized NMF framework leads to good performance in practice. Therefore, we pursue a deeper investigation of this regularization framework in this paper. For simplicity, we will use the term ``AA'' to refer to the framework of~\citet{javadi2019nonnegative}.
When it is not clear from the context, we use the term ``exact AA'' to refer to the method of~\citet{Breiman}.

\smallskip

\noindent \textbf{Robustness.} If $\B{X}$ admits an exact factorization of the form $\B{X}=\B{W}\B{H}$ for some $\B{W}\in\R^{m\times k}_{\geq 0}$ and $\B{H}\in\R^{k\times n}_{\geq 0}$, \citet{NIPS2003_2463} show that under the so-called {\emph{separability}} assumption\footnote{The factorization $\B{X}=\B{W}\B{H}$ is called separable if rows of $\B{H}$ are a subset of rows of $\B{X}$.}, it is possible to recover $\B{W}$ and $\B{H}$ from $\B{X}$ (up to permutation and scaling of rows of $\B{H}$/columns of $\B{W}$). The separability assumption has been consequently generalized to less restrictive cases, including noisy settings. In particular, \citet{arora2016computing} consider an approximately separable model, where the data is assumed to be a noisy version of a separable data set. They show that in this model, under additional regularity conditions, a polynomial-time algorithm exists that finds a factorization that is close to the factorization of the noiseless data in a suitable metric. Their results are further improved by \citet{NIPS2012_4518}, showing that noisy separable NMF can be solved by linear programs. \citet{pmlr-v37-geb15} show that a relaxed version of separability, the so-called \emph{subset separability} condition, suffices to achieve a good factorization from noisy data. Note that the separability/subset separability assumptions are usually  hard to verify on real data; and in our development we do not make use of this assumption. \citet{javadi2019nonnegative} show that AA enjoys {\em robustness} to perturbation: under certain assumptions, the resulting solution from AA on the perturbed data is close to the underlying model in a suitable metric (as discussed in Section \ref{robussec}). In particular, this implies that at least one of the recovered archetypes is close to the underlying archetypes that contain and represent the noiseless data. In this paper, we generalize this notion to a stronger version of robustness which implies that \emph{each} recovered archetype is close to some true archetype. 
 \\

\noindent Some earlier works have considered NMF formulations robust to outliers~\citep{chen2014fast,kong2011robust}. In this paper, we consider a different notion of robustness, unrelated to outliers that arise in the context of robust statistics~\citep{huber2004robust}.

\smallskip

\noindent \textbf{Sparsity.} Generally speaking, due to nonnegativity constraints on the latent factors, NMF is known to produce sparse solutions, that is, $\boldsymbol{H}$ and $\boldsymbol{W}$ have some zero entries~\citep{5438836}. A sparse representation of the data aids in interpretability and requires less storage space. This property of NMF has been utilized in different applications such as image processing \citep{hoyer2004non}, computational biology \citep{10.1093/bioinformatics/btm134}, medical imaging \citep{woo2018sparse}, document clustering \citep{kim2008sparse} and audio processing \citep{4100700}. Several papers have proposed formulations of NMF with additional penalties and/or constraints to encourage enhanced sparsity in NMF---we refer to these as sparse NMF methods. Specifically, \citet{hoyer2004non} consider a sparse NMF problem where they use a combination of the
$\ell_1$ and $\ell_2$ penalty on the entries of $\B{H}$.
\citet{peharz2012sparse} add a constraint of the form $\|\B{H}\|_0\leq \ell$ to problem \eqref{basic1}, where $\|\B{H}\|_0$ is the $\ell_0$-pseudonorm, the number of nonzero entries of $\B{H}$ and $\ell$ is the desired sparsity level.  \citet{kim2008sparse,10.1093/bioinformatics/btm134} add an $\ell_1$ norm penalty on the entries of $\B{H}$ to the cost function of \eqref{basic1} to impose sparsity on $\B{H}$. To the best of our knowledge, there is limited theoretical work on sparsity in NMF---in particular, towards understanding the effect of sparsity constraints on the robustness of the representation returned by NMF, an aspect we study here.
In this paper, we present a simultaneous analysis of sparsity and archetypal regularization in the form of Sparse AA (SAA). We study regularized AA~\citep{javadi2019nonnegative} in the presence of additional sparsity constraints on $\B{H}$. In other words, we look for $\boldsymbol{H}$ such that it has a few nonzero entries (i.e., $\|\B{H}\|_0$ is small), its rows describe the data well and are close to the convex hull of data. See Figure~\ref{fig:0} (c) for a numerical illustration of this problem.
In particular, we show that sparsity constrained AA leads to robust solutions, both in the weak sense of \citet{javadi2019nonnegative} and the stronger notion of robustness proposed here. An important feature of our analysis is that we do not assume the underlying archetypes are sparse---i.e., we can handle model misspecification---this makes our proofs different from existing work. We also discuss how noise and sparsity affect the robustness properties of the model.

\smallskip

\smallskip

\noindent \textbf{Algorithms.} Due to the bilinearity of the mapping $(\B{W}, \B{H}) \mapsto \B{W}\B{H}$,  most formulations of NMF end up in a nonconvex optimization problem, although some convex formulations exist \citep{bach2008convex} when the dimension of the latent factors grows to infinity. Some basic approaches to these nonconvex problems include projected gradient methods \citep{6795860}, multiplicative update rules \citep{lee1999learning,Gonzalez05acceleratingthe} and alternating optimization \citep{paatero1994positive,Chu04optimality}. More sophisticated algorithms for NMF have been proposed in recent years, for example see \citet{8682280,mizutani2014ellipsoidal,6656801}. 
In this paper we present algorithms to obtain good solutions for the regularized AA problem with sparsity constraints. To this end, we present proximal block coordinate methods, and establish that they lead to a stationary point. We discuss a useful initialization scheme based on Mixed Integer Programming (MIP)~\citep{nemhauser,bertsimas2016best} that leads to high-quality solutions. 
To further improve the quality of solutions available from our block coordinate procedure, we present local search based methods---to this end, our framework draws inspiration from the work of~\citet{beck,hazimeh2020fast} and adapts it to the setting of matrix factorization problems.
Note that~\citet{6247852,MORUP201254,abrolgeometric} have proposed algorithms for the original (exact) AA problem without sparsity constraints.
In addition, prior work has extended the notion of separability arising from NMF to address the exact AA problem: For example, \citet{damle2017geometric} use geometric interpretations of AA and separability to develop a new algorithm for NMF. Our numerical experiments on synthetic, and real  data sets validate our theoretical results, and suggest the superiority of SAA over other popular sparse NMF methods.

\smallskip

\noindent \textbf{Our Contributions.} Our contribution in this paper can be summarized as follows: 

\smallskip

\begin{compactitem}
    \item We generalize the robustness framework of \citet{javadi2019nonnegative} to notions of \emph{weak} and \emph{strong} robustness, introduced in this paper---these notions differ in how they describe the proximity of our estimators to the underlying archetypes. Furthermore, we prove robust solutions are good representatives of the noiseless data. 
    
    \smallskip
    
\item  We show how sparsity and AA can be used together to produce sparse factors that are robust to noise and perturbation in the data. Our results apply to the mis-specified setting--i.e., situations where the underlying archetypes are not necessarily sparse.

\smallskip

\item  We present algorithms\footnote{Implementation can be found at~\url{https://github.com/kayhanbehdin/SparseAA}.} based on block proximal descent and combinatorial local search, discuss initialization strategies based on mixed integer programming (MIP); and present convergence properties of our proposed algorithmic framework.  We demonstrate via numerical experiments on synthetic and real data sets the usefulness of our proposed approach.
\end{compactitem}

\smallskip

\noindent \textbf{Notation.} 
For a matrix $\boldsymbol{X}$, we let $\boldsymbol{X}_{i,j}$, $\boldsymbol{X}_{i,.}$ and $\boldsymbol{X}_{.,j}$ denote the $(i,j)$-th element, $i$-th row and $j$-th column of $\B{X}$, respectively. With a slight abuse of notation, for indexed variables, we move the subscript to superscript: that is, the $i$-th row of $\boldsymbol{X}_0$ is shown as $\boldsymbol{X}^0_{i,.}$. The number of rows of a matrix $\boldsymbol{X}$ is denoted as $\text{\#row}(\boldsymbol{X})$. The convex hull of rows of the matrix $\boldsymbol{X}\in\mathbb{R}^{m\times n}$ is denoted by 
$$\text{Conv}(\boldsymbol{X})=\left\{\sum\limits_{i=1}^m \alpha_i \boldsymbol{X}_{i,.}: \alpha_i\geq 0, \sum\limits_{i=1}^m \alpha_i = 1\right\}.$$
We define (i) the distance between a vector $\boldsymbol{x}$ and the convex hull $\text{Conv}(\boldsymbol{X})$ as $$D(\boldsymbol{x},\boldsymbol{X})=\min_{\boldsymbol{v}\in\text{Conv}(\boldsymbol{X})}\|\boldsymbol{x}-\boldsymbol{v}\|_2^2$$
(ii) the distance between a set of points (i.e., the rows of $\boldsymbol{X}$) and $\text{Conv}(\boldsymbol{Y})$ as $$D(\boldsymbol{X},\boldsymbol{Y})=\sum_{i=1}^{\text{\#row}(\boldsymbol{X})}D(\boldsymbol{X}_{i,.},\boldsymbol{Y})~~~
\text{and}~~~D(\boldsymbol{X},\boldsymbol{Y})^{1/2}=\sqrt{D(\boldsymbol{X},\boldsymbol{Y})}.$$ 
The number of nonzero entries in a matrix $\boldsymbol{A}$ is denoted as $\|\boldsymbol{A}\|_0$. We define $P_{\ell}(\boldsymbol{H})$ to be the projection of $\boldsymbol{H}\in\mathbb{R}^{k\times n}$ onto the $\ell$-sparse set $\{\boldsymbol{X}\in\mathbb{R}^{k\times n}:\|\boldsymbol{X}\|_0\leq \ell\}.$ Moreover, we define the complement as $P_{\ell}^{\perp}(\boldsymbol{H})=\boldsymbol{H}-P_{\ell}(\boldsymbol{H})$. The support of a matrix $\B{H}\in\R^{k\times n}$, $S(\B{H})$, is defined as the set of its nonzero coordinates:
$$S(\B{H})=\{(i,j)\in[k]\times [n]: |\B{H}_{i,j}|>0\}.$$
We set $\B{E}^{i,j}\in\R^{k\times n}$ to be the matrix with coordinate $(i,j)$ equal to one and other coordinates equal to zero. Throughout this paper, we use $\sigma_{\min}(\boldsymbol{H})$ and $\sigma_{\max}(\boldsymbol{H})$ to denote the smallest and largest singular values of $\boldsymbol{H}$ (respectively). We let $\kappa(\boldsymbol{H}):=\sigma_{\max}(\B{H})/\sigma_{\min}(\B{H})$ denote the condition number of $\boldsymbol{H}$. For a convex and subdifferentiable function $f:\R^d\to \R$, $\partial f(\B{x})$ denotes the set of subgradients of $f$ at $\B{x}\in\R^d$. Proofs of main results have been relegated to the appendix to improve readability.

\begin{figure}[t!]
     \centering
\begin{tabular}{ccc}
\includegraphics[width=0.3\linewidth]{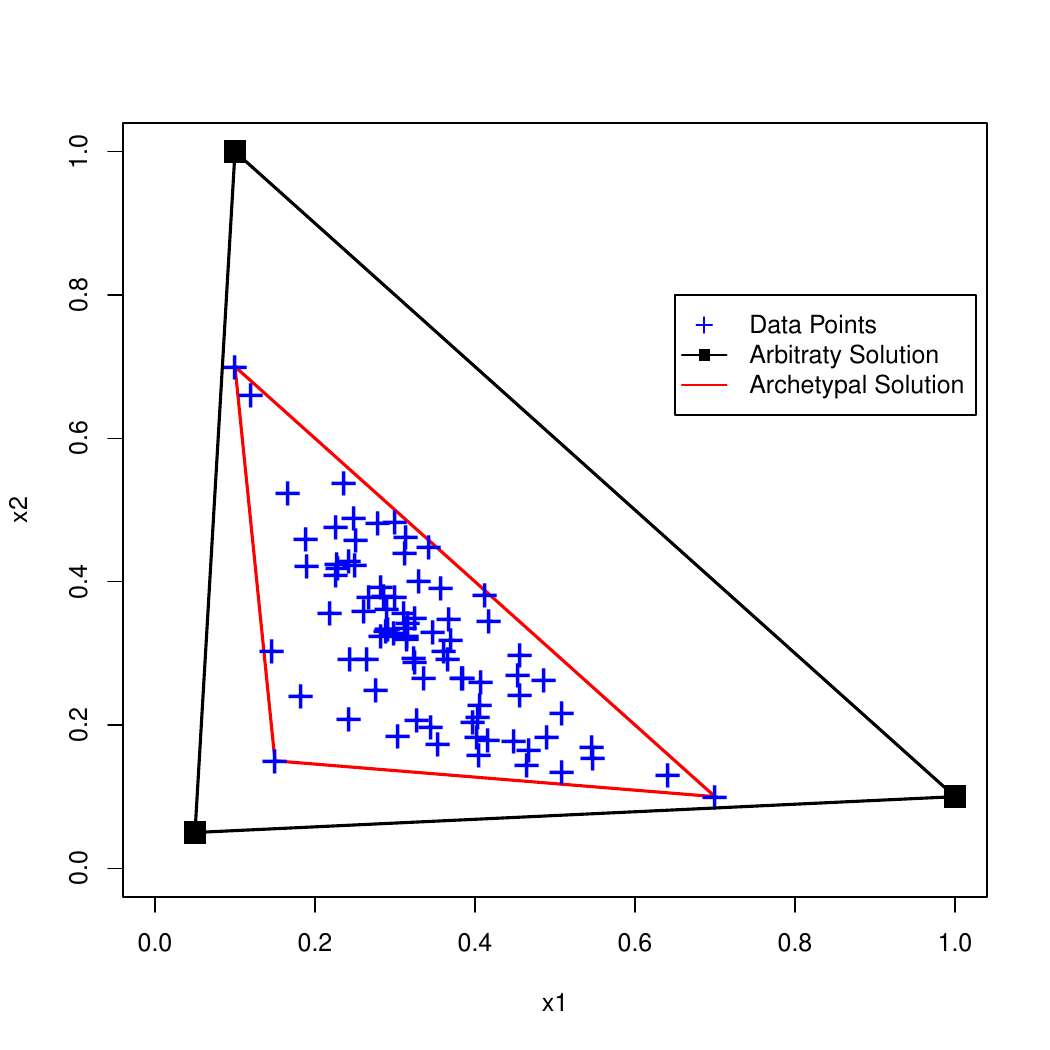}&
         \includegraphics[width=0.3\linewidth]{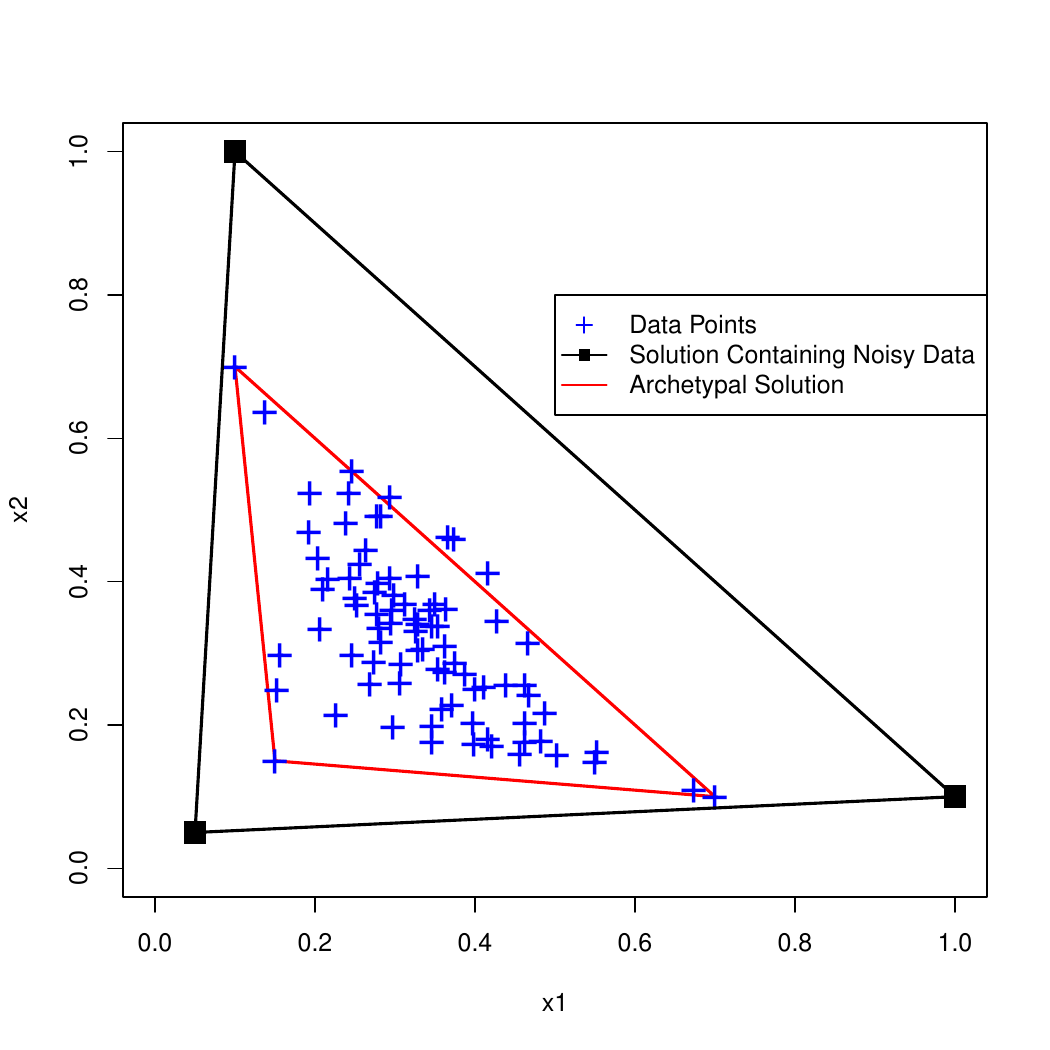}&
         \includegraphics[width=0.3\linewidth]{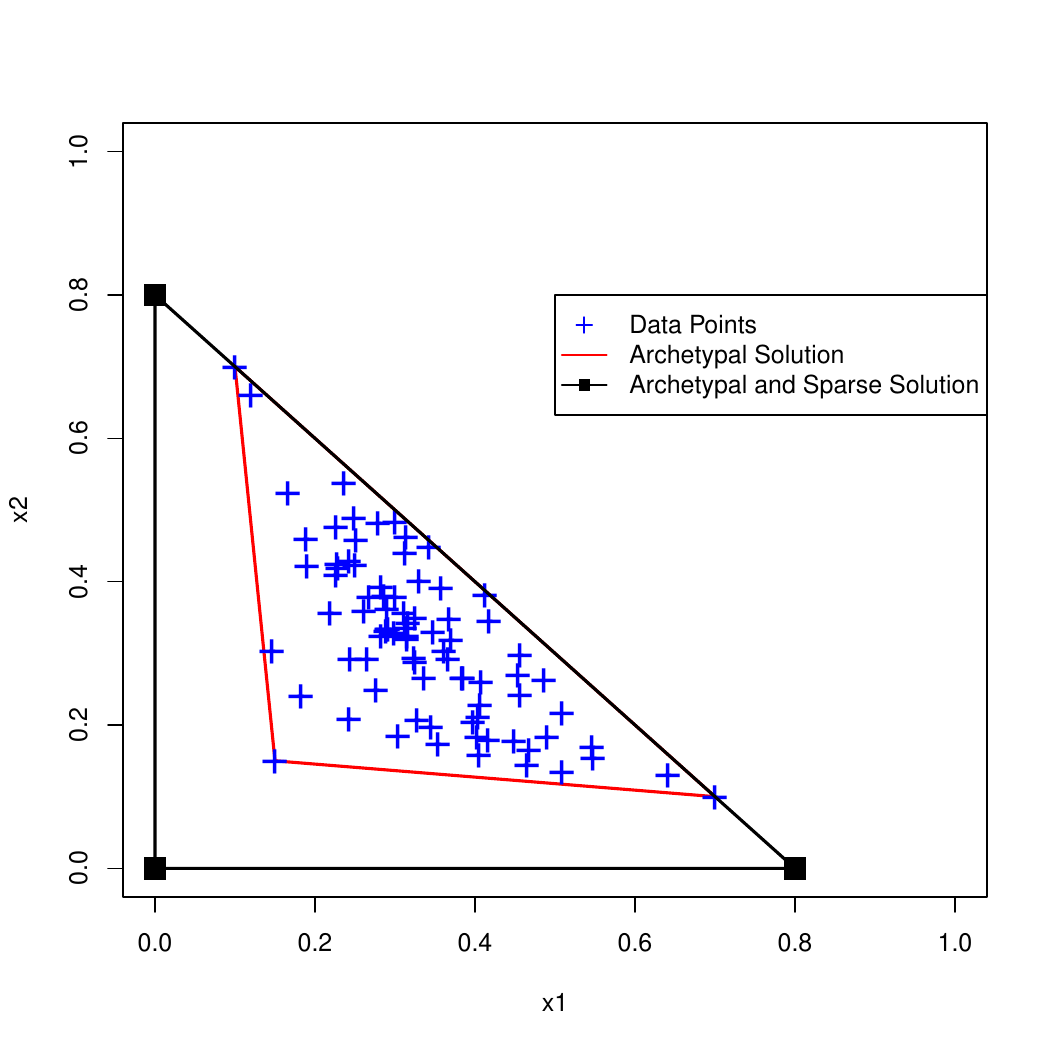} \\
         (a) & (b) & (c) 
\end{tabular}
        \caption{{\small In these figures, blue crosses (`+') represent the data points in $\mathbb{R}^2$. We seek to find 3 archetypes such that their convex hull contains the data. Panel (a): the black convex hull (triangle) shows an arbitrary solution to NMF, while the red convex hull shows the exact AA solution that is the smallest triangle containing the data. Panel (b): the black convex hull shows a solution that describes the data with no error, while it is not close to the convex hull of the data. The red convex hull shows the regularized AA solution, which both describes the data well (but with nonzero error) and is close to the data. Panel (c): the red convex hull shows the exact AA solution which does not have any zero coordinate, while the black convex hull shows a solution which is sparse and only has 2 nonzero coordinates. In addition, no other solution with the same sparsity can be found which is closer to the data.}}   \label{fig:0}
     \end{figure}

\section{Problem Formulation}

Given $m$ data points in $\mathbb{R}^n$, stacked along the rows of $\boldsymbol{X}\in\mathbb{R}^{m\times n}$, the goal of exact AA is to find $k$ archetypes  $\boldsymbol{H}_{1,.},\cdots,\boldsymbol{H}_{k,.}\in\mathbb{R}^{n}$ such that:  (i)~the rows of $\B{X}$ are contained in the convex hull of the rows of $\boldsymbol{H}$; and 
(ii)~the rows of $\boldsymbol{H}$ are themselves close to the convex hull of the rows of $\boldsymbol{X}$. 
Equivalently, we seek to learn $\boldsymbol{H}$ such that 
\begin{equation}\label{archa}
D(\boldsymbol{X},\boldsymbol{H})=D(\boldsymbol{H},\boldsymbol{X})=0
\end{equation}
where the first term ensures the data is described by $\boldsymbol{H}$ (as it implies there exists $\B{W}\in\mathbb{R}_{\geq 0}^{m\times k}$ such that $\B{X}=\B{WH}$ and each row of $\B{W}$ sums to one). The second term ensures that rows of $\boldsymbol{H}$ are in the convex hull of rows of $\B{X}$. 
As discussed earlier, the constraint (\ref{archa}) is too restrictive for most practical cases. Moreover, a general factorization of the form $\B{X}=\B{W}\B{H}$ is not unique due to multiplicative invariance. Therefore, to remove such ambiguity, offer some flexibility over the exact AA described above, \citet{javadi2019nonnegative} propose a relaxed version of (\ref{archa}) for the NMF problem. In this framework, among all the archetypes that contain the data with acceptable accuracy in their convex hull, they choose archetypes that are closest to the convex hull of data. We follow suit, but in addition, we impose a sparsity on $\B{H}$ via the constraint $\|\boldsymbol{H}\|_0\leq \ell$, where, $\ell$ is a budget on the number of nonzero entries.
Specifically, for a pre-specified value of $\alpha$, our proposed Sparse Archetypal Analysis (SAA) estimator considers:
\begin{align}\label{archl0noise}
     \hat{\boldsymbol{H}}\in
     \argmin_{\boldsymbol{H}\in\mathbb{R}_{\geq 0}^{k\times n}}~ D(\boldsymbol{H},\boldsymbol{X})~~~
     \text{s.t.}~~ D(\boldsymbol{X}_{i,.},\boldsymbol{H})^{1/2}\leq \alpha, i\in[m]; ~~ \|\boldsymbol{H}\|_0\leq \ell. 
\end{align}
Above the constraint ``$D(\boldsymbol{X}_{i,.},\boldsymbol{H})^{1/2}\leq \alpha, i\in[m]$" restricts the data points (rows of $\boldsymbol{X}$) to be close to the convex hull of the rows of $\boldsymbol{H}$. The archetypes (i.e., the right latent factors) are constrained to be sparse. The objective function in 
problem~\eqref{archl0noise} chooses the archetypes closest to the convex hull of data points among all feasible solutions. This serves to regularize the solution.
\begin{rmk}\label{rmk-feasible}
    We note that by taking the value of $\alpha$ to be too small in problem~\eqref{archl0noise}, problem~\eqref{archl0noise} can become infeasible. However, as we show in Theorem~\ref{robustnessthm}, under our model setup it is always possible find a value of $\alpha$ such that problem~\eqref{archl0noise} is feasible.
\end{rmk}

\subsection{Model setup}\label{modelsetting}
Suppose $\B{X}_0\in\R^{m\times n}_{\geq 0}$ with $\text{rank}(\B{X}_0)=k$ (where we assume $k$ is given a priori) admits a nonnegative factorization of rank $k$. That is, there exist $\bar{\B{W}}\in\R^{m\times k}_{\geq 0},\bar{\B{H}}\in\R^{k\times n}_{\geq 0}$ such that $\B{X}_0=\bar{\B{W}}\bar{\B{H}}$ and rows of $\bar{\B{W}}$ sum to one. This is equivalent to $D(\B{X}_0,\bar{\B{H}})=0$. However, such a  factorization is not generally unique. Hence, we let
\begin{align}\label{h0opt}
     \boldsymbol{H}_0\in\argmin_{\boldsymbol{H}\in\mathbb{R}_{\geq 0}^{k\times n}} ~~ D(\B{H},\boldsymbol{X}_0)~~\text{s.t.}~~ D(\boldsymbol{X}_0,\B{H})=0
\end{align}
and assume that $\B{H}_0$ is unique. The choice of $\B{H}_0$ in problem \eqref{h0opt} guarantees that $\B{X}_0$ has an exact factorization of the form $\B{X}_0=\B{W}_0\B{H}_0$ for some $\B{W}_0\in\R^{m\times k}_{\geq 0}$ such that its rows sum to one and $\B{H}_0$ is defined as in \eqref{h0opt}. Note that $\text{rank}(\B{H}_0)=k$---otherwise, if $\text{rank}(\B{H}_0)<k$, then $\text{rank}(\B{X}_0)<k$ which contradicts our assumption on $\B{X}_0$. 

We assume that the observed data matrix $\boldsymbol{X}$, is given by $\B{X}=\B{X}_0+\B{Z}$ where $\B{Z}$ is additive noise.  In what follows, we do not make any distributional assumptions on $\B{Z}$.

\begin{rmk}
Our analysis of robustness is goes through without the uniqueness assumption on $\B{H}_0$. We make the uniqueness assumption for simplicity of exposition. 
\end{rmk}
\begin{rmk}
Note that we do not assume that $\B{H}_0$ is sparse, though formulation~\eqref{archl0noise} imposes an explicit cardinality constraint on $\B{H}$. This model misspecification leads to
technical challenges: The analysis presented in \citet{javadi2019nonnegative} does not readily generalize to our setting, and 
we present a new analysis technique. 
\end{rmk}

\begin{rmk}
We also note that 
$$\B{X}_0=\B{W}_0\B{H}_0=\sum_{j=1}^k(\B{W}_0)_{:,k}(\B{H}_0)_{k,:}.$$
This implies 
$$\text{rank}(\B{X}_0)=\text{rank}_+(\B{X}_0)=k$$
where $\text{rank}_+$ denotes the nonnegative rank~\citep{cohen1993nonnegative}, defined as the minimum number of nonnegative rank-one matrices that sum up to $\B X_0$.
Therefore, nonnegative and ordinary ranks of $\B{X}_0$ are equal which might not be true in general. However, we note that this only concerns the noiseless data and the observed noisy data can be of any rank.
\end{rmk}

\subsection{Robustness to noise in Archetypal Analysis}\label{robussec}
We are interested to see if a solution $\hat{\B{H}}$ of (\ref{archl0noise}) is close to $\boldsymbol{H}_0$, the underlying set of archetypes. 
To this end, following \citet{javadi2019nonnegative}, we define a distance between two sets of archetypes $\boldsymbol{H}_1\in\mathbb{R}^{k_1\times n},\boldsymbol{H}_2\in\mathbb{R}^{k_2\times n}$ as
\begin{equation}\label{archdist}
\mathcal{L}(\boldsymbol{H}_1,\boldsymbol{H}_2)=\sum_{i=1}^{k_1}\min_{j\in [k_2]} \|\boldsymbol{H}^1_{i,.}-\boldsymbol{H}^2_{j,.}\|_2^2.
\end{equation}
\citet{javadi2019nonnegative} show that under the so-called \emph{uniqueness} assumption, one has $\mathcal{L}(\boldsymbol{H}_0,\boldsymbol{H}_{\boldsymbol{X}})^{1/2}\leq C\max_{i\in[m]}\|\boldsymbol{Z}_{i,.}\|_2$
for some constant $C>0$ where $\boldsymbol{H}_{\boldsymbol{X}}$ is the solution to the relaxed AA (i.e. (\ref{archl0noise}) with $\ell=nk$). Note that $\mathcal{L}(\boldsymbol{H}_1,\boldsymbol{H}_2)$ in (\ref{archdist}) is a sum of the squared distances between each row of $\B{H}_1$ and the row of $\B{H}_2$ closest to it.
Observe that $\mathcal{L}(\boldsymbol{H}_1,\boldsymbol{H}_2)$ is not symmetric in its arguments. In fact, a small value of $\mathcal{L}(\boldsymbol{H}_1,\boldsymbol{H}_2)$ does not imply that $\mathcal{L}(\boldsymbol{H}_2,\boldsymbol{H}_1)$ is also small (see Section~\ref{strongtoweak} for details). Definition \ref{robdef} below presents a 
 formal definition of weak and strong robustness.

  \begin{dfn}\label{robdef}
({\bf Robustness}) An estimator $\B{H}\in\mathbb{R}^{k\times n}_{\geq 0}$ is said to be:\\
\smallskip
~~(1) ({\it Weak robustness}): Weakly robust if $\mathcal{L}(\boldsymbol{H}_0,\B{H})^{1/2}\leq f_1(\max_{i\in[m]}\|\boldsymbol{Z}_{i,.}\|_2)$  where $f_1$ is an increasing real function that does not depend on $\B{X}$.\\
\smallskip
~~(2) ({\it Strong robustness}): Strongly robust if $\mathcal{L}(\B{H},\boldsymbol{H}_0)^{1/2}\leq f_2(\max_{i\in[m]}\|\boldsymbol{Z}_{i,.}\|_2)$  where $f_2$ is an increasing real function that does not depend on $\B{X}$.
    \end{dfn}

\noindent Note that based on Definition~\ref{robdef}, the result of \citet{javadi2019nonnegative} is an instance of weak robustness with $f_1(x)=Cx$ for $x \geq 0$ and some constant $C>0$.

\subsection{Strong robustness implies weak robustness}\label{strongtoweak}
Here we explain the differences between weak and strong robustness and provide some intuition around the choice of terminology: strong and weak.
\subsubsection{Illustrative Examples}

First, we present an illustrative example below exploring strong and weak robustness.
\begin{exa}\label{robustness-newexample}
Consider the example shown in Figure~\ref{fig:newexample}. In this example, we set $n=2,k=3,m=200$. In Figure~\ref{fig:newexample}(a), we show the underlying archetypes and their convex hull in blue. We also illustrate noisy data points with red dots. In Figure~\ref{fig:newexample}(b), we show the regularized AA solution of~\citet{javadi2019nonnegative} in black; in addition to the underlying archetypes and the noisy data. The regularized AA solution is close to the convex hull of the data, and while it does not completely contain the data, it contains the data with a small error. 
In Figure~\ref{fig:newexample}(c,d,e) we show 3 sets of candidate archetypes which contain the data perfectly (i.e., $D(\B{X},\B{H})=0$). In other words, these sets of candidate archetypes are feasible for problem~\eqref{archl0noise} (with $\ell=kn$, no sparsity). However, as we can see, these solutions are not that close to the underlying true archetypes that generate the noiseless data.
 Comparing Figure~\ref{fig:newexample}(b) to Figure~\ref{fig:newexample}(c,d,e), we see that the archetypal solution (that is closer to the convex hull of the data) is closer to the true model, and therefore more desirable in practice. We show that the solution in Figure~\ref{fig:newexample}(b), indeed, is the most robust solution among the ones presented here. For concreteness, we report strong and weak robustness quantities for each example in Table~\ref{table:new_example}.

\begin{figure}[t!]

     \centering
     \centering
    \begin{minipage}{.22\linewidth}
    \begin{tabular}{c}
         \includegraphics[width=0.98\textwidth, trim = 2cm .4cm 2cm .5cm,  clip = true]{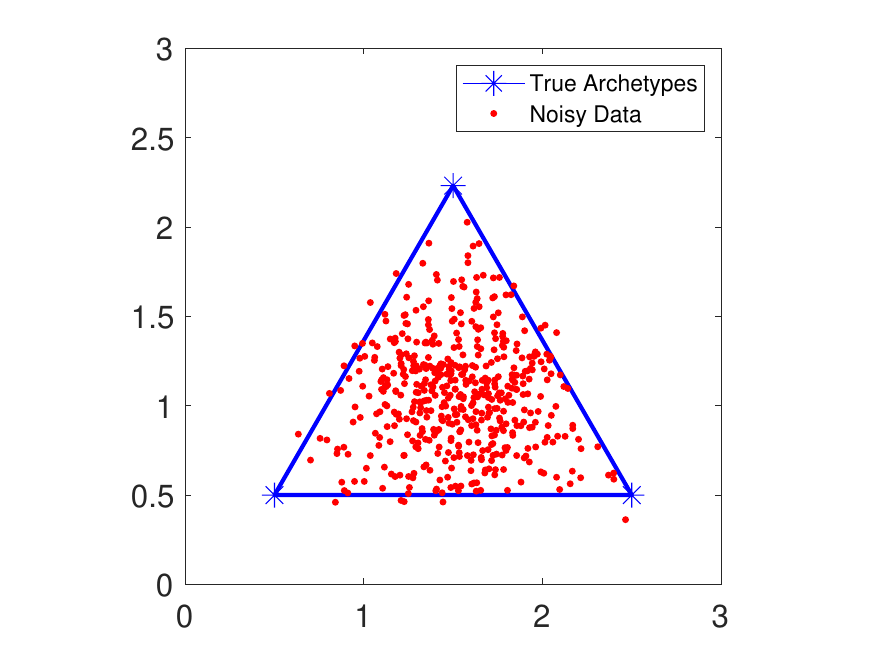} \\ 
         (a)\\
         \includegraphics[width=0.98\textwidth, trim = 2cm .4cm 2cm .5cm,  clip = true]{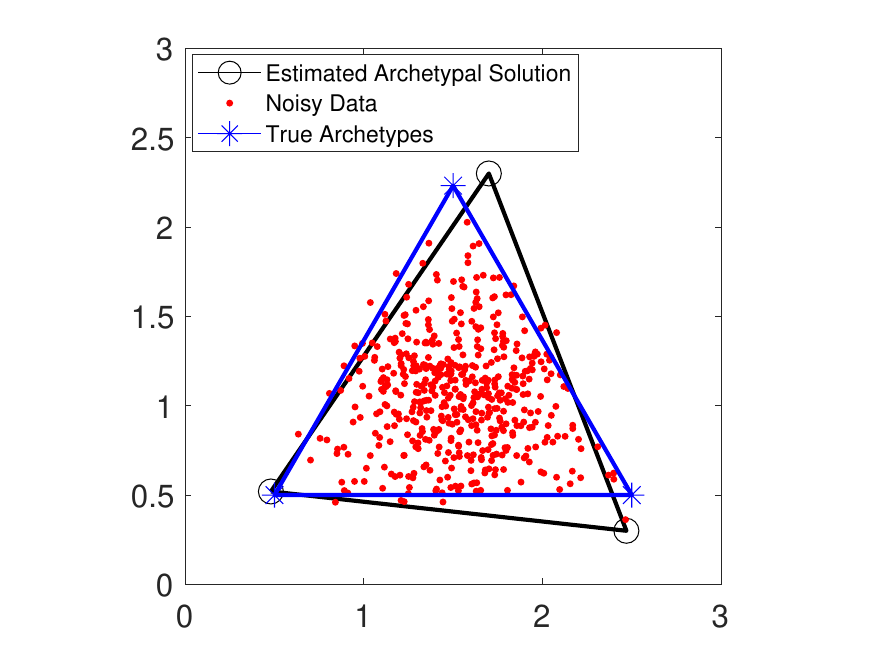}\\
         (b)
         \end{tabular}
    \end{minipage}
    \begin{minipage}{.22\linewidth}
    \begin{tabular}{c}
         \includegraphics[width=0.98\textwidth, trim = 5cm .5cm 5cm .5cm,  clip = true]{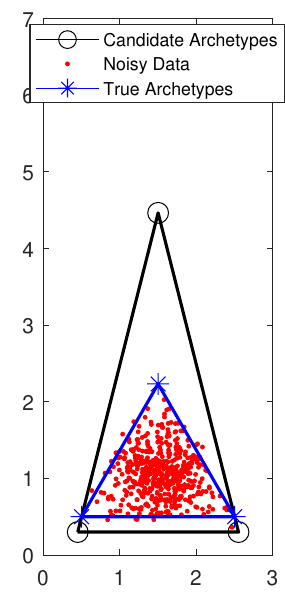}\\
         (c)
         \end{tabular}
    \end{minipage}
        \begin{minipage}{.22\linewidth}
    \begin{tabular}{c}
         \includegraphics[width=0.98\textwidth, trim = 5cm .5cm 5cm .5cm,  clip = true]{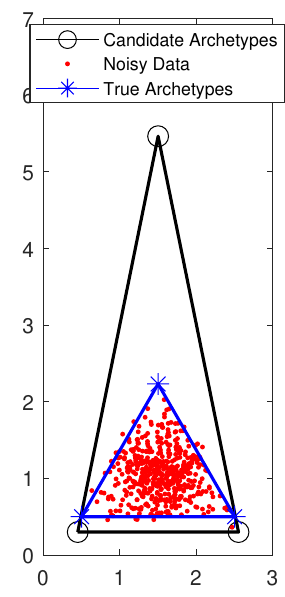}\\
         (d)
         \end{tabular}
    \end{minipage}
            \begin{minipage}{.22\linewidth}
    \begin{tabular}{c}
         \includegraphics[width=0.98\textwidth, trim = 5cm .5cm 5cm .5cm,  clip = true]{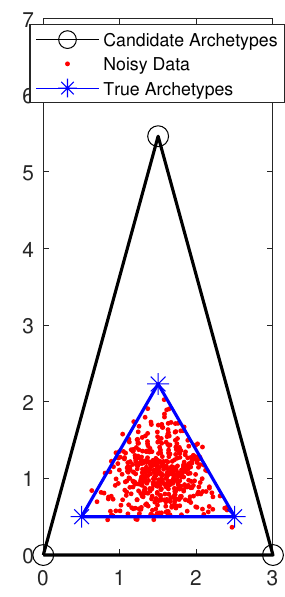}\\
         (e)
         \end{tabular}
    \end{minipage}
     \caption{\small  Solutions for Example~\ref{robustness-newexample}. Panel (a): The true (underlying) archetypes and the noisy data Panel (b): Noisy data and true archetypes, in addition to the AA solution which is strongly and weakly robust and is close to the true model. Panels (c,d,e): Noisy data and true archetypes, in addition to some candidate archetypes that contain the noisy data perfectly (i.e. are feasible for problem~\eqref{archl0noise} with $\ell=kn$). Particularly, the candidate archetypes in Panels (c) and (d) have equal weak robustness error, while the one in Panel (c) has a smaller strong robustness error. This shows that strong robustness can distinguish between sets of candidate solutions in Panels (c) and (d), unlike weak robustness.}    \label{fig:newexample}
\end{figure}

    \begin{table}[h!]
 
        \centering
        \begin{tabular}{|c|c|c|}
        \hline 
             Solution & $\mathcal{L}(\B{H}_0,\B{H})/\|\B{H}_0\|_F^2$    & $\mathcal{L}(\B{H},\B{H}_0)/\|\B{H}_0\|_F^2$ \\
                        &  (Weak)  & (Strong)\\\hline
             $\B{H}_0 (a)$ & 0 & 0 \\
             \hline
             $\B{H}_1 (b)$ & 0.0071 & 0.0071 \\
             \hline
             $\B{H}_2 (c)$ & 0.4022 & 0.4142\\
             \hline
             $\B{H}_3 (d)$ & 0.4022  & 0.8609\\
             \hline
             $\B{H}_4 (e)$ & 0.6730 & 0.9358\\
             \hline
        \end{tabular}
        \caption{\small  Robustness quantities for solutions appearing in Example~\ref{robustness-newexample}. }
        \label{table:new_example}
    \end{table}
    
   The ground truth candidate in Figure~\ref{fig:newexample}(a) has zero robustness error (i.e. is perfectly robust). As we move towards solutions that are further from the true model (and therefore less desirable/interpretable), robustness error quantities ($\mathcal{L}(\B{H}_0,\B{H}),\mathcal{L}(\B{H},\B{H}_0)$) become larger. We can also see the difference between weak and strong robustness here. We see that (c) and (d) have the same weak robustness quantities, while (c) has a lower strong robustness error. In other words, the cases (c) and (d) are the same in terms of weak robustness but differ in terms of strong robustness.
 This occurs because in both (c) and (d), each true archetype is fairly close to an estimated archetype, while in (d), there is an estimated archetype that is far from all true archetypes. Finally, (e) has larger weak and strong robustness errors compared to (c,d). This example, hence, shows how robustness, and specially strong robustness, can be useful in practice by quantifying the closeness to the correct model.
\end{exa}

In Example~\ref{example} we present a concrete instance where the estimator is weakly robust but not strongly robust.
\begin{exa}\label{example}
Let $m=3$, $n=2$ and $k=2$ and
$$\boldsymbol{X}_0=\begin{bmatrix}
0& 1 \\
1 & 0 \\
1/2 & 1/2
\end{bmatrix},~~~~\boldsymbol{H}_0=\begin{bmatrix}
1& 0 \\
0 & 1 
\end{bmatrix}.$$
For $\theta \in (0, \pi/4)$, we let the noisy data ($\B{X}_{\theta}$) and noise matrix ($\B{Z}_{\theta})$ be: 
$$\boldsymbol{X}_{\theta}=\begin{bmatrix}
\sqrt{1-\cos\theta}\cos (\tfrac\pi4-\tfrac\theta2)&1+\sqrt{1-\cos\theta}\sin (\tfrac\pi4-\tfrac\theta2)\\
1-\frac{\sin\theta}{\sqrt{2}\sin(\theta+\tfrac\pi4)}&0\\
\tfrac12 & \tfrac12
\end{bmatrix}~~~~~\text{and}$$
$$\boldsymbol{Z}_{\theta}=\begin{bmatrix}
\sqrt{1-\cos\theta}\cos (\tfrac\pi4-\tfrac\theta2)&\sqrt{1-\cos\theta}\sin (\tfrac\pi4-\tfrac\theta2)\\
-\frac{\sin\theta}{\sqrt{2}\sin(\theta+\tfrac\pi4)}&0\\
0 & 0
\end{bmatrix}.~~~~~\phantom{\text{and}}$$
\begin{figure}[t!]
     \centering
\begin{tabular}{cc}
         \includegraphics[width=0.47\textwidth,trim = 4cm 4cm 4cm 4cm,  clip = true]{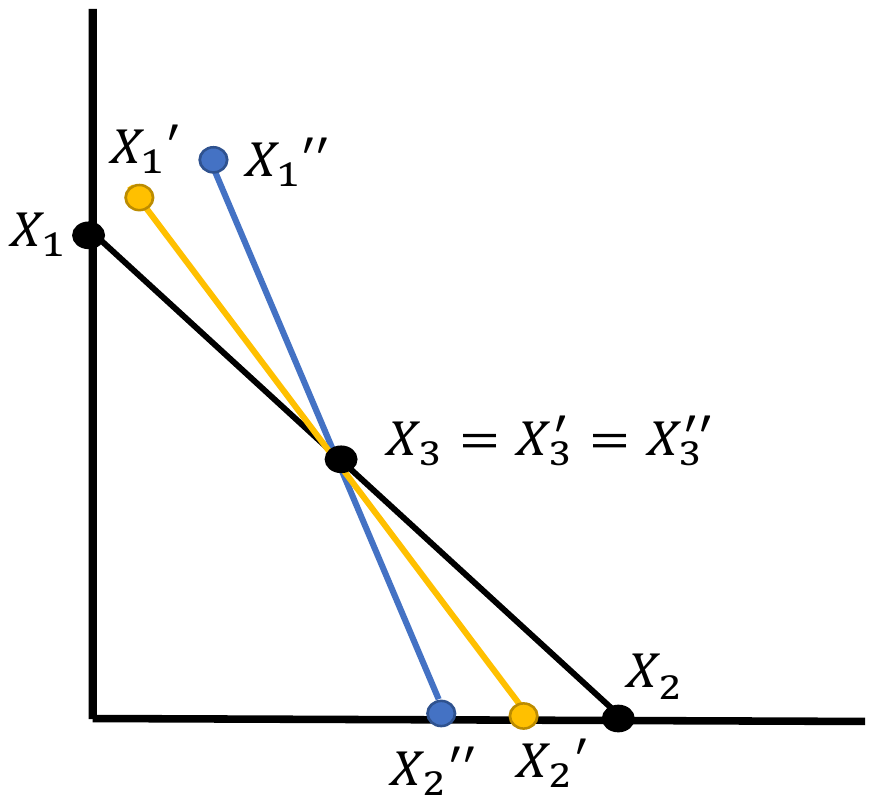}&
          \includegraphics[width=0.47\textwidth,trim = 4cm 4.0cm 4cm 4cm,  clip = true]{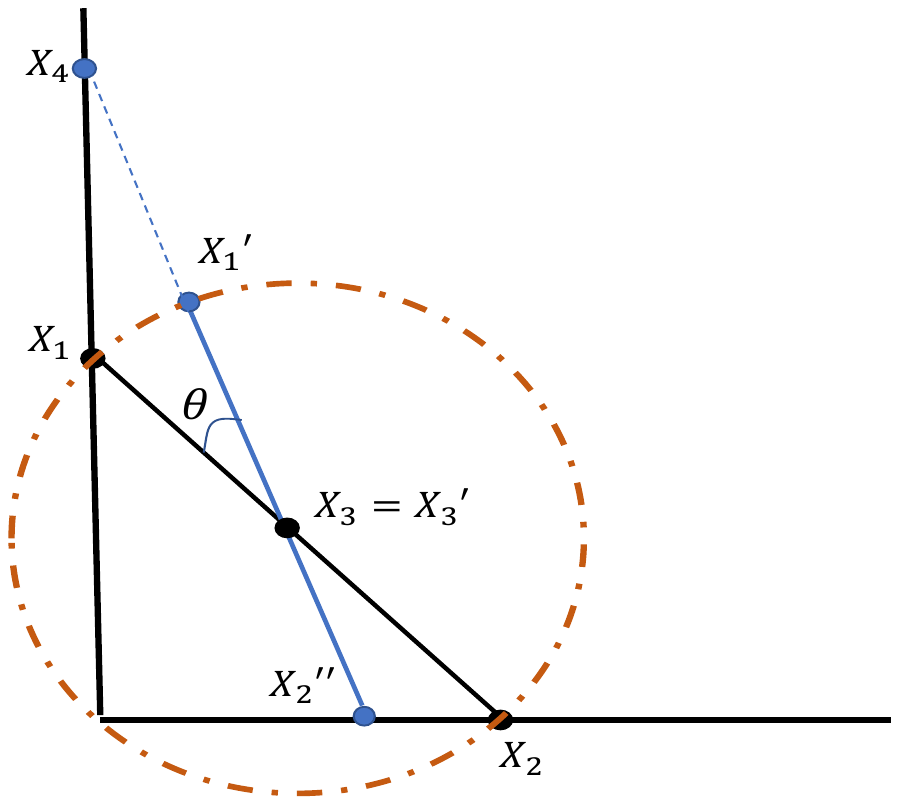} \\
          (a) & (b) 
         \end{tabular}
     \caption{\small Illustration of Example 1: (a) The  noiseless  data  and  two  examples  of noisy data. (b) Details of the example for a specific value of $\theta$. }    \label{fig12}
\end{figure}
Figure~\ref{fig12} presents an illustration of Example~\ref{example}. 
In this figure (panel (a)), $X_1,X_2,X_3$ correspond to the original data points (i.e., the rows of $\B{X}_0$) and points $X_1',X_2',X_3'$ and $X_1'',X_2'',X_3''$ are the noisy data points for two different values of $\theta$. Note that, in all cases, the data points lie on a line. 
The lines $X_1',X_2',X_3'$ and $X_1'',X_2'',X_3''$ are obtained by rotating the noiseless data line $X_1,X_2,X_3$ along its center $(1/2,1/2)$.
Figure~\ref{fig12} (panel (b)) shows line $X_{1}, X_{2}, X_{3}$ and its rotated version $X'_{1}, X_{2}', X_{3}'$ (after being rotated by an angle $\theta$).
Let 
$$\boldsymbol{H}_{\theta}=\begin{bmatrix}
0 &\left[1-\frac{\sin\theta}{\sqrt{2}\sin(\theta+\frac\pi4)}\right]\tan(\theta+\frac\pi4)\\
1-\frac{\sin\theta}{\sqrt{2}\sin(\theta+\frac\pi4)}&0
\end{bmatrix}.$$
Note that $\B{H}^{\theta}_{2,.}$ is the same as the point $\B{X}_{2,.}^{\theta}$. 
The line passing through the noisy data points intersects the y-axis at $\B{H}^{\theta}_{1,.}$. For the line, $X_{1}', X_{2}', X_{3}'$,  the point $\B{H}^{\theta}_{1,.}$ is given by 
$X_4$ in Figure \ref{fig12}, (b). As a result, $\B{X}^{\theta}_{1,.},\B{X}^{\theta}_{2,.},\B{X}^{\theta}_{3,.}$ are on the segment connecting $\B{H}^{\theta}_{1,.}$ and $\B{H}^{\theta}_{2,.}$ and $D(\B{X}_{\theta},\B{H}_{\theta})=0$. In addition, $\max_{i\in[3]}\|\B{Z}_{i,.}^{\theta}\|_2\leq \sqrt{2}$ showing the amount of noise added to the data is limited. Moreover,
\begin{align*}
    \mathcal{L}(\boldsymbol{H}_0,{\boldsymbol{H}}_{\theta})\leq \|\boldsymbol{H}^0_{1,.}-{\boldsymbol{H}}_{2,.}^{\theta}\|_2^2+\|\boldsymbol{H}^0_{2,.}-{\boldsymbol{H}}_{2,.}^{\theta}\|_2^2 \leq  4
\end{align*}
for all $\theta \in (0, \pi/4)$, showing $\B{H}_{\theta}$ is weakly robust. However, note that
\begin{align*}
    \mathcal{L}({\boldsymbol{H}}_{\theta},\boldsymbol{H}_0)& \geq \min_{i\in[2]}\|{\boldsymbol{H}}^{\theta}_{1,.}-\boldsymbol{H}^0_{i,.}\|_2^2 = \|{\boldsymbol{H}}^{\theta}_{1,.}-\boldsymbol{H}^0_{2,.}\|_2^2 \\&= \left(\left[1-\frac{\sin\theta}{\sqrt{2}\sin(\theta+\pi/4)}\right]\tan(\theta+\pi/4)-1\right)^2 
\end{align*}
and $\mathcal{L}({\boldsymbol{H}}_{\theta},\boldsymbol{H}_0)\to\infty$ as $\theta\to\pi/4$, showing $\B{H}_{\theta}$ is not strongly robust.
\end{exa}
\noindent $\B{H}_{\theta}$ is not a strongly robust estimator in Example \ref{example} as the first row of $\B{H}_{\theta}$ can be far from $\B{H}_0$, the set of underlying archetypes. 
Weak robustness implies that among recovered archetypes, there is at least one of them which is close to the correct archetypes (for example, the second row of $\B{H}_{\theta}$ in Example \ref{example} is close to the first row of $\B{H}_0$). On the other hand, strong robustness implies that all recovered archetypes are close to some underlying archetype.
This is true because considering the definition of $\mathcal{L}(\boldsymbol{H}_1,\boldsymbol{H}_2)$ in (\ref{archdist}), strong robustness limits the distance between each recovered archetype and the set of correct archetypes. However, weak robustness limits the distance between each true archetype and the set of recovered archetypes, therefore, this distance can be small even if some recovered archetypes are far.  
\subsubsection{Theoretical Results}
Theorem \ref{weakstrongthm} below states that strong robustness implies weak robustness (and in light of Example~\ref{example}, this containment is strict).
\begin{thm}[Strong robustness implies weak robustness]\label{weakstrongthm}
Let us define the quantity
\begin{equation}\label{bh}
    b(\B{H}_0)=\max_{i,j\in [k]}\|\boldsymbol{H}^0_{i,.}-\boldsymbol{H}^0_{j,.}\|_2.
\end{equation}
Then for any $\B{H}\in\mathbb{R}_{\geq 0}^{k\times n}$ we have
\begin{equation}\label{thm1-bound-1}
\mathcal{L}(\boldsymbol{H}_0,\boldsymbol{H})\leq 2kb(\B{H}_0)^2+2\mathcal{L}(\boldsymbol{H},\boldsymbol{H}_0).
\end{equation}
\end{thm}

If $\B{H}$ is a strongly robust estimator, $\mathcal{L}(\boldsymbol{H},\boldsymbol{H}_0)$ is bounded and as $b(\B{H}_0)$ depends on the underlying archetypes, $b(\B{H}_0)$ is finite. Therefore, the right hand side in~\eqref{thm1-bound-1} is bounded and $\mathcal{L}(\boldsymbol{H}_0,\boldsymbol{H})$ is bounded, implying $\B{H}$ is also weakly robust.

\begin{rmk} Interestingly, Theorem~\ref{weakstrongthm} does not tell us that strong robustness can be obtained from weak robustness. Note that we can interchange $\B{H}$ and $\B{H}_0$ in Theorem~\ref{weakstrongthm}, which results in the following inequality:
\begin{equation}\label{vacuous}
    \mathcal{L}(\boldsymbol{H},\boldsymbol{H}_0)\leq 2kb(\B{H})^2+2\mathcal{L}(\boldsymbol{H}_0,\boldsymbol{H})
\end{equation}
where 
$$b(\B{H})=\max_{i,j\in [k]}\|\boldsymbol{H}_{i,.}-\boldsymbol{H}_{j,.}\|_2.$$
However, $b(\B H)$ can be arbitrarily large---from Example \ref{example}, we have:
$$b(\B{H}_\theta)=\left(\left[1-\frac{\sin\theta}{\sqrt{2}\sin(\theta+\tfrac\pi4)}\right]^2\tan^2(\theta+\tfrac\pi4) + \left[1-\frac{\sin\theta}{\sqrt{2}\sin(\theta+\tfrac\pi4)}\right]^2\right)^{\tfrac12}$$
and $b(\B{H}_\theta)\to\infty$ as $\theta \rightarrow \pi/4$, making the upper bound in~\eqref{vacuous} uninformative. 
\end{rmk}

\noindent Theorem \ref{noiseremove} shows that a weak/strong robust estimator $\B{H}$ serves as a good approximation to the noiseless data matrix $\B{X}_0$. 
\begin{thm}\label{noiseremove}
For any $\boldsymbol{H}\in\mathbb{R}_{\geq 0}^{k\times n}$ we have
\begin{equation}\label{noiseremoveeq}
   D(\boldsymbol{X}_0,\boldsymbol{H})^{1/2}\leq \sqrt{m}\min\left\{\mathcal{L}(\boldsymbol{H}_0,\boldsymbol{H})^{1/2}, k\|\boldsymbol{H}_0\|_F+ \mathcal{L}(\boldsymbol{H},\boldsymbol{H}_0)^{1/2}\right\}.
\end{equation}
\end{thm}
\noindent If the right hand side in (\ref{noiseremoveeq}) is small, which means $\boldsymbol{H}$ is either a weakly or strongly robust estimator, then the left hand side in (\ref{noiseremoveeq}) is small. This shows that the noiseless data is close to the convex hull of $\B{H}$---so rows of $\B{H}$ are good representatives of the noiseless data. See Figure \ref{fig:thm} for an illustrative example.

\begin{figure}[t]
         \centering
         \includegraphics[width=.4\textwidth,trim = 0cm .5cm 1cm 1cm,  clip = true]{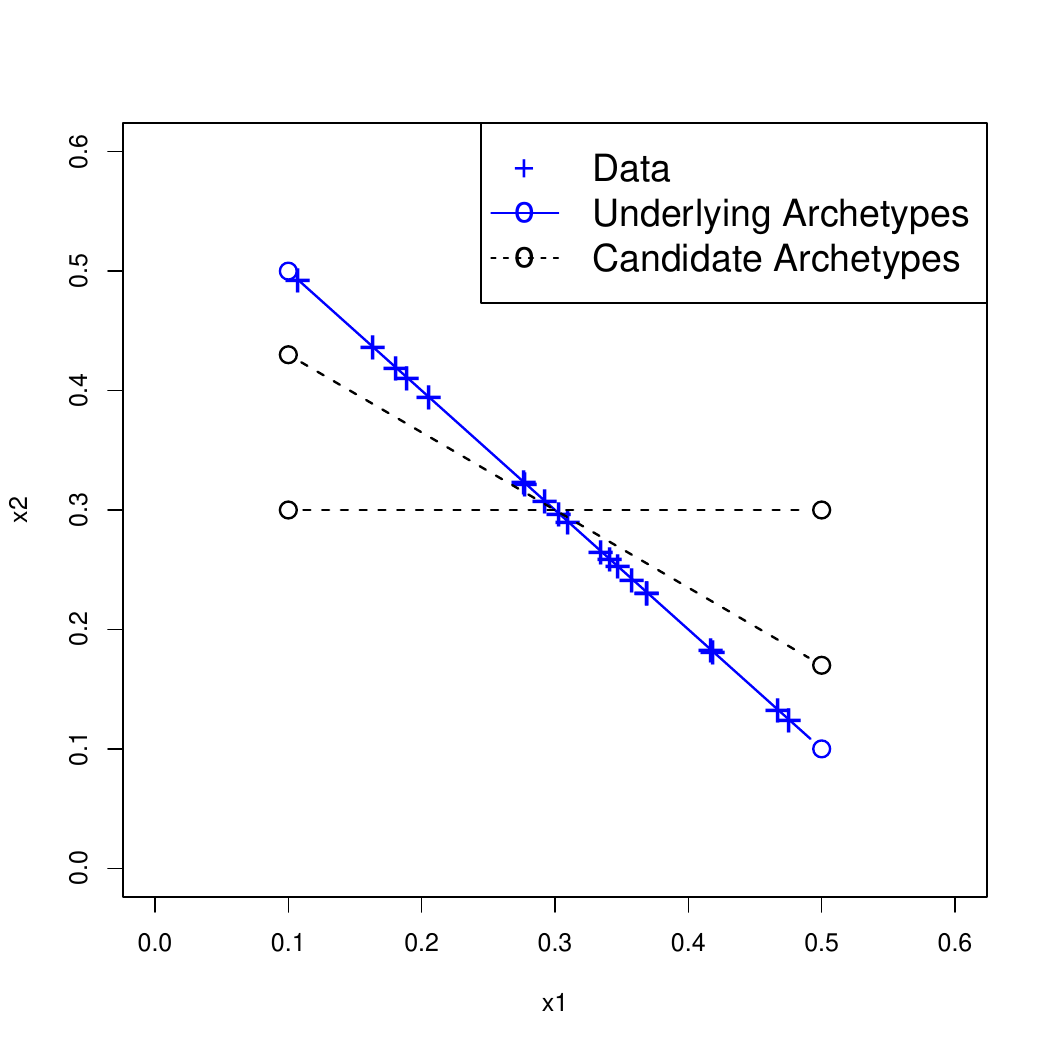}
        \caption{\small In this figure, blue crosses represent the data points in $\mathbb{R}^2$ and blue circles and line represent $\boldsymbol{H}_0$ and its convex hull containing the data. Black circles and lines represent two candidate set of archetypes and their convex hulls. Note that the set that is closer to $\boldsymbol{H}_0$ describes the data better as anticipated by Theorem \ref{noiseremove}.  }
        \label{fig:thm}
\end{figure}

\section{Archetypal Analysis with Sparsity Constraints}\label{SAA}
In this section, our primary goal is to show that under suitable conditions, both weak and strong robustness hold for the Sparse Archetypal Analysis (SAA) estimator given by (\ref{archl0noise}).

\begin{thm}\label{robustnessthm}
Let $\B{X}_0,\B{H}_0$ be as described in Section \ref{modelsetting} and $\hat{\B{H}}$ be a solution of problem \eqref{archl0noise}. Set $\alpha=\delta+\beta$ where $\delta= \max_{i\in[m]}\|\boldsymbol{Z}_{i,.}\|_2$ and $\beta = \|P_{\ell}^{\perp}(\boldsymbol{H}_0)\|_F$. Moreover, let $\B{X}=\B{X}_0+\B{Z}$ and $\tilde{\boldsymbol{X}}_0\in\mathbb{R}^{k\times n}$ be such that
$$\tilde{\boldsymbol{X}}^0_{i,.}=\argmin_{\boldsymbol{u}\in\{\boldsymbol{X}^0_{j,.}:j\in[m]\}} \|\boldsymbol{u}-\boldsymbol{H}^0_{i,.}\|_2.$$
There exist constants\footnote{To aid readability, the precise expressions for these constants are presented in Section~\ref{app:sec:proofthm3}.} $c_1,\ldots,c_{10}$ depending on $k,\kappa(\boldsymbol{H}_0),\sigma_{\min}(\boldsymbol{H}_0)$ such that the following bounds hold:\\
1. (Weak Robustness) 
\begin{equation}\label{weakrob}
    \mathcal{L}(\boldsymbol{H}_0,\hat{\boldsymbol{H}})^{1/2} \leq c_1 D(\boldsymbol{H}_0,\tilde{\boldsymbol{X}}_0)^{1/2} \\
     + c_2 \max_{i\in[m]}\|\boldsymbol{Z}_{i,.}\|_2  + c_3\|P^{\perp}_{\ell}(\boldsymbol{H}_0)\|_F 
\end{equation}
2. (Strong Robustness) If
\begin{equation}\label{condstr}
    c_4 D(\boldsymbol{H}_0,\tilde{\boldsymbol{X}}_0)^{1/2} + c_5 \max_{i\in[m]}\|\boldsymbol{Z}_{i,.}\|_2      \\+ c_6\|P^{\perp}_{\ell}(\boldsymbol{H}_0)\|_F\leq c_7,
\end{equation}
we have the following strong robustness guarantee
\begin{equation}\label{strrob}
    \mathcal{L}(\hat{\boldsymbol{H}},\boldsymbol{H}_0)^{1/2}\leq c_8 D(\boldsymbol{H}_0,\tilde{\boldsymbol{X}}_0)^{1/2} \\+ c_9 \max_{i\in[m]}\|\boldsymbol{Z}_{i,.}\|_2 + c_{10}\|P^{\perp}_{\ell}(\boldsymbol{H}_0)\|_F.
\end{equation}
\end{thm}
\noindent In Theorem \ref{robustnessthm}, $\delta$ controls the amount of additive noise $\B{Z}$; and $\beta = \|P_{\ell}^{\perp}(\boldsymbol{H}_0)\|_F$ captures the sparsity level of $\B{H}_0$. If the amount of noise in the data is high (i.e. $\delta$ is large) and/or $\B{H}_0$ is not sparse (i.e. $\beta$ is large), the value of $\alpha$ in problem \eqref{archl0noise} is larger. This implies that $D(\boldsymbol{X}_{i,.},\hat{\boldsymbol{H}})$ for $i\in[m]$ can be potentially larger because of the constraint $D(\B{X}_{i,.},\B{H})^{1/2}\leq \alpha$ in problem \eqref{archl0noise} becomes more loose---$\hat{\B{H}}$ might not represent the data points well. However, this is the price we pay to guarantee robustness.

\noindent If $\B{H}_0$ is not $\ell$-sparse, or equivalently $\|P^{\perp}_{\ell}(\boldsymbol{H}_0)\|_F>0$, problem \eqref{archl0noise}
obtains a sparse estimator $\hat{\B{H}}$, which approximates ${\B{H}_0}$.
This is an example of model misspecification; and even in this case, 
Problem~\eqref{archl0noise}  leads to an estimator that is weakly and strongly robust. 
In Theorem~\ref{robustnessthm}, the quantity $D(\boldsymbol{H}_0,\tilde{\boldsymbol{X}}_0)$ determines how close the underlying model is to the noiseless data; and it depends upon $\B{H}_0$ and $\B{X}_0$. Choosing $\B{H}_0$ as in \eqref{h0opt} results in a smaller value of $D(\boldsymbol{H}_0,\tilde{\boldsymbol{X}}_0)$ and improves our bounds. The constant $D(\boldsymbol{H}_0,\tilde{\boldsymbol{X}}_0)$ can be zero which is a generalization of the separable case \citep{NIPS2003_2463}, where the underlying archetypes are assumed to be among noiseless data points. In addition, condition (\ref{condstr}) ensures that the noise in the data is not too large and the underlying archetypes are suitably sparse --- this suffices to derive a strong robustness guarantee for $\hat{\B{H}}$.

Another interesting point is that the bounds in Theorem~\ref{robustnessthm} do not directly depend on the dimension of the data, $n$. This shows our results can be applied to the high-dimensional settings where $n\gg m$. In addition, the bounds in Theorem~\ref{robustnessthm} do not depend on the number of samples, $m$. This might seem counter-intuitive as increasing the number of samples and adding new observations does not necessarily result in a smaller robustness bound. We believe this is due to the nature of our worst-case analysis. 
In fact, in an extreme case, one can increase the number of samples by recycling old samples, and this does not add any `new information' to the problem. 
With more samples, however, we obtain a better representation of the (underlying) data. This is captured by the average error in the representation of the noiseless data, defined as
$$\frac{1}{m}D({\B{X}_0,\B{H}})^{1/2}.$$
Note that by Theorem~\ref{noiseremove} and Theorem~\ref{robustnessthm}, increasing $m$, while keeping other parameters constant, results in $D({\B{X}_0,\B{H}})^{1/2}/m\to 0$, showing the resulting archetypes can represent the original data well on average. We would also like to remind the reader that our analysis makes no distributional assumption on the data and is valid under minimal assumptions, making our result applicable to quite general situations.

It is also worth noting that parameters appearing in the bounds in Theorem~\ref{robustnessthm} depend on the underlying model and may be unknown in practice. Therefore, numerically computing these upper bounds on real data instances will be tricky. Nonetheless, our bounds here show the dependence of robustness on model parameters such as $m,n,k$. This can provide guidance on how the model performance degrades/improves if, for example, new data points are added or the rank of the factorization is changed. See Section~\ref{synthnumundrob} for a numerical exploration of this idea. 

\noindent Finally, we outline how our proof techniques differ from the ones by~\citet{javadi2019nonnegative}. The proof of robustness in \citet{javadi2019nonnegative} critically depends on constructing a feasible solution for problem (2.4) in~\citet{javadi2019nonnegative}, and comparing its objective to the optimal solution. In particular, this is done in Lemma~B.6 of \citet{javadi2019nonnegative}, and the feasible solution is denoted by $\widetilde{\B{H}}$. This way of constructing a feasible solution does not apply in our case because they use the underlying model to construct this solution, not the properties of the optimization problem. Therefore, the sparsity and feasibility of $\widetilde{\B{H}}$ cannot be guaranteed. In summary, it is unclear to us if we can obtain our results as an
 extension/adaptation of~\citet{javadi2019nonnegative} --- so we use different proof techniques to arrive at our results.

\subsection{A deeper investigation of the separable case~\texorpdfstring{\citep{NIPS2003_2463}}{}}

\noindent In the special case where the underlying model is separable (i.e., $D(\boldsymbol{H}_0,\tilde{\boldsymbol{X}}_0)=0$) we can derive tighter bounds for robustness: see Proposition~\ref{sepinexact} below. 
\begin{pro}\label{sepinexact}
Let $\hat{\B{H}}$ be a solution to Problem~\eqref{archl0noise}, $\delta=\max_{i\in[m]}\|\boldsymbol{Z}_{i,.}\|_2$ and $\beta=\|P^{\perp}_{\ell}(\boldsymbol{H}_0)\|_F$. Under the assumption of Theorem~\ref{robustnessthm} and assuming separability  i.e, $D(\boldsymbol{H}_0,\tilde{\boldsymbol{X}}_0)=0$, we have the following:\\
1. (Weak Robustness)
    $$\mathcal{L}(\boldsymbol{H}_0,\hat{\boldsymbol{H}})^{1/2}\leq  \left(4\sqrt{2}\kappa(\B{H}_0)\sqrt{k}+3(1+\sqrt{2})\sqrt{2}k\right)\delta + \left(2\sqrt{2}\kappa(\B{H}_0)\sqrt{k}+(2+2\sqrt{2})k\right)\beta.$$
2. (Strong Robustness) If $ \left(12\sqrt{2}k+18\sqrt{2}k\right)\delta + \left(6\sqrt{2}k+12k\right)\beta\leq {\sigma_{\min}(\boldsymbol{H}_0)}$,
    $$\mathcal{L}(\hat{\boldsymbol{H}},\boldsymbol{H}_0)^{1/2}\leq \left(21\kappa(\B{H}_0)\sqrt{2k}+2(1+\sqrt{2})\sqrt{2}k\right)\delta + \left(14\kappa(\B{H}_0)\sqrt{k}+\sqrt{2}(1+\sqrt{2})k\right)\beta. $$
\end{pro}

\noindent The results in Proposition~\ref{sepinexact} can be summarized as 
$$\mathcal{L}(\boldsymbol{H}_0,\hat{\boldsymbol{H}})^{1/2} = \mathcal{O}(k\delta+k\beta)~~~\text{and}~~~\mathcal{L}(\hat{\boldsymbol{H}},\boldsymbol{H}_0)^{1/2} = \mathcal{O}(k\delta+k\beta)$$
where $\delta=\max_{i\in [m]}\|\B{Z}_{i,.}\|_2$ (assuming $\kappa(\B{H}_0)$ is constant). 
In particular, in the separable case, Proposition~\ref{sepinexact} implies that robustness degrades linearly with rank. Our numerical experiments in Section~\ref{synthnumundrob} further validate this result.

\paragraph{Comparison with \citet{javadi2019nonnegative}:} Although \citet{javadi2019nonnegative} do not consider sparsity in their formulation, it is useful to compare our results with theirs. The results of \citet{javadi2019nonnegative} are valid under a specific uniqueness assumption on the model---this matches with our assumption in the separable case. Therefore, to compare our results we consider the case $\beta=D(\boldsymbol{H}_0,\tilde{\boldsymbol{X}}_0)=0$. In this regime, the result of \citet{javadi2019nonnegative}  is similar to the first part of the Proposition~\ref{sepinexact} without any sparsity guarantee on the solution and with different coefficients. 
The bound of \citet{javadi2019nonnegative} is 
$$\mathcal{L}(\boldsymbol{H}_0,\hat{\boldsymbol{H}})^{1/2} = \mathcal{O}\left(k^{5/2} \max(1,\kappa(\B{H}_0)k^{-1/2})\delta\right)$$ 
(assuming other parameters in their bound are constant) which is loose compared to our bound. While the uniqueness assumption of \citet{javadi2019nonnegative} is more general than the separable case, this assumption is difficult to verify except for very simple cases, as discussed by the authors.
\citeauthor{javadi2019nonnegative} do not provide results similar to the second part of Proposition~\ref{sepinexact}. 

Finally, we discuss an important special case of the separable setting in Example~\ref{ex-pure}.
\begin{exa}[Pure AA]\label{ex-pure}
    In the pure AA setting, in the underlying archetypal decomposition $\B{X}_0=\B{W}_0\B{H}_0$, the matrix $\B{W}_0$ only has binary entries. In words, this implies that noiseless data points are the same as true archetypes. Without loss of generality, we can assume each true archetype appears at least once in the noiseless data, otherwise, that archetype is redundant and can be discarded. Hence, with this assumption we have $\text{rank}(\B{X}_0)=\text{rank}(\B{H}_0)$. Therefore, as long as $\B{H}_0$ has full-rank, our rank assumptions on $\B{X}_0$ are satisfied and the results of Proposition~\ref{sepinexact} readily apply to this case. \\
    Next, we investigate how well our bounds work in the case of pure AA. To this end, assume $m=n=k=\ell$ and let $\B{X}_0=\B{W}_0=\B{H}_0=\B{I}$. We also let $\B{Z}=\delta\B{I}$. We claim $\B{H}=(1+\delta)\B{I}$ is the solution to problem~\eqref{archl0noise}. To see this, note that $\B{X}=\B{H}$ so $D(\B{X},\B{H})=D(\B{H},\B{X})=0$ showing optimality and feasibility. Next, we calculate robustness metrics of this solution. Note that for $i\in[k]$, $\|\B{H}_{i,.}-\B{H}^0_{i,.}\|_2^2=\delta^2$ and for $i\neq j$, $\|\B{H}_{i,.}-\B{H}^0_{j,.}\|_2^2=1+(1+\delta)^2$. Therefore, 
    $$\mathcal{L}(\B{H},\B{H}_0)^{1/2}=\mathcal{L}(\B{H}_0,\B{H})^{1/2}=\sqrt{k}\delta.$$
    The above robustness quantity is 
    smaller than the worst-case bound in Proposition~\ref{sepinexact} by a
    factor of $\sqrt{k}$. However, as we observe in our numerical experiments (see Section~\ref{synthnumundrob}), in the general case, our bounds seem to capture the dependence of robustness on $k$ correctly.
\end{exa}

\subsection{The penalized formulation} Theorem \ref{robustnessthm} presents robustness guarantees for the constrained SAA problem (\ref{archl0noise}). 
We now consider the penalized formulation:
\begin{align}\label{archl0dual}
     \hat{\boldsymbol{H}}_{\lambda}\in\argmin_{\boldsymbol{H}\in\mathbb{R}_{\geq 0}^{k\times n}} ~~ \Phi_{\lambda}(\B{H}) := D(\boldsymbol{X},\boldsymbol{H}) + \lambda D(\boldsymbol{H},\boldsymbol{X})~~~
     \text{s.t.} ~~~ \|\boldsymbol{H}\|_0\leq \ell
\end{align}
which is amenable to optimization algorithms (see Section~\ref{saaalgorithm}). 
In (\ref{archl0dual}), $D(\boldsymbol{X},\boldsymbol{H})$ is the data fidelity term, $D(\boldsymbol{H},\boldsymbol{X})$ is the regularization term and $\lambda$ is the regularization parameter. In fact, $\lambda = 0$ is equivalent to setting $\alpha=0$ in problem (\ref{archl0noise}) (which can lead to an infeasible problem) and $\lambda\to \infty $ is equivalent to removing the data fidelity constraint all together. 

\noindent We show robustness properties of estimator~(\ref{archl0dual}). 
Proposition~\ref{penalizedpro} establishes both weak (\ref{weakrob}) and strong (\ref{strrob}) robustness. For simplicity, we consider the separable case ($D(\boldsymbol{H}_0,\tilde{\boldsymbol{X}}_0)=0$) where $\B{H}_0$ is sparse ($\|P^{\perp}_{\ell}(\boldsymbol{H}_0)\|_F=0$).
\begin{pro}[The penalized formulation~(\ref{archl0dual})]\label{penalizedpro}
Let $\hat{\boldsymbol{H}}_{\lambda}$ be a solution to problem \eqref{archl0dual}. Suppose the assumptions of Theorem \ref{robustnessthm} hold; and in addition $D(\boldsymbol{H}_0,\tilde{\boldsymbol{X}}_0) =\|P^{\perp}_{\ell}(\boldsymbol{H}_0)\|_F = 0$. There exist constants\footnote{The values of the constants can be found in the appendix, Section~\ref{sec:proof-of-prop1}} $c^1_{\lambda},c^2_{\lambda},c^3_{\lambda}$ depending on $m,k,\kappa(\boldsymbol{H}_0),\lambda$ such that 
$c^1_{\lambda},c^2_{\lambda},c^3_{\lambda} \rightarrow \infty$ as $\lambda\rightarrow 0$ or $\lambda \rightarrow \infty$; and the following holds:
    \begin{equation}
        \mathcal{L}(\boldsymbol{H}_0,\hat{\boldsymbol{H}}_{\lambda})^{1/2}\leq c^1_{\lambda}\max_{i\in[m]}\|\boldsymbol{Z}_{i,.}\|_2.
    \end{equation}
    Moreover, if 
    \begin{equation}\label{propcond}
        c^3_{\lambda}\max_{i\in[m]}\|\boldsymbol{Z}_{i,.}\|_2\leq c_7,
    \end{equation}
    then
    \begin{equation}
    \mathcal{L}(\hat{\boldsymbol{H}}_{\lambda},\boldsymbol{H}_0)^{1/2}\leq c^2_{\lambda}\max_{i\in[m]}\|\boldsymbol{Z}_{i,.}\|_2.
    \end{equation}
\end{pro}
\noindent Note that $c^1_{\lambda},c^2_{\lambda} \rightarrow \infty$ as $\lambda\rightarrow 0$ or $\lambda \rightarrow \infty$---hence, there is no trivial value of $\lambda$ that guarantees robustness. In fact, if we have $\lambda=0$  (this is equivalent to normal NMF), the archetypes need not be close to the data (which is approximated by the underlying archetypes)---they can be far from the underlying archetypes and robustness is not guaranteed. If $\lambda\to \infty$, we do not reduce the recovery error ($D(\B{X},\B{H})$) which does not result in robustness. This shows the usefulness of using AA as a regularization term.

\begin{rmk} 
Results for the penalized form of AA as in Proposition \ref{penalizedpro} are not discussed in
\citet{javadi2019nonnegative}, as far as we can tell. 
\end{rmk}

\begin{rmk}
  Unlike problem~\eqref{archl0noise} which can become infeasible for small values of $\alpha$ (see Remark~\ref{rmk-feasible}), problem~\eqref{archl0dual} is always feasible.
\end{rmk}

\section{Optimization Algorithms}\label{saaalgorithm}
In this section, we propose algorithms to obtain good solutions to the penalized SAA problem (\ref{archl0dual}). 
In Section \ref{cdalg}, we present a block coordinate descent method and derive its convergence guarantees. As problem~\eqref{archl0dual} is nonconvex, we discuss methods to improve the quality of solutions: Section~\ref{initsection} discusses an initialization scheme based on MIP techniques, and Section \ref{lsalg} 
presents a local search algorithm.  

\subsection{A Block Coordinate Descent Algorithm}\label{cdalg}
We rewrite problem (\ref{archl0dual}) as follows:
\begin{align}\label{constrainedtoslove}
     \min_{\boldsymbol{W},\boldsymbol{\tilde{\boldsymbol{W}}},\boldsymbol{H}} ~~ & \Psi(\boldsymbol{W},\boldsymbol{\tilde{\boldsymbol{W}}},\boldsymbol{H}):= \|\boldsymbol{X}-\boldsymbol{W}\boldsymbol{H}\|_F^2+\lambda \|\boldsymbol{H}-\tilde{\boldsymbol{W}}\boldsymbol{X}\|_F^2\\
     \text{s.t.} ~~ & \boldsymbol{H}\geq \B{0},\boldsymbol{W} \geq 0,\tilde{\boldsymbol{W}} \geq \B{0}\nonumber \\
     & \boldsymbol{W}\boldsymbol{1}_k=\boldsymbol{1}_m,\tilde{\boldsymbol{W}}\boldsymbol{1}_m=\boldsymbol{1}_k,\|\boldsymbol{H}\|_0\leq \ell\nonumber.
\end{align}
We propose a proximal gradient based block coordinate descent algorithm \citep{xu2017globally} for~\eqref{constrainedtoslove}.
We first note that the gradient of the objective function is Lipschitz for every block $\B{W}$, $\tilde{\B{W}}$, and $\B{H}$, that is
$$\|\nabla_{\B{H}}\Psi(\B{H}_1,\B{W},\tilde{\B{W}})-\nabla_{\B{H}}\Psi(\B{H}_2,\B{W},\tilde{\B{W}})\|_F\leq L_1(\B{W})\|\B{H}_1-\B{H}_2\|_F,$$
$$\|\nabla_{\B{W}}\Psi(\B{H},\B{W}_1,\tilde{\B{W}})-\nabla_{\B{W}}\Psi(\B{H},\B{W}_2,\tilde{\B{W}})\|_F\leq L_2(\B{H})\|\B{W}_1-\B{W}_2\|_F$$
$$\|\nabla_{\tilde{\B{W}}}\Psi(\B{H},\B{W},\tilde{\B{W}}_1)-\nabla_{\tilde{\B{W}}}\Psi(\B{H},\B{W},\tilde{\B{W}}_2)\|_F\leq L_3(\B{X})\|\tilde{\B{W}}_1-\tilde{\B{W}}_2\|_F$$
where $L_1(\B{W})=2(\lambda + \sigma_{\max}(\B{W})^2)$, $L_2(\B{H})=2 \max\{\sigma_{\max}(\B{H})^2,\varepsilon\}$ and $L_3(\B{X})=2\lambda\sigma_{\max}(\B{X})^2$ for any fixed $\varepsilon>0$. 
Our algorithm follows the block proximal update of \citet{xu2017globally}. Specifically, for step size values 
$\{2L_1(\B{W}_j)\}^{-1},\{2L_2(\B{H}_j)\}^{-1},\{2L_3(\B{X})\}^{-1}$, at iteration $j$ we perform the following updates:
\begin{align}
    \boldsymbol{H}_{j+1}=& \argmin_{\substack{\boldsymbol{H}\geq 0 \\ \|\B{H}\|_0\leq \ell}} \left\Vert\boldsymbol{H}-\left(\boldsymbol{H}_{j}-\frac{1}{2L_1(\B{W}_j)}\nabla_{\boldsymbol{H}}\Psi(\boldsymbol{H}_j,\B{W}_j,\tilde{\B{W}}_j)\right)\right\Vert_F^2 \label{alg1u1} \\
 \boldsymbol{W}_{j+1}=&\argmin_{\substack{\boldsymbol{W}\geq 0 \\ \B{W}\B{1}_k=\B{1}_m}} \left\Vert\boldsymbol{W}-\left(\boldsymbol{W}_{j}-\frac{1}{2L_2(\B{H}_{j+1})}\nabla_{\boldsymbol{W}}\Psi(\boldsymbol{H}_{j+1},\B{W}_j,\tilde{\B{W}}_j)\right)\right\Vert_F^2 \label{alg1u2} \\
 \tilde{\boldsymbol{W}}_{j+1}=&\argmin_{\substack{\tilde{\boldsymbol{W}}\geq 0 \\ \tilde{\B{W}}\B{1}_m=\B{1}_k}} \left\Vert\tilde{\boldsymbol{W}}-\left(\tilde{\boldsymbol{W}}_{j}-\frac{1}{2L_3(\B{X})}\nabla_{\tilde{\boldsymbol{W}}}\Psi(\boldsymbol{H}_{j+1},\B{W}_{j+1},\tilde{\B{W}}_j)\right)\right\Vert_F^2.  \label{alg1u3}
\end{align}
After a sweep across the updates \eqref{alg1u1}, \eqref{alg1u2} and \eqref{alg1u3} the objective decreases:
$$\Psi(\B{H}_{j+1},\B{W}_{j+1},\tilde{\B{W}}_{j+1})\leq \Psi(\B{H}_{j},\B{W}_{j},\tilde{\B{W}}_{j}).$$
Algorithm 1 summarizes the above procedure, where $P_{\text{simplex}}(\boldsymbol{W})$ projects each row of $\boldsymbol{W}$ onto the unit simplex. See \citet{10.1145/1390156.1390191} for an efficient algorithm to calculate $P_{\text{simplex}}$. Before presenting the theoretical analysis of Algorithm \ref{algorithm1}, we define a stationarity point.
\begin{dfn}\label{statpoint}
We say a point $\B{\theta}^*=(\boldsymbol{H}^*,\boldsymbol{W}^*,\tilde{\boldsymbol{W}}^*)$ is stationary for problem (\ref{constrainedtoslove}) if update rules \eqref{alg1u1}, \eqref{alg1u2} and \eqref{alg1u3} initialized with $\B{\theta}^*$ result in the same solution $\B{\theta}^*$.
\end{dfn}
\begin{rmk}
Definition \ref{statpoint} is a generalization of the notion of $L$-stationarity by \citet{beck} to the case of the block proximal method. Moreover, the stationarity condition in Definition \ref{statpoint} is a necessary condition for an optimal solution to 
problem~\eqref{constrainedtoslove}.
\end{rmk}

\begin{algorithm}[h] 
\SetAlgoLined
$j=0$ \\
 \While{not converged}{
 $\boldsymbol{\pi} = \boldsymbol{H}_j-[1/L_1(\B{W}_j)](-\boldsymbol{W}_j^T[\boldsymbol{X}-\boldsymbol{W}_j\boldsymbol{H}_j]+\lambda[\boldsymbol{H}_j-\tilde{\boldsymbol{W}}_j\boldsymbol{X}])$\\
 $\boldsymbol{H}_{j+1}=P_{\ell}\left(\max\{\boldsymbol{\pi},0\}\right)$\\
 $\boldsymbol{W}_{j+1}=P_{\text{simplex}}(\boldsymbol{W}_j-[1/L_2(\B{H}_{j+1})](\boldsymbol{X}-\boldsymbol{W}_j\boldsymbol{H}_{j+1})\boldsymbol{H}_{j+1}^T)$\\
 $\tilde{\boldsymbol{W}}_{j+1}=P_{\text{simplex}}(\tilde{\boldsymbol{W}}_j-\lambda[1/L_3(\B{X})](\boldsymbol{H}_{j+1}-\tilde{\boldsymbol{W}}_j\boldsymbol{X})\boldsymbol{X}^T)$\\
 $j = j+1$\\
 }
 \caption{SparseAA$(\boldsymbol{H}_0,\boldsymbol{W}_0,\tilde{\boldsymbol{W}}_0,\lambda)$}
 \label{algorithm1}
\end{algorithm}
\noindent In Theorem \ref{convthm}, first we show that problem \eqref{constrainedtoslove} satisfies the convergence conditions of \citet{xu2017globally} and therefore Algorithm \ref{algorithm1} converges. Then, we show that the limit point of Algorithm \ref{algorithm1} is a stationary point as in Definition \ref{statpoint}.  
\begin{thm}\label{convthm}
Suppose $\lambda>0$ and let $\{(\B{H}_j,\B{W}_j,\tilde{\B{W}}_j)\}_{j \geq 1}$ be the sequence of solutions produced by Algorithm \ref{algorithm1}. The following results hold:\\
1. The sequence $(\B{H}_j,\B{W}_j,\tilde{\B{W}}_j)$ converges to a feasible solution $(\B{H}^*,\B{W}^*,\tilde{\B{W}}^*)$ of~(\ref{constrainedtoslove})\\
2. Let 
\begin{equation}\label{Tcond}
\B{T}=\max\{0,\boldsymbol{H}^*-[1/L_1(\B{W}^*)](-{\boldsymbol{W}^*}^T[\boldsymbol{X}-\boldsymbol{W}^*\boldsymbol{H}^*]+\lambda[\boldsymbol{H}^*-\tilde{\boldsymbol{W}}^*\boldsymbol{X}])\}
\end{equation}
and if $\|\B{T}\|_0>\ell$, assume $\B{T}_{\ell}^{\sharp}> \B{T}_{\ell+1}^{\sharp}$ where $\B{T}_1^{\sharp},\cdots,\B{T}_{kn}^{\sharp}$ are entries of $\B{T}$ ordered from largest to smallest. Then, the limit point $(\B{H}^*,\B{W}^*,\tilde{\B{W}}^*)$ is a stationary point of~(\ref{constrainedtoslove}). \end{thm}
\noindent The condition on $\B{T}$ above  is needed to make sure $P_{\ell}(\B{T})$ is unique. Otherwise, there will be multiple possible solutions for update rule \eqref{alg1u1} when initialized with $\B{\theta}^*$.
Note however, that this condition is quite mild, and is unlikely to be violated in practice (due to noise in data). 
\subsection{Initialization via Mixed Integer Programming (MIP)}\label{initsection}
Problem (\ref{constrainedtoslove}) is not convex, hence having a good initialization can help obtain a high-quality local solution. To initialize 
Algorithm~\ref{algorithm1}, we use a continuation method as discussed below.
We first obtain a solution to~\eqref{constrainedtoslove} for a large value of $\lambda$ --- this leads to the following problem:

\begin{align}\label{infilambda}
    \min_{\boldsymbol{\tilde{\boldsymbol{W}}},\boldsymbol{H}} ~~   \|\boldsymbol{H}-\tilde{\boldsymbol{W}}\boldsymbol{X}\|_F^2~~~ \text{s.t.} ~~~ \boldsymbol{H}\geq 0;~\tilde{\boldsymbol{W}}\geq 0;~\tilde{\boldsymbol{W}}\boldsymbol{1}_m=\boldsymbol{1}_k;~\|\boldsymbol{H}\|_0\leq \ell.
\end{align}

\noindent Note that Problem (\ref{infilambda}) has a convex quadratic objective in $\tilde{\B{W}},\B{H}$ and the only source of nonconvexity is the cardinality constraint on $\B{H}$. This is a Mixed Integer Quadratic Problem (MIQP)---while these problems are computationally difficult in the worst-case, recent work
~\citep{bertsimas2016best,hazimeh2020sparse,bertsimas2020sparse}\footnote{Note that problem (\ref{constrainedtoslove}) with archetypal regularization (i.e., for $\lambda>0$) does not admit a MIQP representation, unlike (\ref{infilambda}).} has shown that a family of sparse regression problems can be solved to near-optimality using specialized algorithms for fairly large problem instances. Inspired by this line of work, we present new algorithms to solve~\eqref{infilambda} to optimality.

\noindent Once we obtain a near-optimal solution to~\eqref{infilambda}, we decrease $\lambda$ and use Algorithm~1 to obtain a feasible solution for~(\ref{infilambda}). We continue this process by successively decreasing $\lambda$, and using a solution obtained from the previous (larger) value of $\lambda$ to initialize Algorithm~\ref{algorithm1}.

\subsubsection{A Specialized MIP algorithm for solving \texorpdfstring{Problem~\eqref{infilambda}}{}} We present a MIQP formulation of~\eqref{infilambda}. We first show that the solution of this problem is bounded.
\begin{pro}\label{ridgepro}
If $(\B{H}^*, {\tilde{\B{W}}}^*)$ is an optimal solution to (\ref{infilambda}), we have the following bound on $\B{H}^*$:
$$\|\boldsymbol{H}^*\|^2_F\leq k\left(\max_{u\in[m]} \|\boldsymbol{X}_{u,.}\|_2+\sqrt{k}\min_{u\in[m]} \|\boldsymbol{X}_{u,.}\|_2\right)^2:=b.$$
\end{pro}
\noindent Based on Proposition \ref{ridgepro}, we reformulate problem \eqref{infilambda} as the following MIQP:
\begin{align}\label{miqpform}
    \min_{\B{H},\tilde{\B{W}},\B{Y}} \quad & \|\B{H}-\tilde{\B{W}}\B{X}\|_F^2 \\
     \text{s.t.} \quad & \boldsymbol{H} \geq 0,\tilde{\boldsymbol{W}}\geq 0, \B{Y}\in\{0,1\}^{k\times n}\nonumber\\
     &\tilde{\boldsymbol{W}}\boldsymbol{1}_m=\boldsymbol{1}_k, \sum_{i,j}\B{Y}_{i,j}\leq \ell \nonumber \\
     & \B{H}_{i,j}\leq \sqrt{b}\B{Y}_{i,j}\quad \forall (i,j)\in[k]\times[n],\nonumber
\end{align}
where $b$ is as defined in Proposition~\ref{ridgepro}.
Note that the last constraint in \eqref{miqpform} does not change the optimal solution because of Proposition \ref{ridgepro}. 

\noindent Problem~\eqref{miqpform} can be formulated and solved (to optimality) by off-the-shelf MIP solvers (e.g., Gurobi, CPLEX, GLPK) for small/moderate instances---however, the runtimes become long as soon as the number of variables are of the order of a few thousand. With efficiency in mind, we present a cutting plane approach to obtain a certifiably optimal solution\footnote{That is, along with delivering a feasible solution, we also present a dual bound (aka lower bound) on the optimal objective value of~\eqref{miqpform}.}. 
To this end, we rewrite \eqref{miqpform} as a binary convex optimization problem:
\begin{align}\label{fcut}
    \min_{\B{Y}} ~~~   F(\B{Y}) ~~~\text{s.t.} ~~~ \B{Y}\in\{0,1\}^{k\times n},~~&\sum_{(i,j)\in [k]\times [n]} \B{Y}_{i,j}\leq \ell 
\end{align}
where, for any $\B{Y} \in [0,1]^{k \times n}$, the objective function $F(\B{Y})$ is defined via the following convex optimization problem: 
\begin{align}\label{ftomin}
    F(\B{Y})=\min_{\B{H},\tilde{\B{W}}} \quad &   \|\B{H}-\tilde{\B{W}}\B{X}\|_F^2 \\
     \text{s.t.} \quad & \boldsymbol{H} \geq 0,~~\tilde{\boldsymbol{W}} \geq 0,~~\tilde{\boldsymbol{W}}\boldsymbol{1}_m=\boldsymbol{1}_k \nonumber \\
     & \B{H}_{i,j}\leq \sqrt{b}\B{Y}_{i,j}\quad \forall (i,j)\in[k]\times[n].\nonumber
\end{align}
Proposition \ref{dualprop} presents some properties of the function $F(\B{Y})$.
\begin{pro}\label{dualprop}
Let $(\B{H}^*,\tilde{\B{W}}^*)$ be an optimal solution to the minimization problem~\eqref{ftomin}. Then, we have the following:
\begin{enumerate}
    \item The function $\B{Y} \mapsto F(\B{Y})$ is convex on $\B{Y} \in [0,1]^{k \times n}$. 
    \item $\B{G}=-\sqrt{b}\B{\Lambda}$ is a subgradient of $F(\B{Y})$, where for $(i,j)\in[k]\times [n]$,
    $$\B{\Lambda}_{i,j}=\begin{cases}
    2(\tilde{\B{W}}^*\B{X}-\B{H}^*)_{i,j} &\mbox{ if } (\tilde{\B{W}}^*\B{X}-\B{H}^*)_{i,j}>0 \\
    0 & \mbox{ if } (\tilde{\B{W}}^*\B{X}-\B{H}^*)_{i,j}\leq 0
    \end{cases}.$$
\end{enumerate}
\end{pro}

\noindent The function $F(\B{Y})$ is convex and subdifferentiable. Specifically, for any $\B{Y}_0\in\R^{k\times n}$ and any subgradient $\B{G}_0\in\partial F(\B{Y}_0)$, \begin{equation}\label{linearlower}
    F(\B{Y})\geq F(\B{Y}_0) + \langle \B{G}_0 ,\B{Y}-\B{Y}_0\rangle.
\end{equation}
\noindent {\bf MIP Algorithm based on outer approximation:} We present an outer approximation algorithm~\citep{duran1986outer} 
to solve the binary convex program~\eqref{fcut}. This algorithm starts from an initial point $\B{Y}_0$ which is feasible for~\eqref{fcut}.  At iteration $t$, using a list of subgradient-based inequalities~\eqref{linearlower}, 
we consider the following piecewise linear lower bound of $F(\B{Y})$:
\begin{equation}\label{piecewise}
    F(\B{Y})\geq \max\left\{ F(\B{Y}_0) + \langle \B{G}_0 ,\B{Y}-\B{Y}_0\rangle,\cdots, F(\B{Y}_{t-1}) + \langle \B{G}_{t-1} ,\B{Y}-\B{Y}_{t-1}\rangle\right\}
\end{equation}
where $\B{Y}_{0},\cdots,\B{Y}_{t-1}$ are feasible for Problem~\eqref{fcut}
and $\B{G}_t\in\partial F(\B{Y}_t)$ for all $t$. We define $\B{Y}_t$ as a minimizer of the right hand side of~\eqref{piecewise} under the constraints of Problem~\eqref{fcut}. Mathematically, this can be written as a Mixed Integer Linear Program (MILP)
    \begin{align}\label{outerMILP}
    (\B{Y}_t,\eta_t) \in\argmin_{\B{Y},\eta} \quad  & \eta \\
     \text{s.t.} \quad & \B{Y}\in\{0,1\}^{k\times n},\eta\in\mathbb{R}\nonumber\\
     & \eta \geq F(\B{Y}_i)+\langle\B{G}_{i},\B{Y}-\B{Y}_{i} \rangle\quad i = 0,\cdots, t-1\nonumber\\
     & \sum_{(i,j)\in [k]\times [n]} \B{Y}_{i,j}\leq \ell\nonumber .
    \end{align}
The optimal objective value of~\eqref{outerMILP} is a lower bound (aka dual bound) for~\eqref{fcut}; and these lower bounds improve as the iterations progress (i.e., $t$ increases).    
As the feasible set of problem \eqref{fcut} is finite, after finitely many iterations $t$, an optimal solution to~\eqref{fcut} is found. The optimality gap (OG)  of the outer approximation can be calculated as $\text{OG}=(\text{UB}-\text{LB})/\text{UB}$ where $\text{LB}$ is the current (and the best) lower bound achieved by the piecewise approximation and $\text{UB}$ is the best upper bound for~\eqref{fcut} found so far. We summarize the procedure in Algorithm \ref{outer} where, `$\text{tol}$' denotes a pre-specified tolerance level.

    \begin{algorithm}[h]
\SetAlgoLined
$t = 1$\\
\While {$\text{OG} > \text{tol}$}{
$(\B{Y}_t,\eta_t)$ are solutions of \eqref{outerMILP}.\\
$F_{\text{best}}=\min_{i=0,\cdots,t-1} F(\B{Y}_i)$\\
$\text{OG} = (F_{\text{best}}-\eta_t)/F_{\text{best}}$\\
$t = t + 1$
}
 \caption{An outer approximation method to solve~\eqref{fcut}}
 \label{outer}
\end{algorithm}

 The optimization Problem \eqref{ftomin} is convex in $(\B{H}, \tilde{\B{W}})$ and we use an accelerated proximal gradient method~\citep{fista} to solve it.

\begin{rmk} We note that our algorithm is different from that of \citet{bertsimas2020sparse} who consider the sparse linear regression problem with an additional ridge regularization. \citet{bertsimas2020sparse} use an outer approximation algorithm to solve an equivalent convex integer program, with an explicit closed-form expression. In contrast, in our work, the function $F(\B{Y})$ is given (implicitly) by the optimal objective value of a convex optimization problem --- there is no closed form expression for $F(\B{Y}).$ Furthermore, our formulation of~\eqref{miqpform} uses the binary variable as a linear constraint --- this is different from~\cite{bertsimas2020sparse} where, the binary variable appears as a nonlinear expression within the objective function.  
 \end{rmk}
\subsubsection{Quantifying the effect of a MIP-based warm-start}As discussed earlier, we use MIP-techniques to obtain a solution to Problem~\eqref{infilambda}, which is subsequently used to initialize our coordinate descent algorithm (i.e., Algorithm~\ref{algorithm1}) for Problem~\eqref{archl0dual}. The MIP-approach provides an optimality certificate (via dual bounds) for Problem~\eqref{infilambda} and Proposition~\ref{OGprop} below presents a guarantee on using such a initialization.
 \begin{pro}\label{OGprop}
 Let $\B{H}^*_{\infty}$ be the solution returned by Algorithm~\ref{outer} with the optimality gap $\text{OG}$. Moreover, let $\B{H}^*_{\lambda}$ be the solution returned by Algorithm~\ref{algorithm1} initialized with $\B{H}^*_{\infty}$ and $\hat{\B{H}}_{\lambda}$ be the optimal solution to problem~\eqref{archl0dual}. Then,
\begin{align}
    \frac{ \Phi_{\lambda}(\B{H}^*_{\lambda}) - \Phi_{\lambda}(\hat{\B{H}}_{\lambda}) }{\Phi_{\lambda}(\B{H}^*_{\lambda})}\leq \frac{D(\B{X},\B{H}^*_{\infty})+\lambda(\text{OG})D({\B{H}}^*_{\infty},\B{X})}{\lambda(1-\text{OG})D({\B{H}}^*_{\infty},\B{X})}
\end{align}
 where recall that $\Phi_{\lambda}(\cdot)$ is the objective function in problem~\eqref{archl0dual}.
 \end{pro}
When $\lambda$ is sufficiently large, we have
 \begin{align}
    \frac{\Phi_{\lambda}(\B{H}^*_{\lambda}) - \Phi_{\lambda}(\hat{\B{H}}_{\lambda})}{\Phi_{\lambda}(\B{H}^*_{\lambda})}\leq \frac{\text{OG}}{1-\text{OG}} + \epsilon_{\lambda},
\end{align}
where, $\epsilon_{\lambda}$ is of the order $\mathcal{O}(1/\lambda)$.
Note that when our MIP algorithm delivers a near-optimal solution, OG will be close to zero --- hence, $\B{H}^*_{\lambda}$ will be a near-optimal solution to problem~\eqref{archl0dual}. The proposition presents theoretical support towards why our initialization framework appears to work well in practice---this is also seen in 
our numerical experiments in Section~\ref{numericalex}.

\subsection{Improving Algorithm~\ref{algorithm1} with Local Search}\label{lsalg}
As discussed in Section~\ref{cdalg}, the block CD method (Algorithm~\ref{algorithm1}) is guaranteed to deliver a stationary point for Problem~\eqref{constrainedtoslove}. We present some heuristics to improve the solution quality based on local search, drawing inspiration from the work of~\cite{beck,hazimeh2020fast} who also explore local search based updates for a different problem.

\noindent Once we are at a stationary point delivered by Algorithm~\ref{algorithm1}, our local search algorithm swaps a coordinate in the support of $\B{H}$ with a coordinate from outside the support. That is, a nonzero coordinate of $\B{H}$ is set to zero and a zero coordinate of $\B{H}$ is allowed to become nonzero. Then, the optimization is solely done on the coordinate entering the support. If this optimization leads to a lower objective value, we retain the new support. Mathematically, let $(\B{H}^*,\B{W}^*,\tilde{\B{W}}^*)$ be a feasible solution of problem \eqref{constrainedtoslove} with $\|\B{H}^*\|_0=\ell$. This solution can be an output of Algorithm \ref{algorithm1}. Suppose $(i_1,j_1)\in S(\B{H}^*)$ (here, $S(\B{H}^*)$ is the support of $\B{H}^*$) leaves the support and $(i_2,j_2)\notin S(\B{H}^*)$ enters the support. We perform an optimization on the coordinate $(i_2,j_2)$ of $\B{H}$ to decide whether this change in the support improves the objective function. Let $\B{E}^{i_1,j_1}$ be a matrix with all entries equal to zero except coordinate $(i_1,j_1)$ equal to one. We denote
$\B{H}=\B{H}^*-\B{H}^*_{i_1,j_1}\B{E}^{i_1,j_1}$ as the solution with coordinate $(i_1,j_1)$ removed from the support.
The candidate solution with the new support has the form $\B{H}+t\B{E}^{i_2,j_2}$ for $t\geq 0$. This leads to the following problem:
\begin{align}\label{constrainedtosloveonedim}
     \min_{\boldsymbol{W},\boldsymbol{\tilde{\boldsymbol{W}}},t} ~~ &  \|\boldsymbol{X}-\boldsymbol{W}\boldsymbol{H}-t\B{W}\B{E}^{i_2,j_2}\|_F^2+\lambda \|\boldsymbol{H}+t\B{E}^{i_2,j_2}-\tilde{\boldsymbol{W}}\boldsymbol{X}\|_F^2\\
     \text{s.t.} ~~ & t\geq 0,\boldsymbol{W} \geq 0,\tilde{\boldsymbol{W}}\geq 0\nonumber \\
     & \boldsymbol{W}\boldsymbol{1}_k=\boldsymbol{1}_m,\tilde{\boldsymbol{W}}\boldsymbol{1}_m=\boldsymbol{1}_k\nonumber
\end{align}
where, we are optimizing over $(\B{W}, \tilde{\B{W}}, t)$ for a given $(i_{2}, j_{2})$ and $(i_{1}, j_{1})$.
For a fixed value of $t$, the optimal values of $\B{W}$ and $\tilde{\B{W}}$ in~\eqref{constrainedtosloveonedim} are given as
\begin{align} 
    \hat{\B{W}}\in& \argmin_{\B{W}}~~ \|\B{X}-\B{W}(\B{H}+t\B{E}^{i_2,j_2})\|_F^2~~~~\text{s.t.}~~\B{W}\geq 0;~~ \boldsymbol{W}\boldsymbol{1}_k=\boldsymbol{1}_m \label{wupdate1} \\
    \hat{\tilde{\B{W}}}\in& \argmin_{\tilde{\B{W}}}~~ \|\B{H}+t\B{E}^{i_2,j_2} -\tilde{\B{W}}\B{X}\|_F^2~~~~\text{s.t.}~~\tilde{\B{W}}\geq 0;~~ \tilde{\boldsymbol{W}}\boldsymbol{1}_m=\boldsymbol{1}_k. \label{wupdate2}
\end{align}
Problems \eqref{wupdate1} and \eqref{wupdate2} are convex and can be efficiently solved by standard first order methods such as proximal gradient. 
Note that these first order methods also benefit from warm-starts available from prior estimates of $(\B{W}, \tilde{\B{W}})$. 
Once $\B{W}$ and $\tilde{\B{W}}$ are updated by \eqref{wupdate1} and \eqref{wupdate2}, the value of $t$ that minimizes \eqref{constrainedtosloveonedim} with $\B{W}=\hat{\B{W}}$ and $\tilde{\B{W}}=\hat{\tilde{\B{W}}}$ is:
\begin{equation}\label{tupdate}
t = \max\left\{\frac{\sum_{r=1}^ m \B{U}_{r,j_2}\B{W}_{r,i_2}-\lambda \B{V}_{i_2,j_2}}{\lambda+\|\B{W}_{.,i_2}\|_2^2},0\right\}
\end{equation}
 where $\B{U}=\B{X}-\B{W}\B{H}$ and $\B{V}=\B{H}-\tilde{\B{W}}\B{X}$.\\
\noindent 
We use an alternating optimization scheme where the three updates \eqref{wupdate1},\eqref{wupdate2} and \eqref{tupdate} are performed sequentially until convergence. These updates result in a descent method by construction, though there may not be a strict decrease in the objective value (in which case, the swap may not result in a better solution).  

\noindent In the discussion above, we assumed a fixed pair of indices $(i_{1}, j_{1})$ and $(i_{2}, j_{2})$. Ideally, we would like to try all possible choices of such indices and consider the one that leads to the maximal decrease in objective value (if any).
Instead of using an exhaustive search over all such pairs, we use a heuristic to select a suitable pair of indices.
We choose $(i_1,j_1)$ to be the smallest nonzero entry in $\B{H}^*$:
\begin{equation}\label{leaving}
    (i_1,j_1)\in\argmin_{(i,j)\in S(\B{H}^*)}\B{H}^*_{i,j}.
\end{equation}
For the pair $(i_{2}, j_{2})$ from outside the current support of $\B{H}$, we choose the coordinate of $\B{H}^*$ that has the smallest (most negative) gradient of the objective function:
\begin{equation}\label{entering}
    (i_2,j_2)\in\argmin_{(i,j)\in[k]\times[n]\setminus S(\B{H}^*)}\frac{\partial}{\partial \B{H}_{i,j}}\Psi(\B{H}^*,\B{W}^*,\tilde{\B{W}}^*).
\end{equation}
The overall procedure of local search is shown in Algorithm \ref{cwsalg}. 

\begin{algorithm}[!t]
\SetAlgoLined
Initialize with $\B{H}^*,\B{W}^*,\tilde{\B{W}}^*$.\\
 \While{not converged}{

\If {$\|\boldsymbol{H}^*\|_0<\ell$}{
Choose $(i_2,j_2)$ as in \eqref{entering}.\\
$\B{H}=\B{H}^*$\\
} \Else {
Choose $(i_1,j_1)$ as in \eqref{leaving} and $(i_2,j_2)$ as in \eqref{entering}.\\
$\B{H}=\B{H}^* -\B{H}^*_{i_1,j_1}\B{E}^{i_1,j_1}$\\
}
 \While{not converged}{
 Update $\B{W}$, $\tilde{\B{W}}$ and $t$ via \eqref{wupdate1}, \eqref{wupdate2} and \eqref{tupdate}.\\
 }
 \If {$\Psi(\B{H}+t\B{E}^{i_2,j_2},\B{W},\tilde{\B{W}})<\Psi(\B{H}^*,\B{W}^*,\tilde{\B{W}}^*)$} {
 $\B{H}^*=\B{H}+t\B{E}^{i_2,j_2}, \B{W}^*=\B{W}, \tilde{\B{W}}^*=\tilde{\B{W}}$
 }
}
 \caption{A Local Search improvement for Algorithm~\ref{algorithm1}}
 \label{cwsalg}
\end{algorithm}

\subsection{Tuning Parameter Selection}
Our SAA model has three tuning parameters:  rank of the factorization ($k$), sparsity level ($\ell$) and the archetypal regularization tuning parameter ($\lambda$). 
The problem of selecting the rank $k$ in a factorization arises in earlier work on NMF: For example, \citet{Owen_bi-cross-validationof,bro2008cross} discuss cross-validation-type approaches to guide the selection of $k$.  
The choice of $\ell$ controls the sparsity of the latent factor $\B{H}$---this can improve the interpretability of the model while potentially decreasing the accuracy and robustness of the model (especially, if the underlying factors are not sparse), as discussed in Section~\ref{SAA}. In many cases however, the choices of $k,\ell$ are made by the practitioner depending upon problem-specific considerations.  
 Assuming $k$ and $\ell$ are given (or have been pre-specified by the user), we consider the problem of selecting $\lambda$. To this end, we propose a validation-based scheme. Let $\mathcal{S}\subseteq [m]$ be a held-out subset of the data. Suppose $\B{H}^*_{\lambda}$ be the estimator from problem~\eqref{archl0dual} obtained on the data points $i\in[m]\setminus \mathcal{S}$. We choose the value of $\lambda$ that minimizes the validation loss $\text{V}_{\lambda}(\mathcal{S})$, which is defined as
\begin{equation}\label{validloss}
    \text{V}_{\lambda}(\mathcal{S})=\sum_{i\in\mathcal{S}} D(\B{X}_{i,.},\B{H}^*_{\lambda}).
\end{equation}
 Intuitively, the validation loss above captures the fact that a robust solution can describe the data better than other solutions. Proposition~\ref{loothm} formalizes this intuition.  
 \begin{pro}\label{loothm}
 Let $\B{H}^{-\mathcal{S}}$ be the estimator trained on the training data (with the validation set $\mathcal{S}$ removed). Moreover, let 
$$d(\boldsymbol{X})=\max_{i\in[m]}\min_{j\in[m]\setminus\mathcal{S}}\|\boldsymbol{X}_{i,.}-\boldsymbol{X}_{j,.}\|_2.$$
Then,
\begin{multline*}
    D(\boldsymbol{X}_{\mathcal{S}},\boldsymbol{H}_{(-i)})^{1/2}\leq \sqrt{|\mathcal{S}|}\bigg[ d(\boldsymbol{X}) + \max_{j\in[m]}\|\boldsymbol{Z}_{j,.}\|_2 +\\\min\left\{ k\|\boldsymbol{H}_0\|_F + \mathcal{L}(\boldsymbol{H}^{-\mathcal{S}},\boldsymbol{H}_0)^{1/2},\mathcal{L}(\boldsymbol{H}_0,\boldsymbol{H}^{-\mathcal{S}})^{1/2}\right\}\bigg]
\end{multline*}
where $\B{X}_{\mathcal{S}}\in\R^{|\mathcal{S}|\times n}$ is the validation set.
 \end{pro}
 Proposition~\ref{loothm} shows that a more robust solution leads to a lower validation loss $\text{V}_{\lambda}$, and we seek to choose a $\lambda$ that results in the smallest validation loss.

\section{Numerical Experiments}\label{numericalex}
In this section, we discuss results of our numerical experiments on synthetic and real data and investigate how our framework performs in practice. Our experiments are done on a computer equipped with \verb+Intel(R) Core(R) i7 6700HQ CPU @ 2.60GHz+, running \verb+Microsoft(R)+ \verb+Windows(R) 10+ and using 16GB of RAM. We implemented all of our algorithms in \verb+Julia+ and we use \verb+Gurobi(R)+ to solve MILPs arising in our initialization scheme. An implementation of our framework in Julia is available at:

~~~~~~~~~~~~~\url{https://github.com/kayhanbehdin/SparseAA}.

\subsection{Synthetic Data}\label{synthetic}
In this section, we consider synthetic data to validate the theory we developed in Sections \ref{SAA} and \ref{saaalgorithm}; and gather further insights into the operating characteristics of SAA.

\smallskip

\noindent {\bf Data set generation:} The entries of $\boldsymbol{H}_0$ are drawn iid from $\text{Unif}[0,1]$ and $20\%$ of entries are set to zero at random to produce a matrix with at most $0.8nk$ nonzero entries. Independent of $\B{H}_0$, the entries of $\boldsymbol{W}_0$ are drawn iid from $\text{Unif}[0,1]$ and each row is normalized to sum to one. The noiseless data is $\boldsymbol{X}_0=\boldsymbol{W}_0\boldsymbol{H}_0$ and the noisy data is produced as $\boldsymbol{X}=\max\{\boldsymbol{X}_0+\boldsymbol{Z},0\}$ where entries of $\boldsymbol{Z}$ are from an independent Gaussian ensemble with mean zero and variance $\sigma^2_{z}$.

\subsubsection{Understanding robustness}\label{synthnumundrob}
We compare the performance of our algorithm with other sparse NMF algorithms in terms of robustness of the solution. 
We consider algorithms by \citet{10.1093/bioinformatics/btm134} (shown as KP), \citet{peharz2012sparse} (shown as PP) and \citet{hoyer2004non} (shown as Hoyer). To quantify the usefulness of having a sparse solution over a dense one, we compare our algorithms with the archetypal regularization framework of~\citet{javadi2019nonnegative}, which delivers a dense solution (shown as AA).
We set $m=200,k=15,n=5000, \lambda=1$. 
We consider  settings for our ablation studies: 
\begin{compactitem}
        \item We keep the sparsity level $\ell$ fixed and change the noise level $\sigma_z$
    \item We keep $\sigma_z$ fixed and vary the sparsity level $\ell$
    \item We keep all parameters fixed except $k$ and explore how changing the rank can affect the performance.
    \item We investigate the effect of varying $\lambda$ in our framework
\end{compactitem}

\smallskip

\noindent {\bf{Robustness versus varying $\sigma_{z}$}:} First, we set $\ell/nk=0.8$ and set the tuning parameters for different algorithms to get solutions that have $0.8nk$ nonzeros. This can be thought of as a case where the sparsity level is well-specified. Specifically, as KP considers an $\ell_1$ regularized version of NMF, we start with a small value of the $\ell_1$ regularization parameter and gradually increase it till we reach the target sparsity-level. PP uses an $\ell_0$ constrained version of original NMF (without archetypal regularization) and we set the constraint such that a solution has at most $\ell$ nonzeros. 
The sparsity constraint of Hoyer is set such that the result has $\ell$ nonzeros. We vary $\sigma_z$ and plot the average value of normalized weak ($\mathcal{L}(\B{H}_0,\B{H})/\|\B{H}_0\|_F^2$) and strong ($\mathcal{L}(\B{H},\B{H}_0)/\|\B{H}_0\|_F^2$) robustness quantities\footnote{Note that the robustness error values are reported relative to the baseline estimator $\B{H}=\B{0}$.} over 10 independent repetitions. We also report the average runtime (in seconds) for different methods. The runtime for our approach corresponds to the continuation framework. To compute the initialization which is a one-time cost, we cap the MIP procedure to 3 minutes, and do not include this runtime. The results for this scenario are shown in Figure~\ref{fig:synt1}. We observe that the SAA solutions almost always outperforms other algorithms in terms of strong and weak robustness metrics. Specifically, the difference between SAA and other algorithms is most noticeable when the noise is small. This is expected as other (sparse) algorithms in our experiments do not use any regularization that results in robustness. Moreover, solutions of SAA become less robust as noise is increased, as anticipated by Theorem~\ref{robustnessthm} and Proposition~\ref{penalizedpro}. In addition, SAA has even smaller error compared to AA, while the AA solution is almost fully dense. This shows that the sparsity constraint helps us to obtain a better solution as the underlying model is also sparse. In terms of runtime, most algorithms are generally of the same order of magnitude, but we note that our method obtains a path of solutions due to our continuation scheme, and therefore, can be faster for a given value of $\lambda$.

\begin{figure*}[t!]

     \centering
     \begin{tabular}{lclclc}

&  Weak Robustness & & Strong Robustness & & Runtime \\
     \rotatebox{90}{\tiny~~~~~~~~~~$\mathcal{L}(\B{H}_0,\hat{\B{H}})/\|\B{H}_0\|_F^2$}& \includegraphics[width=0.26\linewidth,trim =1cm 0cm 1cm 0cm, clip = true]{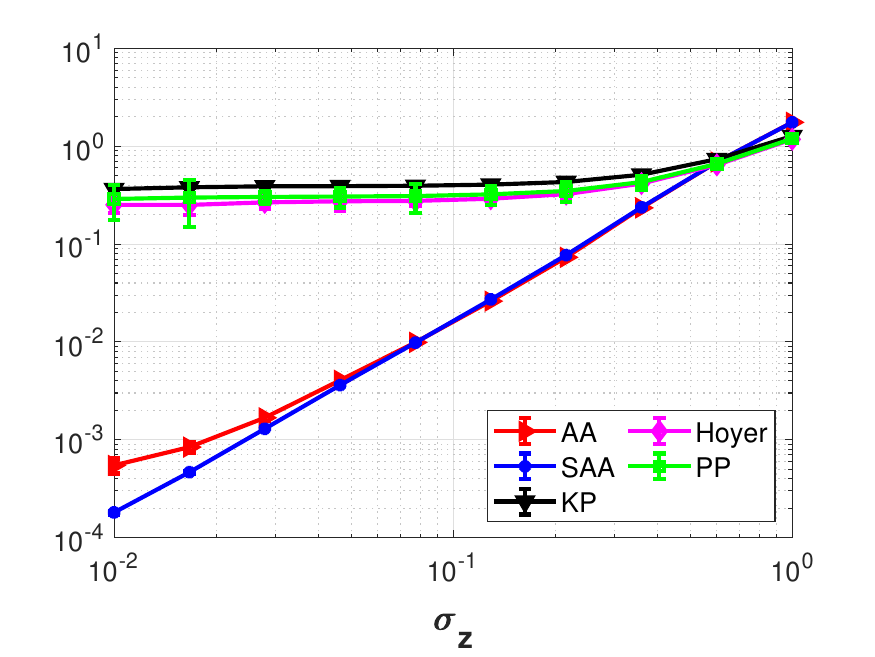}& 
 \rotatebox{90}{\tiny~~~~~~~~~~$\mathcal{L}(\hat{\B{H}},\B{H}_0)/\|\B{H}_0\|_F^2$}&        \includegraphics[width=0.26\linewidth,trim =1cm 0cm 1cm 0cm, clip = true]{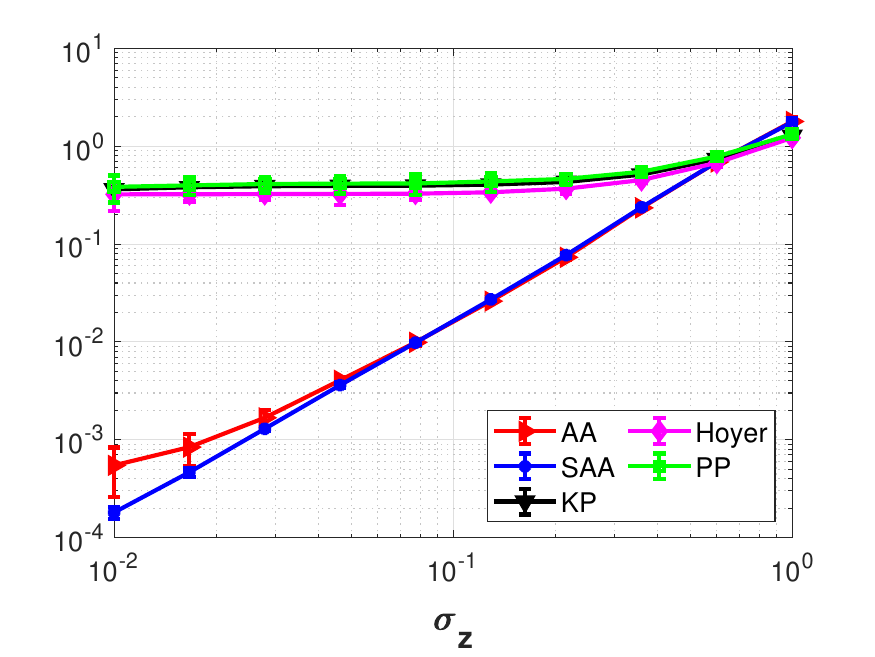}
& &\includegraphics[width=0.26\linewidth,trim =1cm 0cm 1cm 0cm, clip = true]{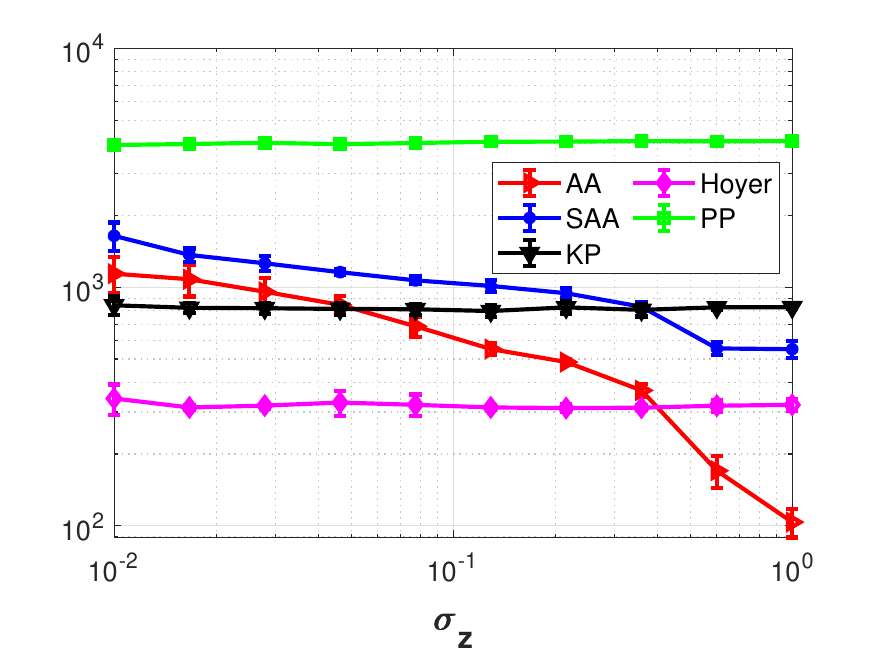} \\
\end{tabular}
        \caption{\small  Effect of varying noise on the performance of different algorithms for the well-specified case in Section~\ref{synthetic}. }
        \label{fig:synt1}
\end{figure*}

\begin{figure*}[t!]

     \centering
     \begin{tabular}{lclclc}

&  Weak Robustness & & Strong Robustness & & Runtime \\
     \rotatebox{90}{\tiny~~~~~~~~~~$\mathcal{L}(\B{H}_0,\hat{\B{H}})/\|\B{H}_0\|_F^2$}& \includegraphics[width=0.26\linewidth,trim =1cm 0cm 1cm 0cm, clip = true]{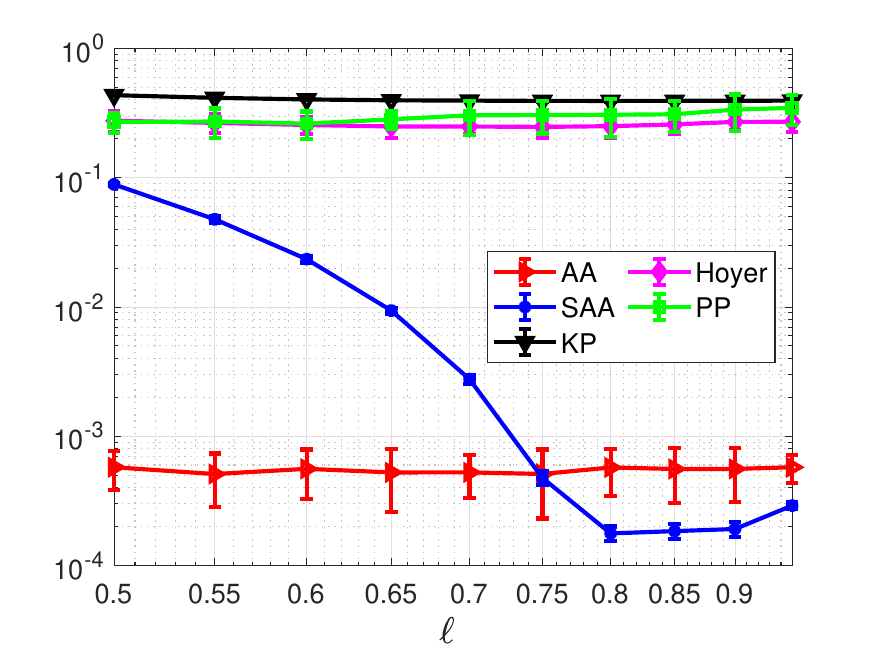}& 
 \rotatebox{90}{\tiny~~~~~~~~~~$\mathcal{L}(\hat{\B{H}},\B{H}_0)/\|\B{H}_0\|_F^2$}&        \includegraphics[width=0.26\linewidth,trim =1cm 0cm 1cm 0cm, clip = true]{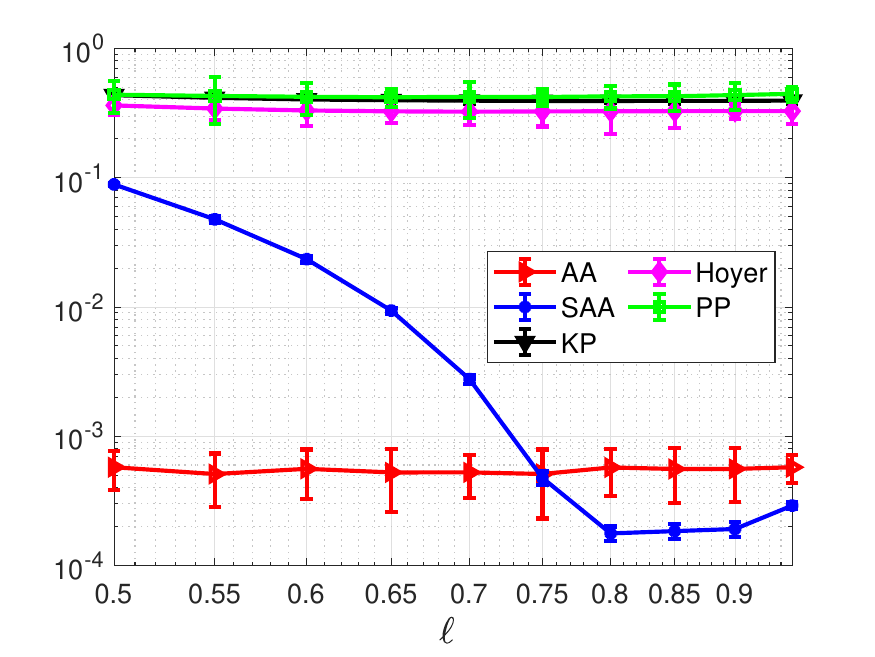}
& &\includegraphics[width=0.26\linewidth,trim =1cm 0cm 1cm 0cm, clip = true]{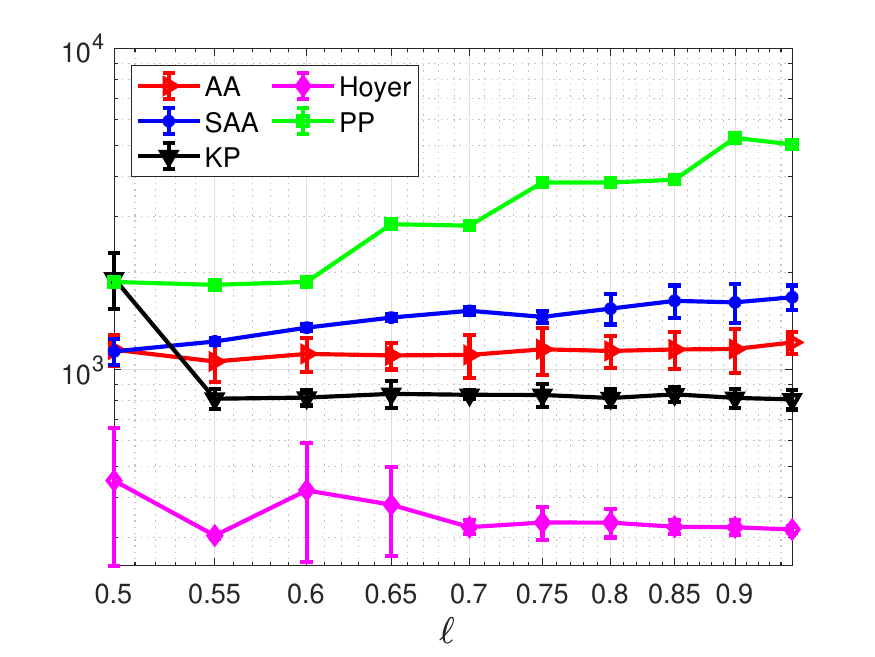} \\
\end{tabular}
        \caption{\small   Effect of varying sparsity $\ell$ on the performance of different algorithms in Section~\ref{synthetic}. }
        \label{fig:synt3}
\end{figure*}

\begin{figure*}[t!]

     \centering
     \begin{tabular}{lclclc}
  
&  Weak Robustness & & Strong Robustness & & Runtime \\
     \rotatebox{90}{\tiny~~~~~~~~~~$\mathcal{L}(\B{H}_0,\hat{\B{H}})/\|\B{H}_0\|_F^2$}& \includegraphics[width=0.26\linewidth,trim =1cm 0cm 1cm 0cm, clip = true]{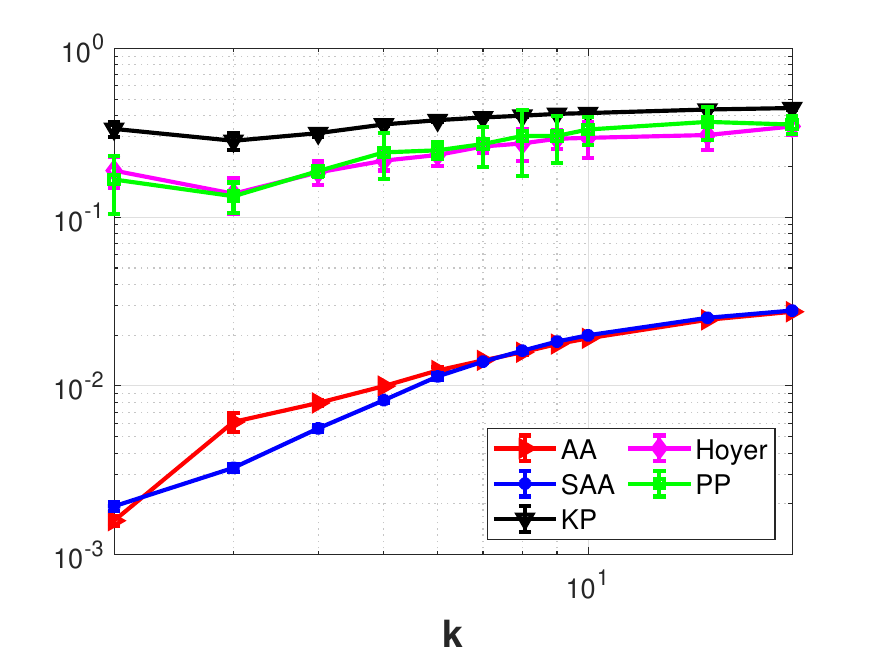}& 
 \rotatebox{90}{\tiny~~~~~~~~~~$\mathcal{L}(\hat{\B{H}},\B{H}_0)/\|\B{H}_0\|_F^2$}&        \includegraphics[width=0.26\linewidth,trim =1cm 0cm 1cm 0cm, clip = true]{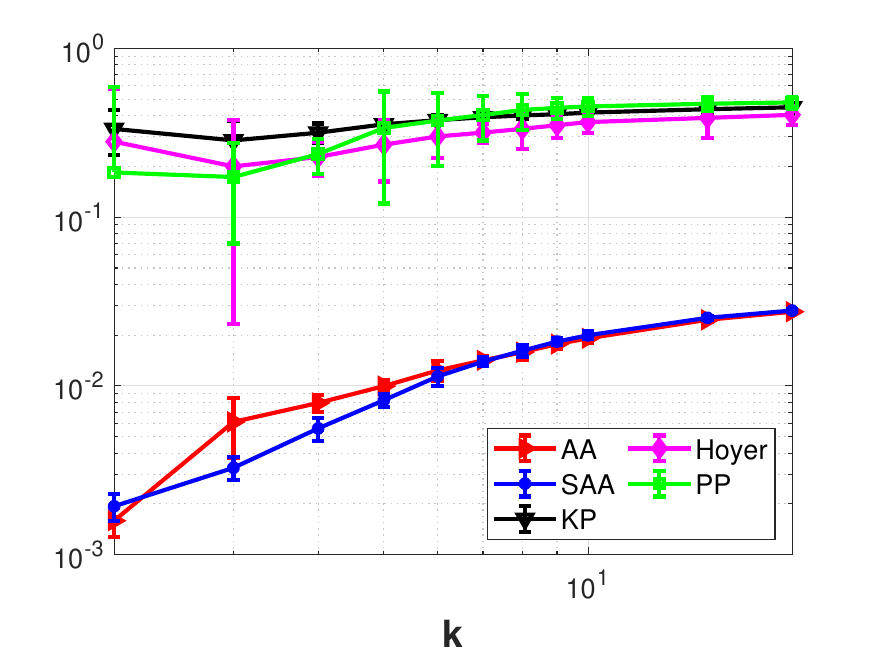}
& &\includegraphics[width=0.26\linewidth,trim =1cm 0cm 1cm 0cm, clip = true]{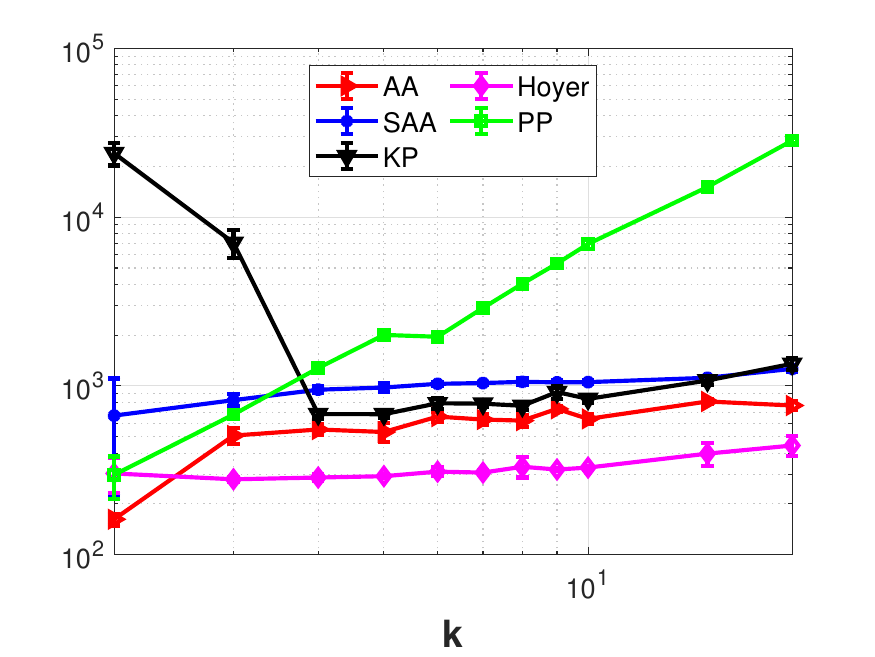} \\
\end{tabular}
        \caption{\small  Effect of varying $k$ on the performance of different algorithms in Section~\ref{synthetic}. }
        \label{fig:synt4}
\end{figure*}

\noindent {\bf{Robustness versus varying $\ell$}:} To compare the performance of different algorithms for varying values of $\ell$, we do another set of experiments. We consider a setup similar to the previous experiment. However, we fix $\sigma_z=0.01$ and change the value of $\ell$ (while keeping the underlying model sparsity budget fixed at $0.8nk$). This shows how well different algorithms can deal with sparsity. The results for this case are shown in Figure~\ref{fig:synt3}. We see that SAA outperforms other sparse algorithms and is the most (strongly and weakly) robust method among sparse ones. In addition, we observe that as $\ell$ is decreased, the solution becomes less robust as anticipated by Theorem \ref{robustnessthm}. Another interesting comparison can be done between SAA and AA. It can be observed that when $\ell/nk\geq 0.75$, SAA outperforms AA, and the best robustness is achieved for $\ell/nk=0.8$ which corresponds to the true sparsity level.

\noindent {\bf{The effect of varying $k$}:} 
For the next set of experiments, we set $\ell=0.8nk$ and $\sigma_z=0.1$, with other parameters as before. However, we vary the value of $k$ (for the underlying model and the algorithm) to see the effect of rank on the results. We report the results for this case in Figure~\ref{fig:synt4}. Overall, we see that increasing $k$ results in more error as expected. Specifically, from this figure we can observe that both robustness quantities for SAA scale overall as $\mathcal{L}=\mathcal{O}(k^{\alpha})$ where $1.5\leq\alpha\leq 2$. Comparing this to what our theory predicts in Proposition~\ref{sepinexact}, we see that the scaling of the robustness quantities in our theory and experiments match, showing that our theoretical bounds are tight. 

\noindent {\bf{The effect of varying $\lambda$}:} To show how changing the parameter $\lambda$ affects the SAA algorithm, we conduct another set of experiments. We set $\sigma_z=1,\ell=0.5nk$ and choose $\lambda$ from a logarithmic grid with 100 points between $0.0625$ and $12.5$. In addition, we draw a validation set of size $m_{\text{validation}}=m=200$ to calculate the validation loss defined in~\eqref{validloss}. The average results for this setting are shown in Figure~\ref{fig:synt_valid}. The left panel in this figure shows the validation loss, the middle panel shows the weak robustness quantity and the right panel shows the strong robustness quantity. As it can be seen, all three quantities are minimized for some value of $\lambda$ in the interval $[1, 2]$ which is close to the value of $\lambda$ we consider.

\begin{figure*}[t!]
     \centering
     \begin{tabular}{lclclc}
&  Validation Loss & & Weak Robustness & & Strong Robustness\\
     \rotatebox{90}{\tiny ~~~~~~~~~~~~~~~~~~$\text{V}_{\lambda}$}& \includegraphics[width=0.26\linewidth,trim =0.7cm 0cm 1.3cm 0cm, clip = false]{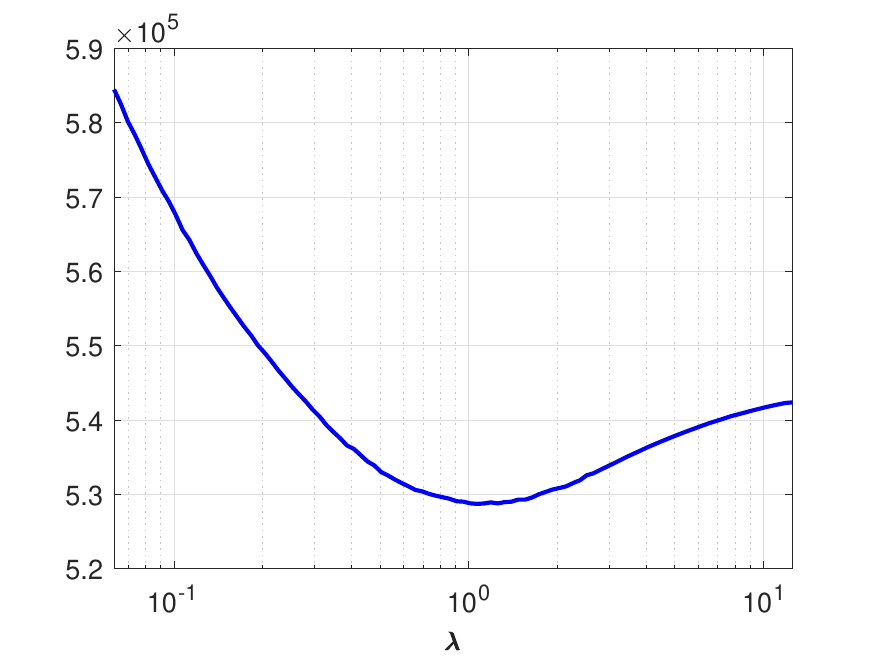}& 
 \rotatebox{90}{\tiny~~~~~~~~~~~~~~~$\mathcal{L}(\B{H}_0,\hat{\B{H}})$}&        \includegraphics[width=0.26\linewidth,trim =0.7cm 0cm 1.3cm 0cm, clip = true]{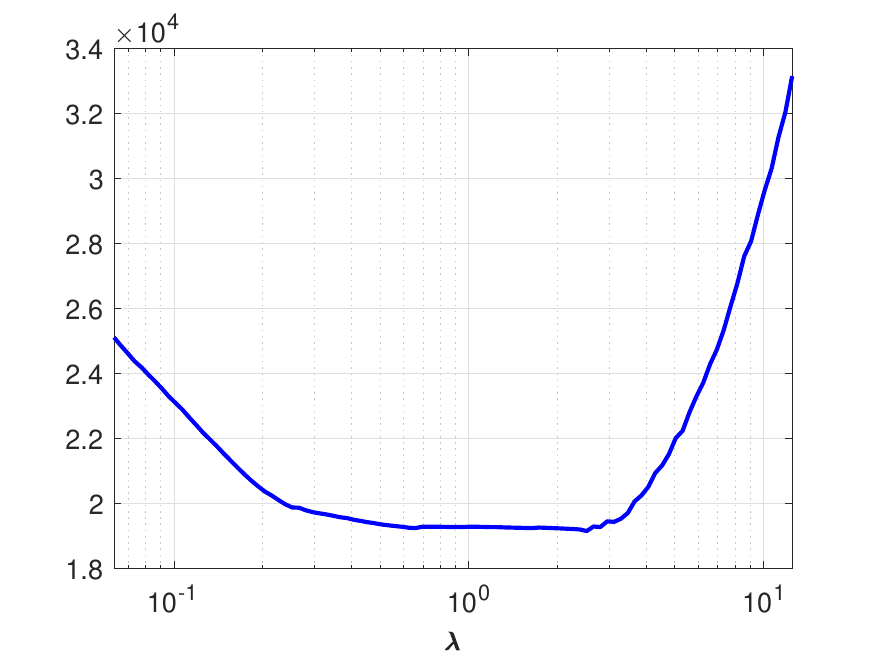} &
 \rotatebox{90}{\tiny~~~~~~~~~~~~~~~$\mathcal{L}(\hat{\B{H}},\B{H}_0)$}&        \includegraphics[width=0.26\linewidth,trim =0.7cm 0cm 1.3cm 0cm, clip = true]{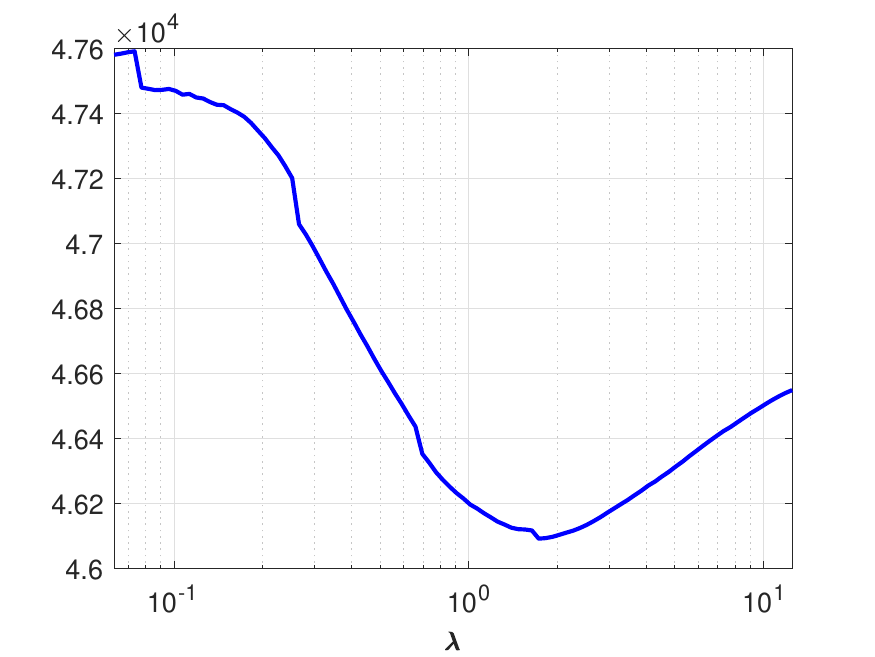} \\ 
\end{tabular}
        \caption{\small Figure showing the profile of 
       validation loss [left], and  robustness quantities [middle and right panels respectively] for varying values of $\lambda$. We use the synthetic data discussed in Section \ref{synthetic}.}
        \label{fig:synt_valid}
\end{figure*}

\subsubsection{Usefulness of MIP-based Initialization and Local Search}\label{initvsobj} We perform numerical experiments to show the usefulness of the initialization procedure in Section \ref{initsection} and the local search scheme in Section \ref{lsalg}. We show these algorithms improve upon a baseline initialization. We set $m=200, k=20, n=12000$ and $\lambda=1$ and vary $\sigma_z$ and $\ell$ (see Table~\ref{costobj}). We consider three cases: 
\begin{compactitem}
    \item[(i)] Algorithm~\ref{algorithm1} initialized with $\B{H}=\B{0}$ (shown as Zero in Table~\ref{costobj}).
    \item[(ii)] Algorithm~\ref{algorithm1} initialized with the successive projections initialization of~\citet{javadi2019nonnegative}, projected on the set of $\ell$-sparse matrices (shown as Projections).
    \item[(iii)] Algorithm~\ref{algorithm1} with MIP initialization (see Section~\ref{initsection}) and using warm-start continuation over 8 values of $\lambda$ on a logarithmic scale from 30 to 1. (Shown as SAA in Table~\ref{costobj}).
    \item[(iv)] Improving solution from (iii) with local search discussed in Section~\ref{lsalg}. This is denoted by SAA+LS in Table~\ref{costobj}. 
\end{compactitem}

\smallskip

\noindent The maximum runtime of Algorithm \ref{outer} is capped to 20 minutes 
and we retain the best solution. 
The average final cost function achieved by three methods explained above (over 5 independently generated data sets) are reported in Table \ref{costobj}. As it can be seen, our initialization scheme achieves a significantly lower objective value compared to the baseline (Zero) while being computationally feasible for such data size. In addition, our local search algorithm can improve the objective value as well as the support in a reasonable time (the number of changes in support for SAA+LS is reported in the parentheses). Moreover, the SAA framework usually outperforms the Projections method in most cases by a significant margin. However, Projections also perform better than Zero. In addition, when the noise is high, Projections performs closer to our SAA method and even better in one case. This suggests that for very large problems where the MIP algorithm might not be feasible, the Projections methods can be a good alternative compared to a simple Zero initial solution. We also note that in general, the runtime of the MIP algorithm~\ref{outer} can be further capped in cases where runtime is of essence, allowing for obtaining good (but possibly sub-optimal) initial solutions quickly.

\begin{table}[h!]
\centering
\begin{tabular}{ |c|c|c|c|c| } 
\hline
 $\sigma_z$ & Method &  $\ell/nk=0.5$ & $\ell/nk=0.65$ & $\ell/nk=0.8$ \\
\hline
\multirow{4}{*}{0.01}& Zero &  47785 & 42123 & 37566 \\
& Projections & 34620 & 33692 &  31254 \\
& SAA & 32193 & 29766 &  27139 \\
&SAA+LS &   \textbf{31867 (12.8)} &  \textbf{29729 (13.4)} & \textbf{27091 (20.2)} \\
\hline

\multirow{4}{*}{0.1}&Zero &  69366 & 63776 & 59276\\
& Projections & 58042 & 54028 &  50316 \\
& SAA &  54784 & 52253 & 49228 \\
&SAA+LS  & \textbf{54465 (18.4)} &  \textbf{52159 (19.8)} & \textbf{49131 (24)} \\
\hline

\multirow{4}{*}{0.5}&Zero & 448088 & 440974 &  434590\\
& Projections & 428025 & 421931 &  \textbf{415053} \\
& SAA & 426238 & 420134 & 416901 \\
&SAA+LS &  \textbf{424776 (22.6)} & \textbf{419932 (26.8)}  & 416628 (32)  \\
\hline
\end{tabular}
\caption{\small Comparison of zero initialization, SAA and SAA+LS in Section \ref{synthetic}. The number in parentheses for SAA+LS shows the number of changes in the support after local search. }
\label{costobj}
\end{table}

\subsubsection{Illustrating the joint effect of AA and sparsity}\label{julia-illus-sec}
Finally, we present an illustrative example here that sheds further light into why archetypal regularization and sparsity work well together. We set $m=200,n=2,k=3,\ell=2$ and take the underlying archetypes to be $(0,0),(0,1)$ and $(1,0)$. The results for this case are shown in Figure~\ref{fig:julia-i}. First, we see in the AA solution, by increasing $\lambda$ the solutions get closer to the data, as expected. However, at the same time, we see that AA solutions are not sparse and although increasing $\lambda$ pushes the solutions towards the data, AA makes mistakes and cannot capture the true noiseless archetypes due to lack of sparsity. In contrast, when SAA is used, SAA recovers the correct pattern of archetypes, and therefore, by choosing a good value of $\lambda$, SAA can have small error. This illustrates how sparsity in conjunction with archetypal regularization helps to make the problem more well-defined.

\begin{figure*}[t!]

     \centering
\begin{tabular}{cc}

  SAA &  AA  \\
      \includegraphics[width=0.3\linewidth,trim =1cm 0cm 1cm 0cm, clip = true]{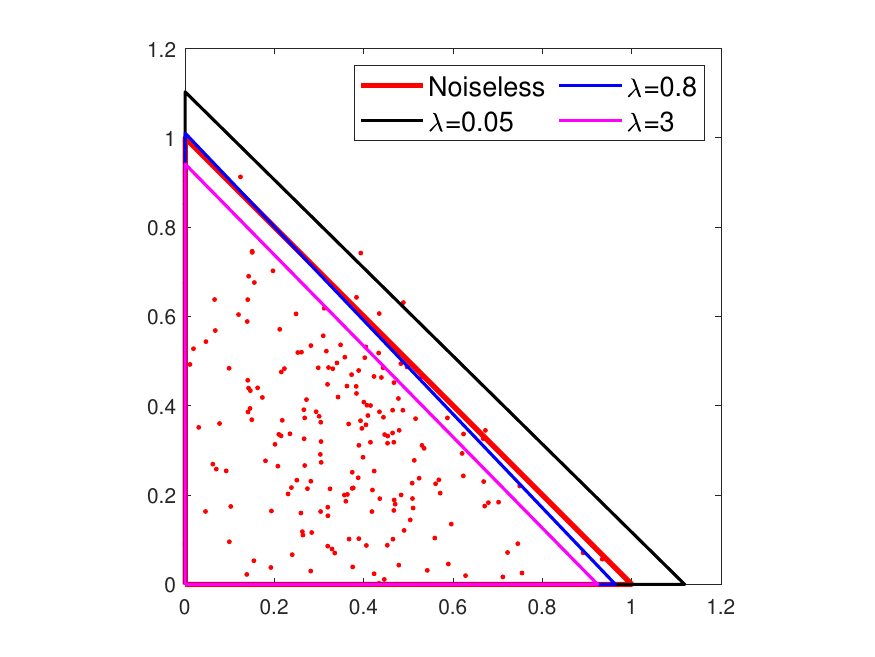}& 
\includegraphics[width=0.3\linewidth,trim =1cm 0cm 1cm 0cm, clip = true]{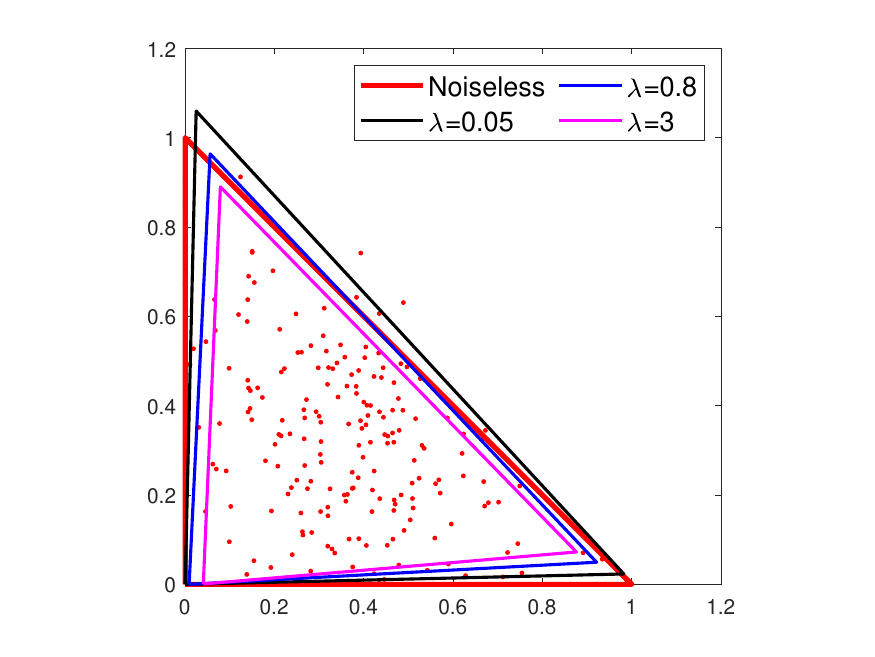} \\
\end{tabular}
        \caption{\small   The illustrative example from Section~\ref{julia-illus-sec} }
        \label{fig:julia-i}
\end{figure*}

\subsection{The Face Data Set \texorpdfstring{\citep{Samaria94parameterisationof}}{}}\label{faceexp}
A classical application of sparse NMF is in face detection and recognition \citep{hoyer2004non}. The goal is to obtain a low-rank representation of a data set of human faces under different lighting and shadow conditions and also different angles of photography. \citet{hoyer2004non} show the effect of sparsity in finding such representations of the data. In particular, they show that sparse NMF leads to part-based representations where each factor represents one part of the face. Here, we are interested in finding the effect of AA as well as the combined effect of AA and sparsity. We use the AT\&T database of faces \citep{Samaria94parameterisationof} which consists of 40 different people and 10 different photos of each person, 400 images in total. Each image is a grayscale $92\times 112$ image, which is converted to a vector of length $10304$. We then concatenate these $400$ vectorized images into one matrix of size $400\times 10304$ matrix. We consider $k=25$ (following \citet{hoyer2004non}) for this data set and do the factorization based on problem (\ref{archl0dual}) for different values of $\lambda$ and $\ell$. The estimated representations of the data (rows of resulting $\B{H}$ which are reshaped into $92\times 112$ images) are shown in Figure \ref{fig:faces}. We use the MIP initialization and continuation framework (over 8 values of $\lambda$) in Section \ref{initsection}. In the rest of this section, we discuss the connections between the robustness theory we developed and our numerical results.

\begin{figure}[h!]
     \centering
\begin{tabular}{lcc}
    &$\lambda=0$ (No AA) & $\lambda=0.4$ (AA)\\ 
       \rotatebox{90}{$~~~~\ell/k=10304$ (No sparsity)}& \includegraphics[width=0.3\linewidth]{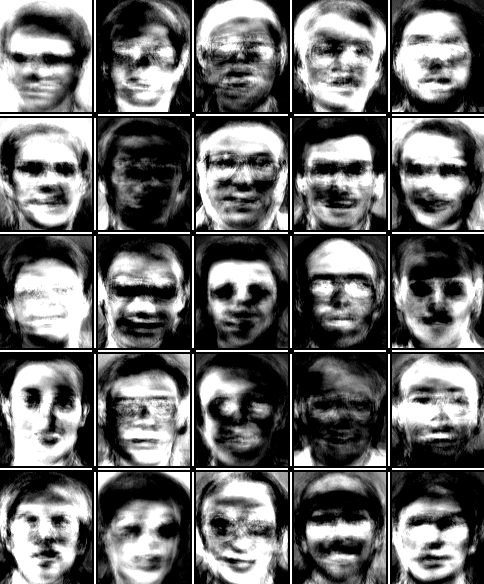}& 
        \includegraphics[width=0.3\linewidth]{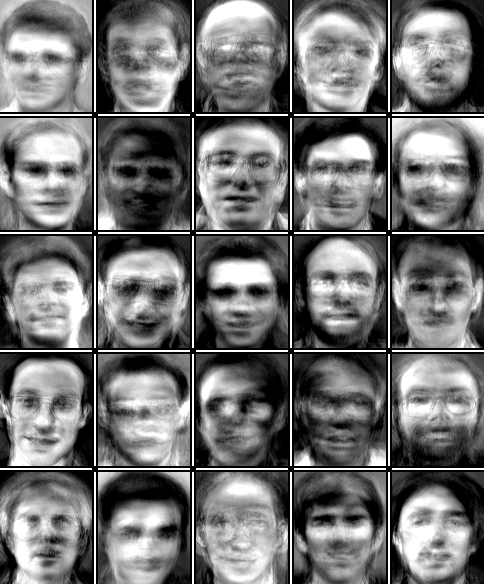}\\
        & (a)  & (b)  \\
         \rotatebox{90}{~~~~~~~~~~$\ell/k=4000$ (Sparse)}& \includegraphics[width=0.3\linewidth]{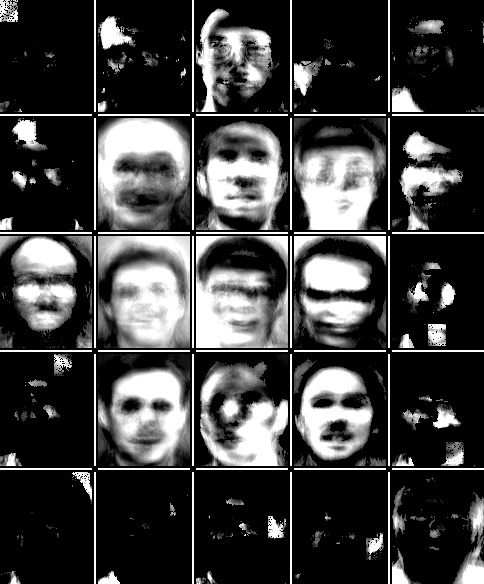}& 
        \includegraphics[width=0.3\linewidth]{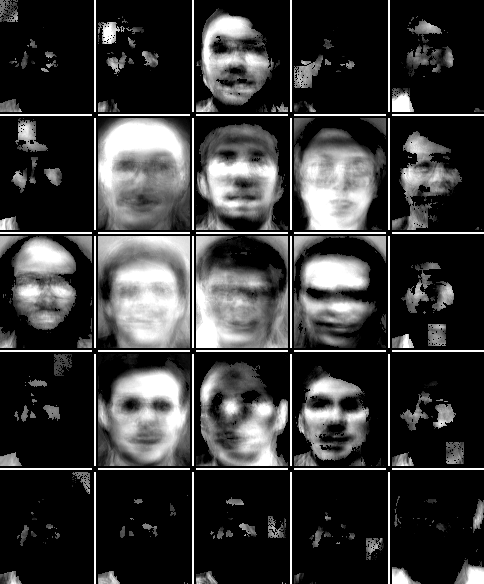}\\
        &(c)  & (d) 
\end{tabular}
        \caption{{\small The resulting face images in Section \ref{faceexp}: (a) $\lambda=0$, $\ell/k=10304$ (b) $\lambda=.4$, $\ell/k=10304$ (c) $\lambda=0$, $\ell/k=4000$ (d) $\lambda=.4$, $\ell/k=4000$}}   \label{fig:faces}
     \end{figure}

\noindent We first explore the difference between AA and basic NMF. By comparing Figures \ref{fig:faces} (a), (b), we deduce that as $\lambda$ is increased, the resulting factors appear to become more similar to human faces---making it easier to recognize the people in the data set. As $\lambda$ is decreased, the results become more abstract and the images do not resemble human faces anymore. Considering that the theory we developed in Section \ref{SAA} holds under mild assumptions, the representation achieved from Problem~\eqref{archl0dual} with $\lambda=0.4$ more likely corresponds to a robust solution---that is, the solution is closer to the underlying representation (we do not expect $\lambda=0$ to result in  robustness). Intuitively, we expect the underlying model that produces the face images to resemble the people in the data set---this suggests that the solution achieved by $\lambda=0.4$ is more robust.

\noindent As discussed by \citet{hoyer2004non}, adding sparsity to NMF produces part-based representations of the data. This can be also seen in our experimental results in Figure~\ref{fig:faces}. By forcing the solution $\B{H}$ to be sparse, we notice that each set of solutions (in Figures \ref{fig:faces} (c), (d)) includes two groups of factors. The first group consists of complete faces (columns and rows 2 to 4 in each set) and the second group consists of parts of a face (the border factors). In fact, this can be interpreted as each face being a combination of an overall shape of a human face and additional details arising from different parts of the face. The factors that contain a complete face in Figures  \ref{fig:faces} (c), (d) represent the overall shape of a human face, while other factors represent different parts of a face, like the forehead, cheeks, eyes and also the background of images. Once again, considering that our theoretical development in Section~\ref{SAA} is valid in the sparse NMF case, we expect the factors recovered by $\lambda=0.4$ to be more robust. Our hypothesis is the solution with $\lambda=0.4$ is more robust as factors in Figure \ref{fig:faces} (d) appear to more closely resemble human faces (compare the central columns in Figures  \ref{fig:faces} (c), (d)).

\subsection{Cancer Gene Expression Example~\texorpdfstring{\citep{ramaswamy2001multiclass}}{}}\label{genedata}
It is well-known that a primary goal of AA is to find a few representative points for a collection of data points. These representative points are useful in cluster analysis, where data points are put into a few clusters based on a suitable similarity/dissimilarity measure. In AA, each archetype (i.e., a row of $\B{H}$) can be considered as a cluster center and data points are assigned to clusters based on their proximity to different archetypes. As a result, each row of matrix $\B{H}$ (or each archetype) is considered as a center and each data point is assigned to the closest row of $\B{H}$ \citep{MORUP201254}.

\noindent An important application area of sparse NMF for clustering is in computational biology. Specifically, this problem has been considered by \citet{10.1093/bioinformatics/btm134} where the authors provide biological interpretations of sparsity in the context of NMF and do an extensive analysis of sparse NMF for microarray data. Here, we are interested in the clustering performance of our method. To this end, we consider a real data set: the 14 Cancers Gene Expression data set \citep{ramaswamy2001multiclass}. This data set consists of gene expression data of 198 samples and 14 different types of tumors. There are 16,063 features in the data set, however, the data is not nonnegative. Therefore, we use a trick introduced by \citet{kim2003subsystem} to transform the data: each feature is divided into two new features where one contains nonnegative coordinates (and zero elsewhere) and the other one contains the absolute value of negative coordinates (and zero elsewhere). Consequently, this leads to 32,126 nonnegative features---the data matrix is given by $\B{X}\in\mathbb{R}_{\geq 0}^{198\times 32126}$. The rank of the factorization is 14, the number of different types of tumors in the data set. The $i$-th data point ($i\in[198]$) belongs to the cluster $j_{i}$: 
$$j_i=\argmin_{j\in[14]}\|\B{X}_{i,.}-\B{H}_{j,.}\|_2$$
where $\B{H}$ is the resulting matrix of archetypes. To compare the performance of different algorithms (discussed below), we use two metrics, Purity and Entropy \citep{kim2008sparse}. Let for $i\in[198]$, $j_i^*\in[14]$ denote the true cluster of point $i$ and $j_i\in[14]$ denote the estimated cluster for the same point. Let $m_r^u$ be the number of samples that belong to the true cluster $u$ but are estimated to be in cluster $r$. Equivalently,
$$m_r^u = |\{i\in[m]: j_i^*=u, j_i=r \}|.$$
The metrics Purity and Entropy are defined as
$$\text{Purity} = \frac{1}{m}\sum_{r=1}^k \max_{u\in[k]} m_r^u~~~\text{and}~~~\text{Entropy} = -\frac{1}{m\log_2 k}\sum_{r=1}^k \sum_{u=1}^k m_r^u\log_2 \frac{m_r^u}{m_r}$$
where $m_r=\sum_{u=1}^k m_r^u$. A larger value of Purity and a smaller value of Entropy imply a better clustering performance. \\
The results of clustering performance of different algorithms is reported in Table \ref{cluster}. SAA is our proposed framework (we use MIP initialization with  continuation over 8 values of $\lambda$ from 30 to 1 as in Section \ref{initsection}), AA is the algorithm proposed by \citet{javadi2019nonnegative} which does not enforce any sparsity. KP and PP are as introduced before. Chen is the exact AA of~\citet{Breiman} with the algorithmic approach of~\citet{chen2014fast}. We also include Kmeans in our experiments as a baseline for the clustering performance. As it can be seen, among algorithms that enforce sparsity, SAA performs the best in terms of clustering. In fact, SAA is at par with Kmeans, while providing a solution that is two times more sparse. AA has the best clustering performance, however, it fails to provide a sparse solution. Other algorithms provide sparse solutions, but their clustering performance is not as good as SAA. In our experiments in this section, all NMF-based methods terminated in less than a minute.

\begin{table}[h!]
\centering
\begin{tabular}{ |c|c|c|c|c|c|c|c| } 
\hline
  & SAA & AA & Chen& Kmeans & KP & PP &  Hoyer  \\
\hline
$\|\boldsymbol{H}\|_0/kn$ & 0.350 & 0.649 & 0.634 & 0.710 & 0.365 & 0.385 & 0.350\\
Purity &  0.660 & 0.868 &  0.698  & 0.654  & 0.446  & 0.477  &  0.433 \\
Entropy & 0.361 & 0.216 & 0.354 & 0.375  & 0.723  & 0.690 &  0.731\\
\hline

\end{tabular}
\caption{ \small Performance of different algorithms for the gene expression data set in Section \ref{genedata}}
\label{cluster}
\end{table}

\subsection{Hyperspectral Unmixing data set~\texorpdfstring{\citep{pinesdata}}{}}\label{hyperdata}
For our next set of experiments, we consider the Indian Pines data~\citep{pinesdata}, which is a Hyperspectral image segmentation data set. This data consists of images of the Indian Pines landscape in Indiana, US, each image with size $\text{145}\times\text{145}$ pixels. These images are taken over different 220 different spectral bands (i.e. portions of the electromagnetic spectrum), resulting in 220 images. A reference image of the data set is shown in Figure~\ref{figpines1} (top left panel). The goal in this data set is to recognize different segments of the landscape from the images. Hence, this problem can be thought of as a clustering problem, where each pixel is an observation, for each pixel different spectral bands constitute as different features, and clusters are different segments and objects in the area. As different areas of the landscape may appear more or less prominent in different spectral bands, one might expect the underlying archetypes to be sparse.

To tackle this data, we vectorize each of 220 images into a vector of length 21025. This results in a data matrix $\B{X}\in\R^{21025\times 220}_{\geq 0}$. We also set the number of clusters to $k=17$ from the ground truth, given in Figure~\ref{figpines1}. Then we scale $\B{X}$ to have entries between 0 and 1. We also set $\ell=0.65nk$. The rest of the setup is the same as the experiment in Section~\ref{genedata}. The numerical results from this data are shown in Table~\ref{pines-cluster} and a visual comparison of results are shown in Figure~\ref{figpines1}.

\begin{table}[h!]

\centering
\begin{tabular}{ |c|c|c|c|c|c|c|c| } 
\hline
  & SAA & AA & Chen& Kmeans & KP & PP &  Hoyer  \\
\hline
$\|\boldsymbol{H}\|_0/kn$ &  0.75 & 1 & 1  &1  & 0.74  & 0.75 & 0.75 \\
Purity &  0.559 &  0.563 & 0.561 & 0.570   &  0.512  &  0.512  &  0.522 \\
Entropy &  0.534 &  0.529 &0.530 &  0.517  &    0.645&0.645   & 0.626 \\
\hline

\end{tabular}
\caption{ \small Performance of different algorithms for the hyperspectral unmixing data set in Section~\ref{hyperdata}}
\label{pines-cluster}
\end{table}

\begin{figure}[t!]

     \centering
\begin{tabular}{cccc}

    Reference Image & Ground Truth & Kmeans & AA/Chen \\
         \includegraphics[width=0.23\textwidth]{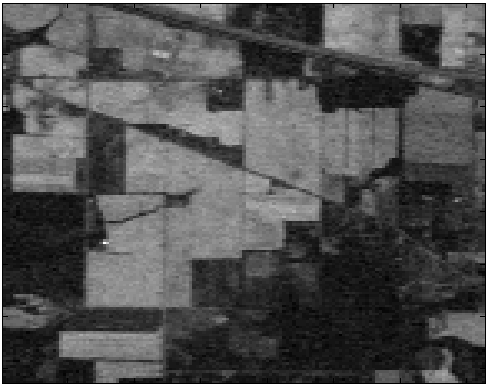} &  \includegraphics[width=0.23\linewidth,trim = 1.5cm 1cm 1.5cm 1cm, clip = true]{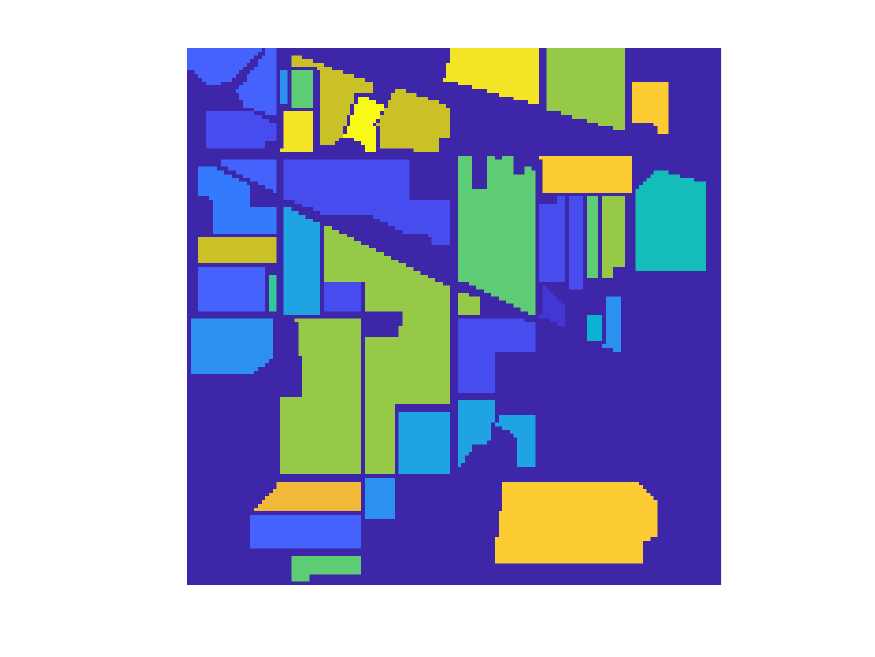} &  \includegraphics[width=0.23\linewidth,trim = 1.5cm 1cm 1.5cm 1cm, clip = true]{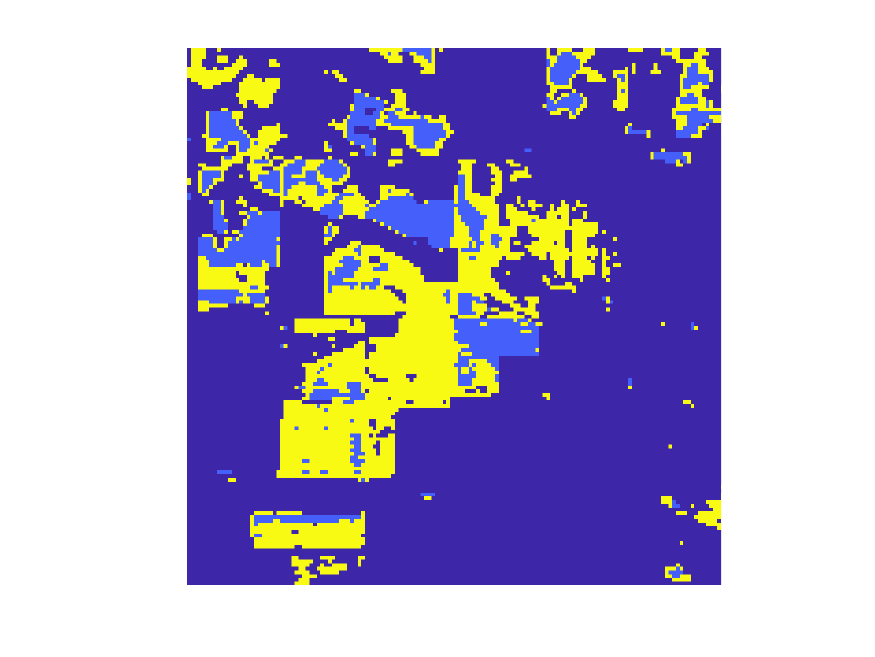}&  \includegraphics[width=0.23\linewidth,trim = 1.5cm 1cm 1.5cm 1cm, clip = true]{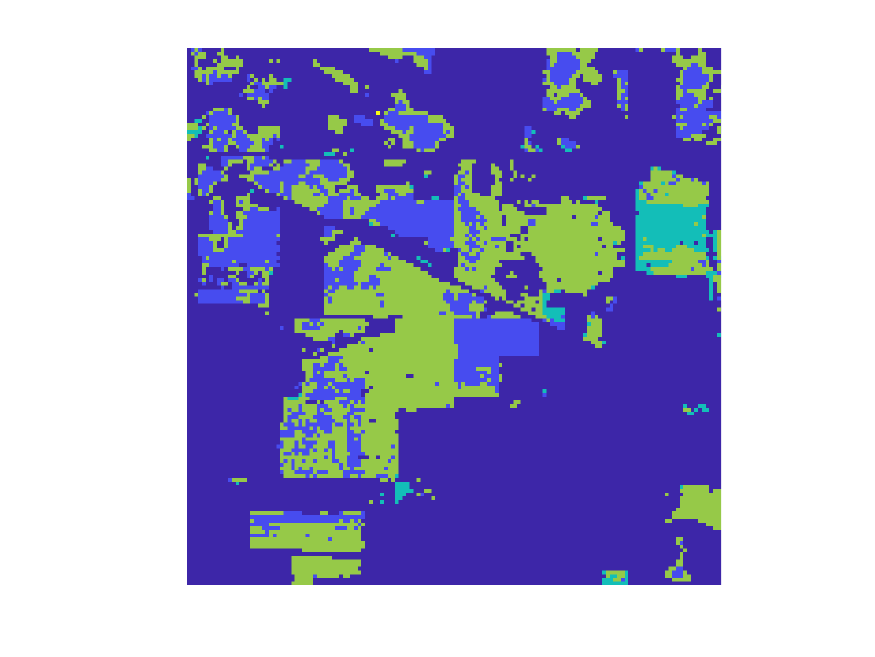}\\
          Hoyer & KP & PP & SAA \\
          \includegraphics[width=0.23\linewidth,trim = 1.5cm 1cm 1.5cm 1cm, clip = true]{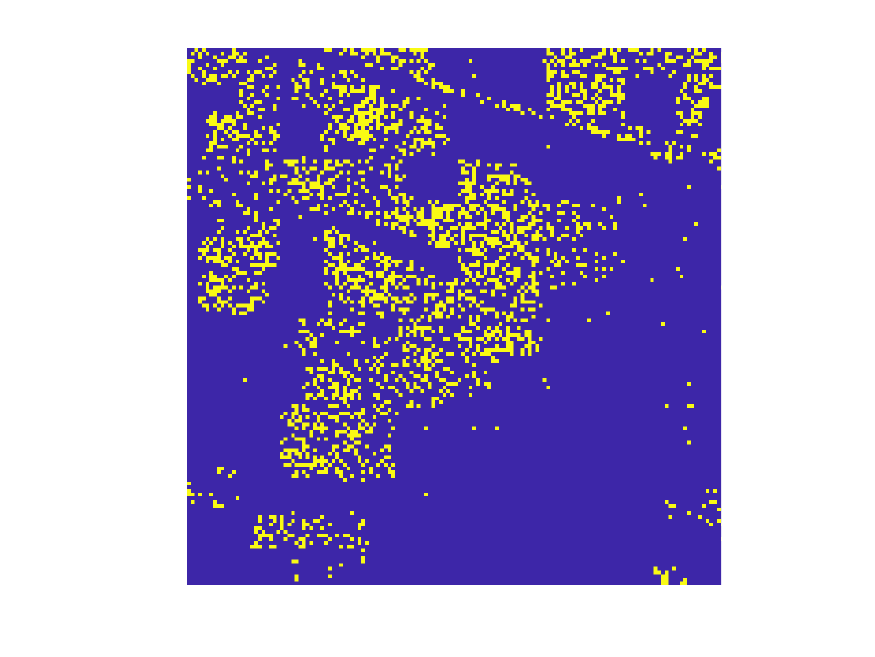} & \includegraphics[width=0.23\linewidth,trim = 1.5cm 1cm 1.5cm 1cm, clip = true]{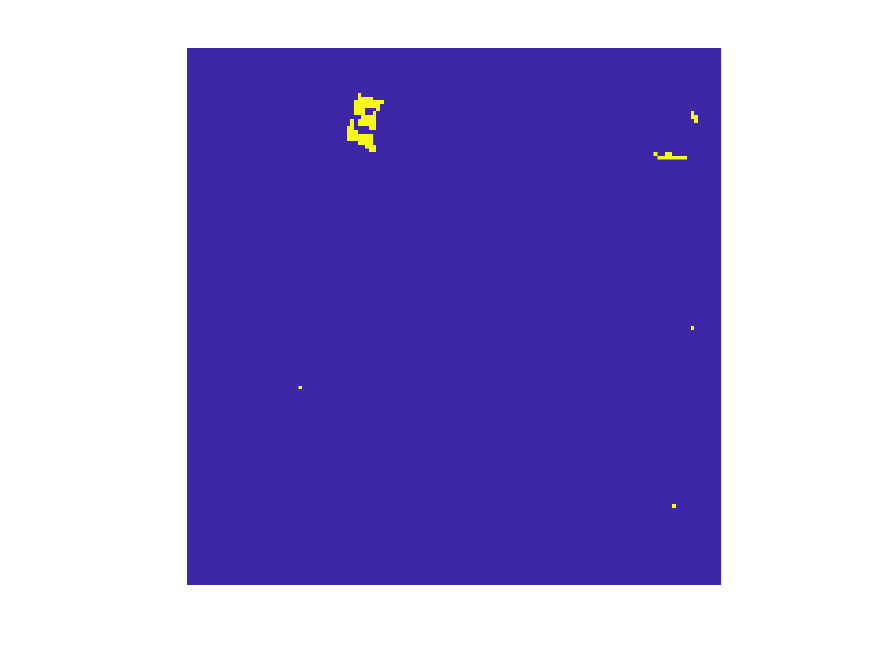}  & \includegraphics[width=0.23\linewidth,trim = 1.5cm 1cm 1.5cm 1cm, clip = true]{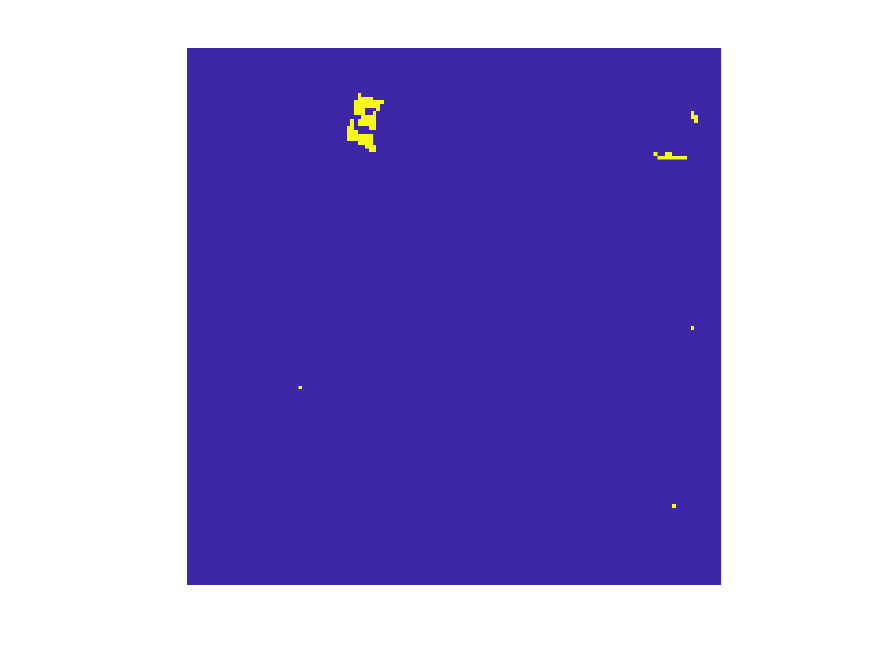}  & \includegraphics[width=0.23\linewidth,trim = 1.5cm 1cm 1.5cm 1cm, clip = true]{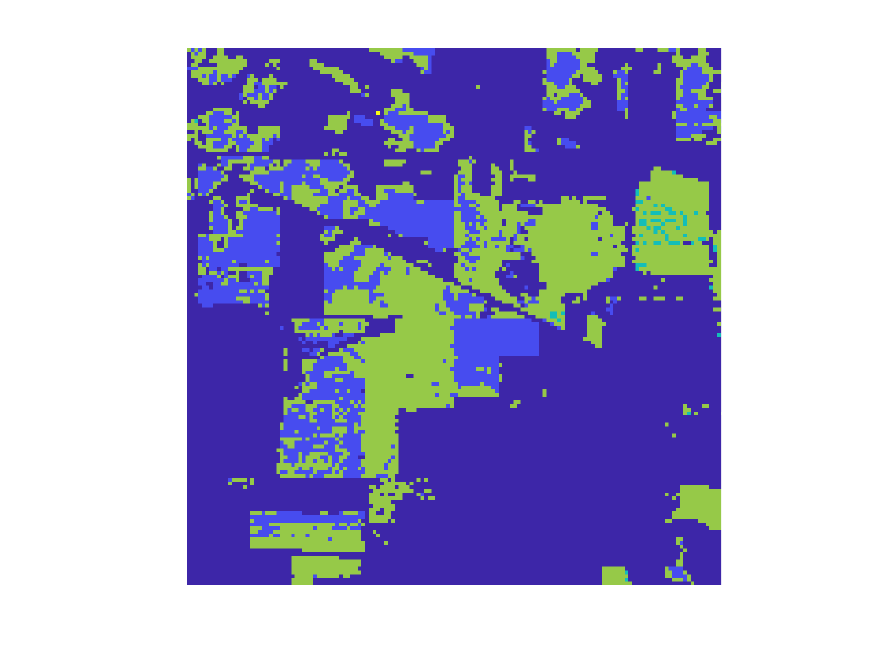} 
         \end{tabular}
     \caption{\small The figures for Section~\ref{hyperdata}. The two top left images show a reference image and the ground truth. The rest show outputs of different methods.}    \label{figpines1}
\end{figure}
Overall, Kmeans has the best performance in this data, although AA/Chen perform close. However, we note that solutions from these methods are fully dense. Moving to sparse methods, we see that SAA has the best performance and performs close to non-sparse methods. Moreover, as seen in Figure~\ref{figpines1}, the SAA solution seems to perform the best visually and identifies the landscape better than Kmeans and AA/Chen. More details on the performance of different methods and visual comparisons can be found in Appendix~\ref{app:hyperdata}.

\subsection{Scene Categorization data set \texorpdfstring{\citep{xiao2010sun}}{} }\label{sunsec}
Finally, we consider another popular application of AA arising in the context of image categorization \citep{abrolgeometric,chen2014fast}. Given a collection of photos, AA can be used to identify a small subset of photos as their representatives (archetypes). Based on the mathematical formulation of AA, archetypes are expected to represent extreme scenes and objects present in the data, for instance, they should categorize different indoor/outdoor settings or city/nature scenes. \\
As far as we can tell, NMF with an explicit $\ell_0$ regularization has not been used before to address the problem of scene categorization. However, in view of Theorem \ref{robustnessthm}, we do not anticipate that additional sparsity will reduce the robustness properties of the estimator if the underlying matrix of archetypes is sparse. In addition, a sparse model may be desirable in terms of compressed storage.
We apply SAA on the Scene Categorization data set~\citep{xiao2010sun}. We select 12 different scenes that consist of different indoor and outdoor settings (2617 images in total). These scenes are toll plaza, hospital exterior, harbor, electricity station, underwater, youth hostel, valley, ski resort, football stadium, residential neighborhood, vineyard and iceberg. We extract and concatenate GIST and HOG features \citep{xiao2010sun} and implement different sparse NMF algorithms on the data with $k=12$. As estimated archetypes in the feature space cannot be visualized, we use the closest data point to each archetype to visualize the result.\\
First, we consider AA and SAA with $\ell/kn=0.5$ (we use the MIP initialization with continuation framework in Section \ref{initsection} and choose the value of $\lambda$ that maximizes purity). The resulting visualization of archetypes for these two cases is the same. The visualization of estimated archetypes is shown in Figure \ref{sundata} (a). We observe that the resulting archetypes appear to span the 12 different scenes in the data set. Figure \ref{sundata} (b) shows the resulting visualization for PP where the resulting archetypes matrix is set to have at most $0.65nk$ nonzeros. As it can be seen, PP can identify 10 distinct scenes and chooses the electricity station and the toll plaza twice. Figure \ref{sundata} (c) shows the results for KP with the same sparsity as PP. This algorithm only identifies 5 distinct scenes. In summary, our SAA algorithm works as well as AA in terms of identifying different scenes while providing a sparse solution. This shows a sparse solution can be achieved without losing categorization performance. \\

\begin{figure}[]
     \centering
\begin{tabular}{c}
SAA (this paper)\\
\includegraphics[width=0.7\linewidth,trim =0cm 1.cm 1.cm 0cm, clip = true]{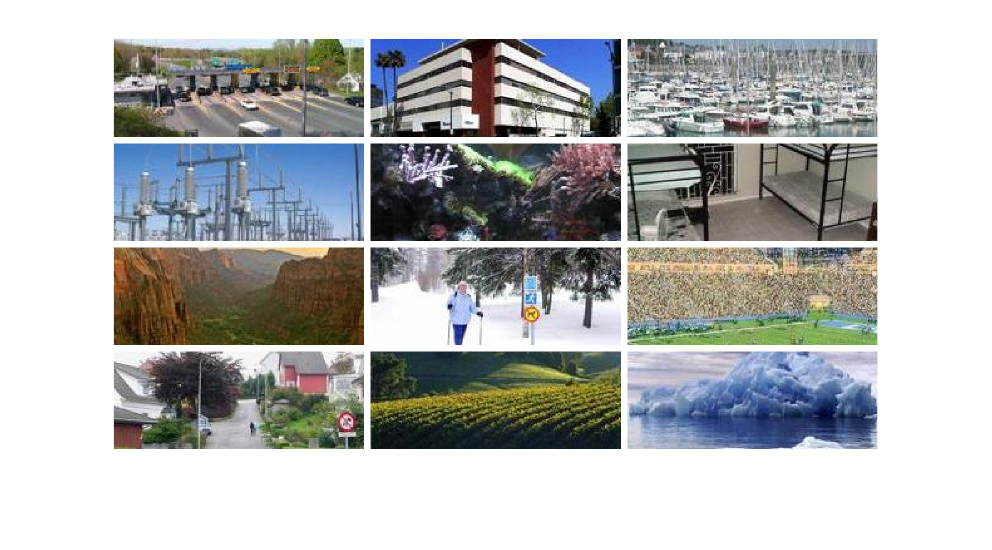} \\
PP~\citep{peharz2012sparse}\\
         \includegraphics[width=0.7\linewidth,trim =0cm 1.cm 1.cm 0cm, clip = true]{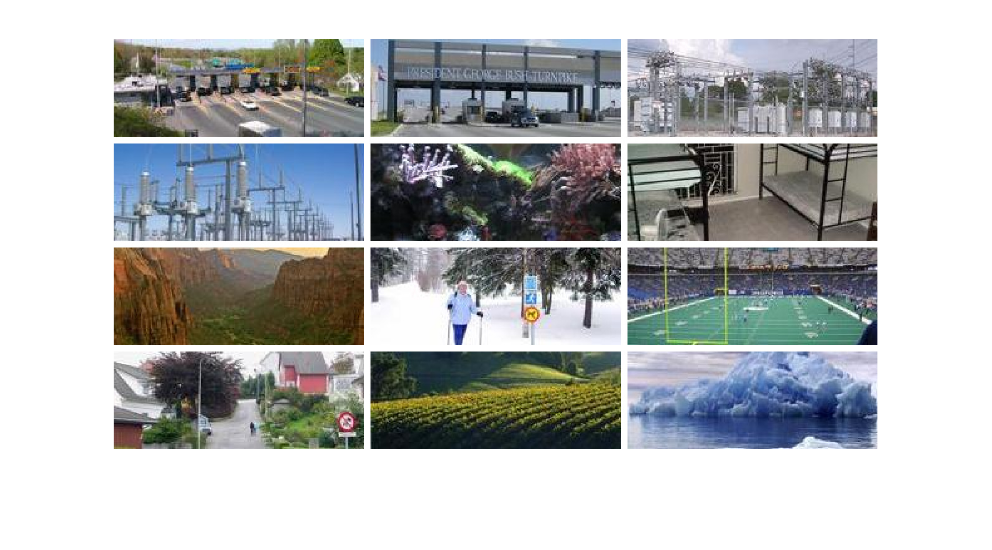}\\
         KP~\citep{10.1093/bioinformatics/btm134}\\
         \includegraphics[width=0.7\linewidth,trim =0cm 1.cm 1.cm 0cm, clip = true]{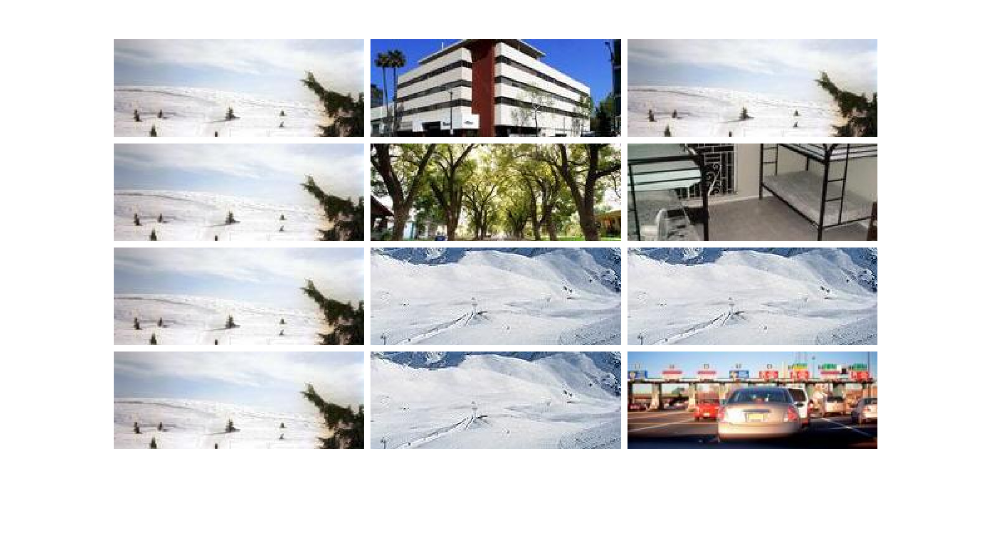}
\end{tabular}
        \caption{\small The visualization of archetypes achieved by different algorithm for the scene categorization data in Section \ref{sunsec}: (a) SAA (b) PP (c) KP. }  
        \label{sundata}
     \end{figure}

\section{Conclusion}
In this paper, we consider the problem of sparse NMF with archetypal regularization where the goal is to represent a collection of data points as nonnegative linear combinations of a few nonnegative sparse factors. \citet{javadi2019nonnegative} recently showed that NMF (without sparsity) with archetypal regularization leads to robustness---factors learnt from noisy data are close to the underlying factors that generate the noiseless data. We generalize the notion of robustness to (a) strong robustness that implies each estimated archetype is close to the underlying archetypes and (b) weak robustness that implies there exists at least one recovered archetype that is close to the underlying archetypes. \citet{javadi2019nonnegative} is an instance of the notion of weak robustness presented herein. We show that under minimal assumptions, robustness in sparse NMF can be achieved by considering a sparsity constrained regularized AA problem, even if the underlying archetypes are not sparse. We present a block coordinate algorithm to get a good solution to the sparse AA problem and also an initialization framework using mixed integer programming that leads to better numerical results. We also present a local search algorithm that improves the quality of the solution of our block coordinate algorithm. Numerical experiments on synthetic and real data sets shed further insights into the theoretical developments pursued in this paper. 

\section{Acknowledgements} 
The authors would like to thank the Action Editor and five referees for their helpful comments that have led to improvements in the paper.
Rahul Mazumder would like to thank Kushal Dey for helpful discussions in the initial stages of this work.
This research was partially supported by grants from the Office of 
Naval Research (ONR-N000141812298, N000142112841 and N000142212665), National Science Foundation (NSF-IIS-1718258), IBM and Liberty Mutual Insurance.

\newpage
\appendix

\section*{Proofs and Technical Details}

\section{Additional Notation}\label{notations}
We define
\begin{equation}\label{app:defn}
\tilde{D}(\boldsymbol{X},\boldsymbol{Y})=\sum_{i=1}^{\text{\#row}(\boldsymbol{X})}\sqrt{D(\boldsymbol{X}_{i,.},\boldsymbol{Y})}~~\text{and} ~~\tilde{\mathcal{L}}(\boldsymbol{X},\boldsymbol{Y})=\sum_{i=1}^{\text{\#row}(\boldsymbol{X})}\min_{j\in [\text{\#row}(\boldsymbol{Y})]}\|\boldsymbol{X}_{i,.}-\boldsymbol{Y}_{j,.}\|_2.
\end{equation}
For a set $S$, $S^c$ denotes the complement of the set. We use $\mathbb{I}_k$ to denote the identity matrix of size $k$.

\section{Technical Details of The Toy Example}\label{toyexa}
For Figure 1 Panel (a), we let 
$$\boldsymbol{H}_0=\begin{bmatrix}
0.15 & 0.15 \\
0.1 & 0.7 \\
0.7 & 0.1
\end{bmatrix}$$
and produce 50 data points. To do so, each entry of $\boldsymbol{W}_0$ is drawn from an independent uniform distribution in $[0,1]$ and each row is normalized to sum to one. The noiseless data matrix is $\boldsymbol{X}_0=\boldsymbol{W}_0\boldsymbol{H}_0$ and we add three rows of $\boldsymbol{H}_0$ to this to a obtain separable problem. In this case, we can see $D(\boldsymbol{X}_0,\boldsymbol{H}_0)=D(\boldsymbol{H}_0,\boldsymbol{X}_0)=0$ which is the exact AA solution of \citet{Breiman}. The red convex hull is $\text{Conv}(\boldsymbol{H}_0)$. Let 
$$\boldsymbol{H}_1=\begin{bmatrix}
0.05 & 0.05 \\
1 & 0.1 \\
0.1 & 1
\end{bmatrix}.$$
The black convex hull is $\text{Conv}(\boldsymbol{H}_1)$ for which we have $D(\boldsymbol{X}_0,\boldsymbol{H}_1)=0$ but $D(\boldsymbol{H}_1,\boldsymbol{X}_0)>0$.

\noindent For Figure 1 Panel (b), the data $\boldsymbol{X}$ is produced by $\boldsymbol{X}=\boldsymbol{X}_0+\boldsymbol{Z}$ where $\boldsymbol{Z}$ has zero-mean iid normal coordinates with variance of $0.1$. In this case, we have for every data point $i$, $D(\boldsymbol{X}_{i,.},\boldsymbol{H}_0)\leq 0.1$ making $\boldsymbol{H}_0$ a feasible solution for the regularized AA of \citet{javadi2019nonnegative} and as $D(\boldsymbol{H}_0,\boldsymbol{X})=0$, this is the solution of the regularized AA problem (the red convex hull). However, we have $D(\boldsymbol{X},\boldsymbol{H}_1)=0$ and $D(\boldsymbol{H}_1,\boldsymbol{X})>0$, so the black convex hull is not an optimal solution. 

\noindent Finally, let 
$$\boldsymbol{H}_2=\begin{bmatrix}
0 & 0 \\
0 & 0.8 \\
0.8 & 0
\end{bmatrix}.$$
In Figure 1 Panel (c), the black convex hull is $\text{Conv}(\boldsymbol{H}_2)$ and the red one is $\text{Conv}(\boldsymbol{H}_0)$. We have $\|\boldsymbol{H}_2\|_0=2$ and $\|\boldsymbol{H}_0\|_0=6$ so the red convex hull is not sparse. In addition, among all solutions that have $\|\boldsymbol{H}\|_0=2$ and $D(\boldsymbol{X}_0,\boldsymbol{H})=0$, the quantity $D(\boldsymbol{H},\boldsymbol{X}_0)$ is minimized for the black convex hull, making it the sparse archetypal solution. 
\section{Technical Lemmas}
\begin{lem}\label{firstlem}
For any two matrices $\boldsymbol{X}\in\mathbb{R}^{m_1\times n}$ and $\boldsymbol{Y}\in\mathbb{R}^{m_2\times n}$, we have:
\begin{align}
& \frac{1}{\sqrt{m_1}}\tilde{D}(\boldsymbol{X},\boldsymbol{Y}) \leq D(\boldsymbol{X},\boldsymbol{Y})^{1/2}\leq \tilde{D}(\boldsymbol{X},\boldsymbol{Y}) \label{sec:lemma1-proof-line1}\\
&\frac{1}{\sqrt{m_1}}\tilde{\mathcal{L}}(\boldsymbol{X},\boldsymbol{Y}) \leq \mathcal{L}(\boldsymbol{X},\boldsymbol{Y})^{1/2}\leq \tilde{\mathcal{L}}(\boldsymbol{X},\boldsymbol{Y}),\label{sec:lemma1-proof-line2}
\end{align}
where, recall $\tilde{D}$ and $\tilde{\mathcal L}$ are as defined in~\eqref{app:defn}. 
\end{lem}
\begin{proof}
Define the vector $\boldsymbol{u}\in\mathbb{R}^{m_1}$ such that $\boldsymbol{u}_i=\sqrt{D(\boldsymbol{X}_{i,.},\boldsymbol{Y})}$. Note that
\begin{align*}
\frac{1}{\sqrt{m_1}}\tilde{D}(\boldsymbol{X},\boldsymbol{Y}) = \frac{1}{\sqrt{m_1}}\|\boldsymbol{u}\|_1\leq \|\boldsymbol{u}\|_2= D(\boldsymbol{X},\boldsymbol{Y})^{1/2} \leq \|\boldsymbol{u}\|_1 = \tilde{D}(\boldsymbol{X},\boldsymbol{Y})
\end{align*}
which establishes~\eqref{sec:lemma1-proof-line1}. 
The proof of~\eqref{sec:lemma1-proof-line2} is similar. 
\end{proof}
\begin{lem}\label{comparison}
If $D(\boldsymbol{A}_1,\boldsymbol{A}_2)\leq D(\boldsymbol{B}_1,\boldsymbol{B}_2)$ where $\boldsymbol{A}_1,\boldsymbol{A}_2\in\mathbb{R}^{m\times n}$,$\boldsymbol{B}_1,\boldsymbol{B}_2\in\mathbb{R}^{k\times n}$ , we have
$$\tilde{D}(\boldsymbol{A}_1,\boldsymbol{A}_2)\leq \sqrt{m}\tilde{D}(\boldsymbol{B}_1,\boldsymbol{B}_2).$$
\end{lem}
\begin{proof}
We make use of Lemma \ref{firstlem}. The proof follows from:
\begin{align*}
    \tilde{D}(\boldsymbol{A}_1,\boldsymbol{A}_2) & \leq \sqrt{m} D(\boldsymbol{A}_1,\boldsymbol{A}_2)^{1/2}\leq \sqrt{m} D(\boldsymbol{B}_1,\boldsymbol{B}_2)^{1/2} \leq \sqrt{m} \tilde{D}(\boldsymbol{B}_1,\boldsymbol{B}_2).
\end{align*}

\end{proof}
\begin{lem}\label{split}
Suppose $\boldsymbol{X}\in\mathbb{R}^{m_1\times n},\boldsymbol{Y}\in\mathbb{R}^{m_2\times n},\boldsymbol{Z}\in\mathbb{R}^{m_3\times n}$, then
$$\tilde{D}(\boldsymbol{X},\boldsymbol{Z})\leq \tilde{\mathcal{L}}(\boldsymbol{X},\boldsymbol{Y})+m_1\tilde{D}(\boldsymbol{Y},\boldsymbol{Z}).$$
\end{lem}
\begin{proof}
Fix $i\in[m_1]$. For any $\boldsymbol{u}\in\mathbb{R}^n$, we have 
\begin{align}\label{app:lemma3-line11}
\sqrt{D(\boldsymbol{X}_{i,.},\boldsymbol{Z})} & = \min_{\boldsymbol{v}\in\text{Conv}(\boldsymbol{Z})} \|\boldsymbol{X}_{i,.}-\boldsymbol{v}\|_2\nonumber  \\
& \leq \min_{\boldsymbol{v}\in\text{Conv}(\boldsymbol{Z})} \left\{ \|\boldsymbol{X}_{i,.}-\boldsymbol{u}\|_2+\|\boldsymbol{u}-\boldsymbol{v}\|_2  \right\}\\
& = \|\boldsymbol{X}_{i,.}-\boldsymbol{u}\|_2 + \min_{\boldsymbol{v}\in\text{Conv}(\boldsymbol{Z})}\|\boldsymbol{u}-\boldsymbol{v}\|_2. \nonumber 
\end{align}
Let 
$$\boldsymbol{u}=\argmin_{\boldsymbol{p}\in\{\boldsymbol{Y}_{j,.}:j\in[\text{\#row}(\boldsymbol{Y})]\}}\|\boldsymbol{X}_{i,.}-\boldsymbol{p}\|_2.$$
As a result, noting that $\tilde{D}(\B{Y},\B{Z})=\sum_{j=1}^{m_2}\sqrt{D(\B{Y}_{j,.},\B{Z})}$ and $\B{u}$ is a row of $\B{Y}$,
\begin{align}\label{app:lemma3-line12}
\min_{\boldsymbol{v}\in\text{Conv}(\boldsymbol{Z})}\|\boldsymbol{u}-\boldsymbol{v}\|_2 & = \sqrt{D(\boldsymbol{u},\boldsymbol{Z})}  \leq \sum_{j=1}^{m_2}\sqrt{D(\boldsymbol{Y}_{j,.},\boldsymbol{Z})}  = \tilde{D}(\boldsymbol{Y},\boldsymbol{Z}).
\end{align}
Using the definition of $\B{u}$ and~\eqref{app:lemma3-line12} we can bound the rhs in~\eqref{app:lemma3-line11} to get:
$$\sqrt{D(\boldsymbol{X}_{i,.},\boldsymbol{Z})} \leq \min_{j\in m_2}\|\boldsymbol{X}_{i,.}-\boldsymbol{Y}_{j,.}\|_2 + \tilde{D}(\boldsymbol{Y},\boldsymbol{Z}). $$ 
By summing the above over $i$, we have: 

\begin{align*}
    \tilde{D}(\boldsymbol{X},\boldsymbol{Z})& =\sum_{i=1}^{m_1}\sqrt{D(\boldsymbol{X}_{i,.},\boldsymbol{Z})}\\
    & \leq \sum_{i=1}^{m_1}\min_{j\in m_2}\|\boldsymbol{X}_{i,.}-\boldsymbol{Y}_{j,.}\|_2  + \sum_{i=1}^{m_1}\tilde{D}(\boldsymbol{Y},\boldsymbol{Z})\\
    & = \tilde{\mathcal{L}}(\boldsymbol{X},\boldsymbol{Y})+m_1\tilde{D}(\boldsymbol{Y},\boldsymbol{Z})
\end{align*}
which completes the proof. 
\end{proof}
\begin{lem}[\citet{javadi2019nonnegative}] \label{hamiddist} If $\boldsymbol{A},\boldsymbol{B}\in\mathbb{R}^{m\times n}$, $m\leq n$ are matrices with linearly independent rows, we have
\begin{equation*}
   \mathcal{L}(\boldsymbol{A},\boldsymbol{B})^{1/2}\leq 2\kappa(\boldsymbol{A})D(\boldsymbol{A},\boldsymbol{B})^{1/2}+(1+\sqrt{2})\sqrt{m}D(\boldsymbol{B},\boldsymbol{A})^{1/2}
\end{equation*}
where recall that $\kappa(\boldsymbol{H}):=\sigma_{\max}(\B{H})/\sigma_{\min}(\B{H})$ denotes the condition number of $\boldsymbol{H}$.
\end{lem}

\begin{lem}\label{perturb}
Suppose $\boldsymbol{H}\in\mathbb{R}^{k\times n}$, $\boldsymbol{X}_0\in\mathbb{R}^{m\times n}$ and $\boldsymbol{Z}\in\mathbb{R}^{m\times n}$ is such that $\max_{i\in[m]}\|\boldsymbol{Z}_{i,.}\|_2\leq \delta$. Let $\boldsymbol{X}=\boldsymbol{X}_0+\boldsymbol{Z}$. We have 
\begin{align}
\tilde{D}(\boldsymbol{H},\boldsymbol{X})\leq \tilde{D}(\boldsymbol{H},\boldsymbol{X}_0)+k\delta \label{app:lemma5-line1}\\
\tilde{D}(\boldsymbol{X},\boldsymbol{H})\leq \tilde{D}(\boldsymbol{X}_0,\boldsymbol{H}) + m\delta. \label{app:lemma5-line2}
\end{align}
\end{lem}
\begin{proof}
Let us denote the $m$-dimensional unit simplex by:
\begin{equation}\label{delta-def}
    \Delta^m = \{\boldsymbol{\alpha}\in \mathbb{R}^m:\sum_{i=1}^m \boldsymbol{\alpha}_i = 1, \boldsymbol{\alpha}_i \geq 0\}.
\end{equation}
We have for any $i\in[k]$
\begin{align*}
\tilde{D}(\boldsymbol{H}_{i,.},\boldsymbol{X})& =\min_{\boldsymbol{\alpha}\in\Delta^m}\|\boldsymbol{H}_{i,.}-\sum_{j=1}^m\boldsymbol{\alpha}_j\boldsymbol{X}_{j,.}\|_2 \\
& = \min_{\boldsymbol{\alpha}\in\Delta^m}\|\boldsymbol{H}_{i,.}-\sum_{j=1}^m\boldsymbol{\alpha}_j\boldsymbol{X}^0_{j,.} - \sum_{j=1}^m \boldsymbol{\alpha}_j\boldsymbol{Z}_{j,.}\|_2 \\
& \leq \min_{\boldsymbol{\alpha}\in\Delta^m} \left\{\|\boldsymbol{H}_{i,.}-\sum_{j=1}^m\boldsymbol{\alpha}_j\boldsymbol{X}^0_{j,.}\|_2+\|\sum_{j=1}^m\boldsymbol{\alpha}_j\boldsymbol{Z}_{j,.}\|_2 \right\}\\
& \leq \min_{\boldsymbol{\alpha}\in\Delta^m} \left\{ \|\boldsymbol{H}_{i,.}-\sum_{j=1}^m\boldsymbol{\alpha}_j\boldsymbol{X}^0_{j,.}\|_2 + \sum_{j=1}^m \boldsymbol{\alpha}_j \|\boldsymbol{Z}_{j,.}\|_2 \right\}\\
& \leq \min_{\boldsymbol{\alpha}\in\Delta^m} \left\{ \|\boldsymbol{H}_{i,.}-\sum_{j=1}^m\boldsymbol{\alpha}_j\boldsymbol{X}^0_{j,.}\|_2 + (\sum_{j=1}^m \boldsymbol{\alpha}_j) \max_{j\in[m]}\|\boldsymbol{Z}_{j,.}\|_2 \right\}\\
& \leq \min_{\boldsymbol{\alpha}\in\Delta^m}\|\boldsymbol{H}_{i,.}-\sum_{j=1}^m\boldsymbol{\alpha}_j\boldsymbol{X}^0_{j,.}\|_2 + \delta \\
& = \tilde{D}(\boldsymbol{H}_{i,.},\boldsymbol{X}_0) + \delta.
\end{align*}
Using the above bound, we have:
\begin{align*}
    \tilde{D}(\boldsymbol{H},\boldsymbol{X})=\sum_{i=1}^k \tilde{D}(\boldsymbol{H}_{i,.},\boldsymbol{X})
    \leq \sum_{i=1}^k \tilde{D}(\boldsymbol{H}_{i,.},\boldsymbol{X}_0)+k\delta = \tilde{D}(\boldsymbol{H},\boldsymbol{X}_0)+k\delta
\end{align*}
which establishes~\eqref{app:lemma5-line1}.
We now show~\eqref{app:lemma5-line2}. For a fixed $i\in[m]$, we have
\begin{align*}
    \tilde{D}(\boldsymbol{X}_{i,.},\boldsymbol{H})& = \min_{\boldsymbol{\alpha}\in\Delta^k}\|\boldsymbol{X}_{i,.}-\sum_{j=1}^k\boldsymbol{\alpha}_j\boldsymbol{H}_{j,.}\|_2\\
    & =\min_{\boldsymbol{\alpha}\in\Delta^k}\|\boldsymbol{X}^0_{i,.}-\sum_{j=1}^k\boldsymbol{\alpha}_j\boldsymbol{H}_{j,.}+\boldsymbol{Z}_{i,.}\|_2 \\
    & \leq \min_{\boldsymbol{\alpha}\in\Delta^k}\|\boldsymbol{X}^0_{i,.}-\sum_{j=1}^k\boldsymbol{\alpha}_j\boldsymbol{H}_{j,.}\|_2+\|\boldsymbol{Z}_{i,.}\|_2 \\
    & \leq \min_{\boldsymbol{\alpha}\in\Delta^k}\|\boldsymbol{X}^0_{i,.}-\sum_{j=1}^k\boldsymbol{\alpha}_j\boldsymbol{H}_{j,.}\|_2+ \delta \\
    & = \tilde{D}(\boldsymbol{X}^0_{i,.},\boldsymbol{H})+\delta.
\end{align*}
By summing the above bound over $i$, we arrive at~\eqref{app:lemma5-line2}.
\end{proof}

\begin{lem}\label{perturb2}
Suppose $\boldsymbol{H}\in\mathbb{R}^{k\times n}$, $\boldsymbol{X}_0\in\mathbb{R}^{m\times n}$ and $\boldsymbol{Z}\in\mathbb{R}^{m\times n}$ is such that $\max_{i\in[m]}\|\boldsymbol{Z}_{i,.}\|_2\leq \delta$. Let $\boldsymbol{X}=\boldsymbol{X}_0+\boldsymbol{Z}$. We have 
\begin{align}
{D}(\boldsymbol{H},\boldsymbol{X})^{1/2}\leq \sqrt{2}{D}(\boldsymbol{H},\boldsymbol{X}_0)^{1/2}+\sqrt{2k}\delta \label{app:lemma5_2-line1}\\
D(\boldsymbol{X},\boldsymbol{H})^{1/2}\leq \sqrt{2}D(\boldsymbol{X}_0,\boldsymbol{H})^{1/2} + \sqrt{2m}\delta. \label{app:lemma5_2-line2}
\end{align}
\end{lem}
\begin{proof}
Let $\Delta^m$ be defined as in~\eqref{delta-def}. We have for any $i\in[k]$ 
\begin{align*}
{D}(\boldsymbol{H}_{i,.},\boldsymbol{X})& =\min_{\boldsymbol{\alpha}\in\Delta^m}\|\boldsymbol{H}_{i,.}-\sum_{j=1}^m\boldsymbol{\alpha}_j\boldsymbol{X}_{j,.}\|_2^2 \\
& = \min_{\boldsymbol{\alpha}\in\Delta^m}\|\boldsymbol{H}_{i,.}-\sum_{j=1}^m\boldsymbol{\alpha}_j\boldsymbol{X}^0_{j,.} - \sum_{j=1}^m \boldsymbol{\alpha}_j\boldsymbol{Z}_{j,.}\|_2^2 \\
& \stackrel{(a)}{\leq} 2\min_{\boldsymbol{\alpha}\in\Delta^m} \left\{\|\boldsymbol{H}_{i,.}-\sum_{j=1}^m\boldsymbol{\alpha}_j\boldsymbol{X}^0_{j,.}\|_2^2+\|\sum_{j=1}^m\boldsymbol{\alpha}_j\boldsymbol{Z}_{j,.}\|_2^2 \right\}\\
& \leq 2\min_{\boldsymbol{\alpha}\in\Delta^m} \left\{ \|\boldsymbol{H}_{i,.}-\sum_{j=1}^m\boldsymbol{\alpha}_j\boldsymbol{X}^0_{j,.}\|_2^2 + \left(\sum_{j=1}^m \boldsymbol{\alpha}_j \|\boldsymbol{Z}_{j,.}\|_2\right)^2 \right\}\\
& \leq 2\min_{\boldsymbol{\alpha}\in\Delta^m} \left\{ \|\boldsymbol{H}_{i,.}-\sum_{j=1}^m\boldsymbol{\alpha}_j\boldsymbol{X}^0_{j,.}\|_2^2 + (\sum_{j=1}^m \boldsymbol{\alpha}_j)^2 \max_{j\in[m]}\|\boldsymbol{Z}_{j,.}\|_2^2 \right\}\\
& \leq 2\min_{\boldsymbol{\alpha}\in\Delta^m}\|\boldsymbol{H}_{i,.}-\sum_{j=1}^m\boldsymbol{\alpha}_j\boldsymbol{X}^0_{j,.}\|_2^2 + 2\delta^2 \\
& = 2{D}(\boldsymbol{H}_{i,.},\boldsymbol{X}_0) + 2\delta^2
\end{align*}
where $(a)$ uses $\|\B{a}+\B{b}\|_2^2\leq 2\|\B{a}\|_2^2+2\|\B{b}\|_2^2$.
Using the above bound, we have:
\begin{align*}
    {D}(\boldsymbol{H},\boldsymbol{X})=\sum_{i=1}^k {D}(\boldsymbol{H}_{i,.},\boldsymbol{X})
    \leq \sum_{i=1}^k 2{D}(\boldsymbol{H}_{i,.},\boldsymbol{X}_0)+2k\delta^2 = 2{D}(\boldsymbol{H},\boldsymbol{X}_0)+2k\delta^2.
\end{align*}
As a result,
\begin{align*}
    {D}(\boldsymbol{H},\boldsymbol{X})^{1/2}=\sqrt{{D}(\boldsymbol{H},\boldsymbol{X})}\leq \sqrt{2{D}(\boldsymbol{H},\boldsymbol{X}_0)+2k\delta^2}\leq \sqrt{2}{D}(\boldsymbol{H},\boldsymbol{X}_0)^{1/2}+\sqrt{2k}\delta
\end{align*}
which establishes~\eqref{app:lemma5_2-line1}.
We now show~\eqref{app:lemma5_2-line2}. For a fixed $i\in[m]$, we have
\begin{align*}
    {D}(\boldsymbol{X}_{i,.},\boldsymbol{H})& = \min_{\boldsymbol{\alpha}\in\Delta^k}\|\boldsymbol{X}_{i,.}-\sum_{j=1}^k\boldsymbol{\alpha}_j\boldsymbol{H}_{j,.}\|_2^2\\
    & =\min_{\boldsymbol{\alpha}\in\Delta^k}\|\boldsymbol{X}^0_{i,.}-\sum_{j=1}^k\boldsymbol{\alpha}_j\boldsymbol{H}_{j,.}+\boldsymbol{Z}_{i,.}\|_2^2 \\
    & \leq 2\min_{\boldsymbol{\alpha}\in\Delta^k}\|\boldsymbol{X}^0_{i,.}-\sum_{j=1}^k\boldsymbol{\alpha}_j\boldsymbol{H}_{j,.}\|_2^2+2\|\boldsymbol{Z}_{i,.}\|_2^2 \\
    & \leq 2\min_{\boldsymbol{\alpha}\in\Delta^k}\|\boldsymbol{X}^0_{i,.}-\sum_{j=1}^k\boldsymbol{\alpha}_j\boldsymbol{H}_{j,.}\|_2+ 2\delta^2 \\
    & = 2{D}(\boldsymbol{X}^0_{i,.},\boldsymbol{H})+2\delta^2.
\end{align*}
Hence,
\begin{align*}
     {D}(\boldsymbol{X},\boldsymbol{H})=\sum_{i=1}^m  {D}(\boldsymbol{X}_{i,.},\boldsymbol{H}) \leq  2\sum_{i=1}^m{D}(\boldsymbol{X}^0_{i,.},\boldsymbol{H})+2m\delta^2=2{D}(\boldsymbol{X}_0,\boldsymbol{H})+2m\delta^2
\end{align*}
which implies~\eqref{app:lemma5_2-line2}.
\end{proof}

\begin{lem}\label{specialcase}
Suppose $\boldsymbol{H}\in\mathbb{R}^{k\times n}$, $k\leq n$, has full row rank and $\boldsymbol{X}\in\mathbb{R}^{m\times n}$ is such that $\text{Conv}(\boldsymbol{X})\subseteq\text{Conv}(\boldsymbol{H})$. Let $\tilde{\boldsymbol{X}}\in\mathbb{R}^{k\times n}$ be a matrix with its $i$-row given by:
\begin{equation}\label{xtildelemma}
\tilde{\boldsymbol{X}}_{i,.}=\argmin_{\boldsymbol{u}\in\{\boldsymbol{X}_{j,.}:j\in[m]\}} \|\boldsymbol{u}-\boldsymbol{H}_{i,.}\|_2
\end{equation}
for all $i \in [k]$. Then,
\begin{equation}\label{lemma6-final}
\mathcal{L}(\boldsymbol{H},\boldsymbol{X})^{1/2}\leq 2k\kappa(\boldsymbol{H})D(\boldsymbol{H},\tilde{\boldsymbol{X}})^{1/2}.
\end{equation}
\end{lem}
\begin{proof}
Fix any $\epsilon>0$. There is a matrix $\boldsymbol{Z}_{\epsilon}\in\mathbb{R}^{k\times n}$ such that $\|\boldsymbol{Z}_{\epsilon}\|_F\leq \epsilon$ (therefore, $\max_{i\in [m]}\|\boldsymbol{Z}^{\epsilon}_{i,.}\|_2\leq \epsilon$) and $\tilde{\boldsymbol{X}}+\boldsymbol{Z}_{\epsilon}$ has full row rank. We have for any $i\in[k]$
\begin{align}
\tilde{\mathcal{L}}(\boldsymbol{H}_{i,.},\boldsymbol{X}) & = \tilde{\mathcal{L}}(\boldsymbol{H}_{i,.},\tilde{\boldsymbol{X}}) \nonumber\\
&= \min_{j\in[k]}\|\boldsymbol{H}_{i,.}-\tilde{\boldsymbol{X}}_{j,.}\|_2 \nonumber\\
& = \min_{j\in[k]}\|\boldsymbol{H}_{i,.}-\tilde{\boldsymbol{X}}_{j,.}+(\boldsymbol{Z}_{\epsilon})_{j,.}-(\boldsymbol{Z}_{\epsilon})_{j,.}\|_2 \nonumber\\
& \leq \min_{j\in[k]} \|\boldsymbol{H}_{i,.}-\tilde{\boldsymbol{X}}_{j,.}-(\boldsymbol{Z}_{\epsilon})_{j,.}\|_2 + \|(\boldsymbol{Z}_{\epsilon})_{j,.}\|_2 \nonumber\\
& \leq \tilde{\mathcal{L}}(\boldsymbol{H}_{i,.},\tilde{\boldsymbol{X}}+\boldsymbol{Z}_{\epsilon}) + \epsilon.\label{lem6ineq}
\end{align}
Note that we have the following inequalities:
\begin{align}
    \mathcal{L}(\boldsymbol{H},\boldsymbol{X})^{1/2}& \stackrel{(a)}{\leq} \tilde{\mathcal{L}}(\boldsymbol{H},\boldsymbol{X})\nonumber\\
    &= \sum_{i=1}^k \tilde{\mathcal{L}}(\boldsymbol{H}_{i,.},\boldsymbol{X})\nonumber\\
    & \stackrel{(b)}{\leq} \tilde{\mathcal{L}}(\boldsymbol{H},\tilde{\boldsymbol{X}}+\boldsymbol{Z}_{\epsilon}) + k\epsilon \nonumber\\
    & \stackrel{(c)}{\leq} \sqrt{k}\mathcal{L}(\boldsymbol{H},\tilde{\boldsymbol{X}}+\boldsymbol{Z}_{\epsilon})^{1/2} + k\epsilon, \label{lem6eq1}
\end{align}
where, $(a),(c)$ are results of Lemma \ref{firstlem} and $(b)$ is a result of \eqref{lem6ineq}. In addition, by Lemmas \ref{firstlem} and \ref{perturb}, we have:
\begin{align}
  D(\boldsymbol{H},\tilde{\boldsymbol{X}}+\boldsymbol{Z}_{\epsilon})^{1/2}&\leq \tilde{D}(\boldsymbol{H},\tilde{\boldsymbol{X}}+\boldsymbol{Z}_{\epsilon}) \leq \tilde{D}(\boldsymbol{H},\tilde{\boldsymbol{X}})+k\epsilon  \label{lem6eq2}
\end{align}
and
\begin{equation}\label{lem6eq3}D(\tilde{\boldsymbol{X}}+\boldsymbol{Z}_{\epsilon},\boldsymbol{H})^{1/2}\leq \tilde{D}(\tilde{\boldsymbol{X}}+\boldsymbol{Z}_{\epsilon},\boldsymbol{H}) \leq \tilde{D}(\tilde{\boldsymbol{X}},\boldsymbol{H})+m\epsilon=m\epsilon,\end{equation}
where the last equality in~\eqref{lem6eq3} follows from 
the fact that $\tilde{D}(\tilde{\boldsymbol{X}},\boldsymbol{H})=0$ (as $\text{Conv}(\boldsymbol{X})\subseteq\text{Conv}(\boldsymbol{H})$).

Starting with~\eqref{lem6eq1}, we have
\begin{align}
    \mathcal{L}(\boldsymbol{H},\boldsymbol{X})^{1/2} &\leq \sqrt{k}\mathcal{L}(\boldsymbol{H},\tilde{\boldsymbol{X}}+\boldsymbol{Z}_{\epsilon})^{1/2} + k\epsilon \label{ineq1}\\
    & \leq 2\sqrt{k}\kappa(\boldsymbol{H})D(\boldsymbol{H},\tilde{\boldsymbol{X}}+\boldsymbol{Z}_{\epsilon})^{1/2} +k(1+\sqrt{2})D(\tilde{\boldsymbol{X}}+\boldsymbol{Z}_{\epsilon},\boldsymbol{H})^{1/2} +k\epsilon \label{ineq2}\\
    & \leq 2\sqrt{k}\kappa(\boldsymbol{H})\tilde{D}(\boldsymbol{H},\tilde{\boldsymbol{X}}) + \epsilon(k+mk(1+\sqrt{2})+ 2\sqrt{k^3}\kappa(\boldsymbol{H})) \label{ineq3}\\
    & \leq 2k\kappa(\boldsymbol{H})D(\boldsymbol{H},\tilde{\boldsymbol{X}})^{1/2} + \epsilon(k+mk(1+\sqrt{2})+ 2\sqrt{k^3}\kappa(\boldsymbol{H}))\label{ineq4}
\end{align}
where above, inequality \eqref{ineq2} is a result of Lemma \ref{hamiddist}; \eqref{ineq3} follows from  \eqref{lem6eq2} and \eqref{lem6eq3}. Finally, inequality~\eqref{ineq4} is a result of Lemma \ref{firstlem}.
As inequality~\eqref{ineq4} is true for any $\epsilon>0$, taking the limit $\epsilon \downarrow 0+$, we arrive at~\eqref{lemma6-final}.
\end{proof}

\section{Proofs of Main Results}
\subsection{Proof of Theorem \ref{weakstrongthm}}
\begin{proof}
For any $i$, let us denote: 
$$j_i=\argmin_{j\in[k]}\|\boldsymbol{H}_{i,.}-\boldsymbol{H}^0_{j,.}\|_2.$$
We have
\begin{align*}
    \mathcal{L}(\boldsymbol{H}_0,\boldsymbol{H}) & = \sum_{i=1}^k\min_{j\in[k]} \|\boldsymbol{H}^0_{i,.}-\boldsymbol{H}_{j,.}\|_2^2 \\
    &\leq \sum_{i=1}^k \|\boldsymbol{H}^0_{i,.}-\boldsymbol{H}_{i,.}\|_2^2 \\
    & = \sum_{i=1}^k  \|\boldsymbol{H}^0_{i,.}-\boldsymbol{H}^0_{j_i,.}+\boldsymbol{H}^0_{j_i,.}-\boldsymbol{H}_{i,.}\|_2^2 \\
    & \leq 2 \sum_{i=1}^k  \|\boldsymbol{H}^0_{i,.}-\boldsymbol{H}^0_{j_i,.}\|_2^2+2 \sum_{i=1}^k\|\boldsymbol{H}^0_{j_i,.}-\boldsymbol{H}_{i,.}\|_2^2\\
    & \leq 2kb(\boldsymbol{H}_0)^2 +  2\mathcal{L}(\boldsymbol{H},\boldsymbol{H}_0),
\end{align*}
where, the last line follows from the definition of $b(\B{H})$ in \eqref{bh}.
\end{proof}
\subsection{Proof of Theorem \ref{noiseremove}}
\begin{proof}
For any given matrices $\boldsymbol{H}$ and $\boldsymbol{H}_0$, there exists a matrix $\boldsymbol{U}_1\in \{0,1\}^{k\times k}$ such that it has exactly one 1 in each row, such that 
\begin{equation}\label{defn-of-U1}
\mathcal{L}(\boldsymbol{H},\boldsymbol{H}_0)=\|\boldsymbol{H}-\boldsymbol{U}_1\boldsymbol{H}_0\|_F^2.
\end{equation}
In fact, for any $i\in[k]$, we have $\boldsymbol{U}^1_{i,j_i}=1$ where
$$j_i=\argmin_{j\in[k]}\|\boldsymbol{H}_{i,.}-\boldsymbol{H}^0_{j,.}\|_2.$$
Noting the noiseless data $\B{X}_0$ is given as $\B{X}_0=\B{W}_0\B{H}_0$ where $\B{W}_0\geq 0$ and $\B{W}_0\B{1}_k=\B{1}_m$,  one has
\begin{align}
    D(\boldsymbol{X}_0,\boldsymbol{H})^{1/2}&=\min_{\substack{\B{W}\geq 0\\ \boldsymbol{W1}_k=\boldsymbol{1}_m}}\|\boldsymbol{X}_0-\boldsymbol{W}\boldsymbol{H}\|_F\leq \|\boldsymbol{X}_0-\boldsymbol{W}_0\boldsymbol{H}\|_F \nonumber\\
    & =\|\boldsymbol{W}_0\boldsymbol{H}_0-\boldsymbol{W}_0(\boldsymbol{H}-\boldsymbol{U}_1\boldsymbol{H}_0+\boldsymbol{U}_1\boldsymbol{H}_0)\|_F\nonumber\\
    &= \|\boldsymbol{W}_0(\mathbb{I}_k-\boldsymbol{U}_1)\boldsymbol{H}_0 + \boldsymbol{W}_0(\boldsymbol{H}-\boldsymbol{U}_1\boldsymbol{H}_0)\|_F\nonumber\\
    & \leq \|\boldsymbol{W}_0(\mathbb{I}_k-\boldsymbol{U}_1)\boldsymbol{H}_0\|_F + \|\boldsymbol{W}_0(\boldsymbol{H}-\boldsymbol{U}_1\boldsymbol{H}_0)\|_F. \label{thm2ineq1}
\end{align}
Note that each row of $\boldsymbol{W}_0$ sums to 1 and is nonnegative. Therefore, for each row, $\|\boldsymbol{W}^{0}_{i,.}\|_2^2\leq 1$ and
\begin{equation}\label{wnorm}
    \|\boldsymbol{W}_0\|_F^2= \sum_{i=1}^m \|\boldsymbol{W}^0_{i,.}\|_2^2\leq m.
\end{equation}
Using~\eqref{thm2ineq1}, we have the following inequalities:
\begin{align*}
    D(\boldsymbol{X}_0,\boldsymbol{H})^{1/2}&\leq  \|\boldsymbol{W}_0(\mathbb{I}_k-\boldsymbol{U}_1)\boldsymbol{H}_0\|_F + \|\boldsymbol{W}_0(\boldsymbol{H}-\boldsymbol{U}_1\boldsymbol{H}_0)\|_F \\
    & \leq \sqrt{m}\|\boldsymbol{H}_0\|_F \|\mathbb{I}_k-\boldsymbol{U}_1\|_F +\sqrt{m}\|\boldsymbol{H}-\boldsymbol{U}_1\boldsymbol{H}_0\|_F\\
    & \leq  k\sqrt{m}\|\boldsymbol{H}_0\|_F + \sqrt{m}\mathcal{L}(\boldsymbol{H},\boldsymbol{H}_0)^{1/2}
\end{align*}
where the second inequality makes use of~\eqref{wnorm}; and 
the last inequality uses \eqref{defn-of-U1} and 
$\|\mathbb{I}_k-\boldsymbol{U}_1\|_F\leq k$ (since, $\mathbb{I}_k-\boldsymbol{U}_1$ is a matrix with every entry between -1 and 1).

Similar to $\boldsymbol{U}_1$ in the definition~\eqref{defn-of-U1}, we can consider $\boldsymbol{U}_2\in\{0,1\}^{k\times k}$ such that 
\begin{equation}\label{defn-u2}
\mathcal{L}(\boldsymbol{H}_0,\boldsymbol{H})=\|\boldsymbol{H}_0-\boldsymbol{U}_2\boldsymbol{H}\|_F^2.
\end{equation}
Since $\boldsymbol{U}_2\boldsymbol{1}_k=\boldsymbol{1}_k$; and using the fact that 
every row of $\B{W}_0$ sums to one, we have:
$\boldsymbol{W}_0\boldsymbol{U}_2\boldsymbol{1}_k=\boldsymbol{1}_m$ and
\begin{equation}\label{belong-simplex}
\B{W}_0\B{U}_{2} \in \{ \B{W}~:~\B{W} \geq 0, \B{W}\B{1}_k = \B{1}_m\}.
\end{equation}
Note that we have the following:
\begin{align*}
    D(\boldsymbol{X}_0,\boldsymbol{H})^{1/2}&=\min_{\substack{\B{W}\geq 0 \\ \boldsymbol{W1}_k=\boldsymbol{1}_m}}\|\boldsymbol{X}_0-\boldsymbol{W}\boldsymbol{H}\|_F\stackrel{(a)}{\leq} \|\boldsymbol{X}_0-\boldsymbol{W}_0\boldsymbol{U}_2\boldsymbol{H}\|_F \\
    & =\|\boldsymbol{W}_0\boldsymbol{H}_0-\boldsymbol{W}_0\boldsymbol{U}_2\boldsymbol{H}\|_F\\
    & \stackrel{(b)}{\leq} \sqrt{m}\|\boldsymbol{H}_0-\boldsymbol{U}_2\boldsymbol{H}\|_F\stackrel{(c)}{=}\sqrt{m}\mathcal{L}(\boldsymbol{H}_0,\boldsymbol{H})^{1/2},
\end{align*}
where, $(a)$ uses feasibility condition~\eqref{belong-simplex}, 
$(b)$ is due to \eqref{wnorm} and $(c)$ uses~\eqref{defn-u2}.
\end{proof}

\subsection{Proof of Theorem \ref{robustnessthm}}\label{app:sec:proofthm3}
Note that the constants in this theorem are listed below.
\begin{equation}\label{defn-of-all-cs}\begin{aligned}
c_1=& 4\sqrt{k^3}\kappa^2(\boldsymbol{H}_0)+(1+\sqrt{2})\sqrt{k^3}\\
c_2 =&4k^2\kappa(\boldsymbol{H}_0)+(1+\sqrt{2})\sqrt{k}(k+\sqrt{k^3})\\
c_3 =&2k^2\kappa(\boldsymbol{H}_0)+(1+\sqrt{2})k^2\\
c_4=& k + 2\sqrt{k^3}\kappa(\boldsymbol{H}_0)\\
c_5=& (k+\sqrt{k^3}) +2k^2\\
c_6=& \sqrt{k^3}+k^2 \\
c_7=&\frac{\sigma_{\min}(\boldsymbol{H}_0)}{6\sqrt{k}}\\
c_8=&7k\kappa(\boldsymbol{H}_0)+2(1+\sqrt{2})k^2\kappa(\boldsymbol{H}_0)\\
c_9 =& 7\kappa(\boldsymbol{H}_0)(k+\sqrt{k^3})+2(1+\sqrt{2})\sqrt{k^5}\\
c_{10} =&7\kappa(\boldsymbol{H}_0)\sqrt{k^3}+(1+\sqrt{2})\sqrt{k^5}.
\end{aligned}
\end{equation}
We first present Lemmas~\ref{thmtech1} and~\ref{thm3proof-helper1} useful for the proof of Theorem~\ref{robustnessthm}.

\begin{lem}\label{thmtech1}
Under the assumptions of Theorem \ref{robustnessthm}, $P_{\ell}(\B{H}_0)$ is feasible for problem \eqref{archl0noise}.
\end{lem}
\begin{proof}
By our model-setup, there is a nonnegative $\boldsymbol{W}_0$ with $\boldsymbol{W}_0\boldsymbol{1}_k=\boldsymbol{1}_m$ 
such that $\boldsymbol{X}_0=\boldsymbol{W}_0\boldsymbol{H}_0$. For $i\in [m]$, let
$$\boldsymbol{v}_i=\boldsymbol{W}^0_{i,.}P_{\ell}(\boldsymbol{H}_0)\in\text{Conv}(P_{\ell}(\boldsymbol{H}_0)).$$
Using the definition of $\B{v}_{i}$ above, we have for $i\in[m]$
\begin{align}
 D(\boldsymbol{X}^0_{i,.},P_{\ell}(\boldsymbol{H}_0))& = \min_{\boldsymbol{u}\in\text{Conv}(P_{\ell}(\boldsymbol{H_0}))}\|\boldsymbol{X}^0_{i,.}-\boldsymbol{u}\|_2^2\leq  \|\boldsymbol{X}^0_{i,.}-\boldsymbol{v}_i\|_2^2\nonumber\\
 &=\|\boldsymbol{X}^0_{i,.}-\boldsymbol{W}^0_{i,.}P_{\ell}(\boldsymbol{H}_0)\|_2^2.\label{thm3ineq1}
 \end{align}
As a result, for $i\in [m]$, we have:
\begin{align}
    D(\boldsymbol{X}^0_{i,.},P_{\ell}(\boldsymbol{H}_0))^{1/2} &\stackrel{(a)}{\leq}\|\boldsymbol{X}^0_{i,.}-\boldsymbol{W}^0_{i,.}P_{\ell}(\boldsymbol{H}_0)\|_2\nonumber\\
    &  = \|\boldsymbol{W}^0_{i,.}(\boldsymbol{H}_0 - P_{\ell}(\boldsymbol{H}_0))\|_2 \nonumber\\ 
    & = \|\boldsymbol{W}^0_{i,.} P_{\ell}^{\perp}(\boldsymbol{H}_0)\|_2 \leq \|\boldsymbol{W}^0_{i,.}\|_2\|P_{\ell}^{\perp}(\boldsymbol{H}_0)\|_F \label{tm3ineq78}
\end{align}
where $(a)$ is a result of \eqref{thm3ineq1}. By \eqref{wnorm} and \eqref{tm3ineq78}, we have:
\begin{equation} \label{thm3:proof:app:line-arbit1} 
    D(\boldsymbol{X}^0_{i,.},P_{\ell}(\boldsymbol{H}_0))^{1/2} \leq \|P_{\ell}^{\perp}(\boldsymbol{H}_0)\|_F = \beta.
    \end{equation}
In what follows, for notational convenience, we denote $\boldsymbol{H}= P_{\ell}(\boldsymbol{H}_0)$. Note that $\B{H}$ satisfies the sparsity constraint in~\eqref{archl0noise}. We have the following:
\begin{align}
    D(\boldsymbol{X}_{i,.},\boldsymbol{H})^{1/2} & \leq \tilde{D}(\boldsymbol{X}^0_{i,.}+\boldsymbol{Z}_{i,.},\boldsymbol{H})\leq  \tilde{D}(\boldsymbol{X}^0_{i,.},\boldsymbol{H}) +\|\boldsymbol{Z}_{i,.}\|_2 \nonumber\\
    & \leq D(\boldsymbol{X}^0_{i,.},\boldsymbol{H})^{1/2} + \underbrace{\max_{i} \|\boldsymbol{Z}_{i,.}\|_2}_{\delta} \nonumber \\
    & \leq \beta + \delta \label{toprop}.
\end{align}
where the first inequality is a result of Lemma \ref{firstlem} and the fact $\B{X}=\B{X}_0+\B{Z}$; the second inequality is a result of Lemma \ref{perturb} (we use~\eqref{app:lemma5-line2} with $m=1$); and the third inequality uses~\eqref{thm3:proof:app:line-arbit1}. Bound~\eqref{toprop} shows that $\boldsymbol{H}=P_{\ell}(\boldsymbol{H}_0)$ is a feasible solution for problem~(\ref{archl0noise}).
\end{proof}
\begin{lem}\label{thm3proof-helper1}
Under the assumptions of Theorem \ref{robustnessthm}, one has 
\begin{align}
    D(\hat{\boldsymbol{H}},\boldsymbol{H}_0)^{1/2} \leq& \sqrt{k}D(\boldsymbol{H}_0,\tilde{\boldsymbol{X}}_0)^{1/2}+\sqrt{k^3} \beta +
    (k+\sqrt{k^3})\delta \label{techlem1}\\
D(\boldsymbol{H}_0,\hat{\boldsymbol{H}})^{1/2}\leq& 2\sqrt{k^3}\kappa(\boldsymbol{H}_0)D(\boldsymbol{H}_0,\tilde{\boldsymbol{X}}_0)^{1/2} + 2k^2\delta + k^2\beta.  
\label{techlem2}
\end{align} 
\end{lem}
\begin{proof}
We first prove \eqref{techlem1}. To this end, note that $\text{Conv}(\boldsymbol{X}_0)\subseteq\text{Conv}(\boldsymbol{H}_0)$, so 
$$D(\hat{\boldsymbol{H}},\boldsymbol{H}_0)^{1/2}\leq D(\hat{\boldsymbol{H}},\boldsymbol{X}_0)^{1/2}\leq \tilde{D}(\hat{\boldsymbol{H}},\boldsymbol{X}_0)$$
where the second inequality is a result of Lemma \ref{firstlem}. In addition, as $\hat{\boldsymbol{H}}$ is the optimal solution of (\ref{archl0noise}) and $P_{\ell}(\boldsymbol{H}_0)$ is feasible for problem~\eqref{archl0noise} (by Lemma \ref{thmtech1}), we have $D(\hat{\boldsymbol{H}},\boldsymbol{X})\leq D(P_{\ell}(\boldsymbol{H}_0),\boldsymbol{X})$ and by Lemma \ref{comparison},
\begin{equation}\label{thm3ineq5}
  \tilde{D}(\hat{\boldsymbol{H}},\boldsymbol{X})\leq \sqrt{k}\tilde{D}(P_{\ell}(\boldsymbol{H}_0),\boldsymbol{X}).  
\end{equation}
Therefore, one can write  
\begin{align*}
  D(\hat{\boldsymbol{H}},\boldsymbol{H}_0)^{1/2}& \stackrel{(a)}{\leq} \tilde{D}(\hat{\boldsymbol{H}},\boldsymbol{X}_0) = \tilde{D}(\hat{\boldsymbol{H}},\boldsymbol{X}-\boldsymbol{Z}) \stackrel{(b)}{\leq} \tilde{D}(\hat{\boldsymbol{H}},\boldsymbol{X})+k\delta \\
  & \stackrel{(c)}{\leq}\sqrt{k}\tilde{D}(P_{\ell}(\boldsymbol{H}_0),\boldsymbol{X})+ k\delta=\sqrt{k}\tilde{D}(\boldsymbol{H}_0 -P^{\perp}_{\ell}(\boldsymbol{H}_0),\boldsymbol{X})+ k\delta\\
  & \stackrel{(d)}{\leq} \sqrt{k}\tilde{D}(\boldsymbol{H}_0,\boldsymbol{X})+ \sqrt{k^3} \max_{j\in[k]}\|P^{\perp}_{\ell}(\boldsymbol{H}_0)_{j,.}\|_2+k\delta \\
  & \leq \sqrt{k}\tilde{D}(\boldsymbol{H}_0,\boldsymbol{X})+ \sqrt{k^3} \|P^{\perp}_{\ell}(\boldsymbol{H}_0)\|_F+k\delta \\
  & \leq \sqrt{k}\tilde{D}(\boldsymbol{H}_0,\boldsymbol{X})+\sqrt{k^3} \beta + k\delta \\
  & \stackrel{(e)}{\leq} \sqrt{k}\tilde{D}(\boldsymbol{H}_0,\boldsymbol{X}_0)+\sqrt{k^3} \beta + (k+\sqrt{k^3})\delta  \\
  & \stackrel{(f)}{\leq} \sqrt{k}\tilde{D}(\boldsymbol{H}_0,\tilde{\boldsymbol{X}}_0)+\sqrt{k^3} \beta + (k+\sqrt{k^3})\delta
\end{align*} 
where $(a)$ is a result of Lemma \ref{firstlem}, $(b),(d),(e)$ are results of Lemma \ref{perturb}, $(c)$ is a result of \eqref{thm3ineq5}, and $(f)$ is true as $\sqrt{k}\tilde{D}(\boldsymbol{H}_0,\boldsymbol{X}_0)\leq \sqrt{k}\tilde{D}(\boldsymbol{H}_0,\tilde{\boldsymbol{X}}_0)$ because $\text{Conv}(\tilde{\boldsymbol{X}}_0)\subseteq\text{Conv}(\boldsymbol{X}_0)$. This establishes \eqref{techlem1}.\\
We now proceed to show~\eqref{techlem2}. To this end, using Lemma~\ref{split},
\begin{align}\label{needed1}
    D(\boldsymbol{H}_0,\hat{\boldsymbol{H}})^{1/2} & \leq \tilde{D}(\boldsymbol{H}_0,\hat{\boldsymbol{H}})\leq \tilde{\mathcal{L}}(\boldsymbol{H}_0,\tilde{\boldsymbol{X}}_0)+k\tilde{D}(\tilde{\boldsymbol{X}}_0,\hat{\boldsymbol{H}}).
\end{align}
Note that by Lemma \ref{specialcase}, 
\begin{equation}\label{needed2}
\tilde{\mathcal{L}}(\boldsymbol{H}_0,\tilde{\boldsymbol{X}}_0)\leq 2\sqrt{k^3}\kappa(\boldsymbol{H}_0)D(\boldsymbol{H}_0,\tilde{\boldsymbol{X}}_0)^{1/2}
\end{equation}
by the definition of $\tilde{\B{X}}$. Let $j_i\in[m]$ for $i\in[k]$ be such that $\tilde{\B{X}}^0_{i,.}=\B{X}^0_{j_i,.}$. Let $\tilde{\B{Z}}\in\R^{k\times n}$ be such that $\tilde{\B{Z}}_{i,.}={\B{Z}}_{j_i,.}$ for $i\in[k]$ and $\tilde{\B{X}}=\tilde{\B{X}}_0+\tilde{\B{Z}}$. As a result,
\begin{align}
   \tilde{D}(\tilde{\boldsymbol{X}}_0,\hat{\boldsymbol{H}}) &= \tilde{D}(\tilde{\boldsymbol{X}}-\tilde{\boldsymbol{Z}},\hat{\boldsymbol{H}}) \stackrel{(a)}{\leq} \tilde{D}(\tilde{\boldsymbol{X}},\hat{\boldsymbol{H}})+k\delta \nonumber \\
   & \stackrel{(b)}{\leq}\sqrt{k} D(\tilde{\boldsymbol{X}},\hat{\boldsymbol{H}})^{1/2}+k\delta \nonumber\\
   & \stackrel{}{=} \sqrt{k} \sqrt{\sum_{i=1}^k D(\tilde{\B{X}}_{i,.},\hat{\B{H}})}+k\delta \nonumber\\
   & \stackrel{}{=} \sqrt{k} \sqrt{\sum_{i=1}^k D({\B{X}}_{j_i,.},\hat{\B{H}})}+k\delta \nonumber\\
    & \stackrel{(c)}{\leq} \sqrt{k} \sqrt{\sum_{i=1}^k(\delta + \beta)^2}+k\delta \nonumber\\
   & \leq k(\delta + \beta ) +k\delta = 2k\delta + k\beta \label{thm3ineq7}
\end{align} 
where $(a)$ is a result of Lemma \ref{perturb}, $(b)$ is a result of Lemma \ref{firstlem} and $(c)$ is due to the constraint $D(\B{X}_{i,.},\B{H})^{1/2}\leq \delta+\beta$ in Problem~\eqref{archl0noise}.
Therefore, by \eqref{needed1}, \eqref{needed2} and \eqref{thm3ineq7},
$$D(\boldsymbol{H}_0,\hat{\boldsymbol{H}})^{1/2}\leq 2\sqrt{k^3}\kappa(\boldsymbol{H}_0)D(\boldsymbol{H}_0,\tilde{\boldsymbol{X}}_0)^{1/2} + 2k^2\delta + k^2\beta$$
which establishes~\eqref{techlem2}. 
\end{proof}
\begin{proof}[{\bf Proof of Theorem \ref{robustnessthm}}]
Part 1) If $\hat{\B{H}}$ has linearly independent rows, the desired result is achieved by substituting (\ref{techlem1}) and (\ref{techlem2}) into Lemma \ref{hamiddist} with $\B{A}=\B{H}_0$ and $\B{B}=\hat{\B{H}}$. If $\hat{\B{H}}$ does not have linearly independent rows, for $\epsilon>0$ there exists $\hat{\B{H}}_{\epsilon}$ with linearly independent rows such that $\|\hat{\B{H}}_{\epsilon}-\hat{\B{H}}\|_F\leq \epsilon$. Following a path similar to the proof of Lemma \ref{specialcase} and taking the limit $\epsilon\downarrow 0+$ the desired result is achieved. \\
Part 2) By summing \eqref{techlem1} and \eqref{techlem2}, 
\begin{align*}
    D(\hat{\boldsymbol{H}},\boldsymbol{H}_0)^{1/2} + D(\boldsymbol{H}_0,\hat{\boldsymbol{H}})^{1/2}& \leq c_4 D(\boldsymbol{H}_0,\tilde{\boldsymbol{X}}_0)^{1/2} + c_5 \max_{i\in[m]}\|\boldsymbol{Z}_{i,.}\|_2 + c_6\|P^{\perp}_{\ell}(\boldsymbol{H}_0)\|_F\\
    & \stackrel{(a)}{\leq}\frac{\sigma_{\min}(\boldsymbol{H}_0)}{6\sqrt{k}}
\end{align*}
where $(a)$ is a result of condition \eqref{condstr}. As a result, 
$$D(\hat{\boldsymbol{H}},\boldsymbol{H}_0)^{1/2}+D(\boldsymbol{H}_0,\hat{\boldsymbol{H}})^{1/2}\leq \frac{\sigma_{\min}(\boldsymbol{H}_0)}{6\sqrt{k}}.$$
Therefore, condition (B.42) of \citet{javadi2019nonnegative} holds and by Lemma B.3 of \citet{javadi2019nonnegative}, we have
\begin{equation}\label{kapparel}
\kappa(\hat{\boldsymbol{H}})\leq \frac{7}{2}\kappa(\boldsymbol{H}_0),    
\end{equation}
which shows $\hat{\B{H}}$ has linearly independent rows. The rest of the proof is achieved by substituting \eqref{kapparel}, (\ref{techlem1}) and (\ref{techlem2}) into Lemma \ref{hamiddist} with $\B{A}=\hat{\B{H}}$ and $\B{B}=\B{H}_0$.
\end{proof}

\subsection{Proof of Proposition~\ref{sepinexact}}
\begin{proof}
One has
\begin{align*}
  D(\hat{\boldsymbol{H}},\boldsymbol{H}_0)^{1/2}& \stackrel{(a)}{\leq}  D(\hat{\boldsymbol{H}},\boldsymbol{X}_0)^{1/2} \\ & ={D}(\hat{\boldsymbol{H}},\boldsymbol{X}-\boldsymbol{Z})^{1/2}\\
  & \stackrel{(b)}{\leq} \sqrt{2}{D}(\hat{\boldsymbol{H}},\boldsymbol{X})+\sqrt{2k}\delta \\
  & \stackrel{(c)}{\leq}\sqrt{2}{D}(P_{\ell}(\boldsymbol{H}_0),\boldsymbol{X})^{1/2}+ \sqrt{2k}\delta\\
  & =\sqrt{2}{D}(\boldsymbol{H}_0 -P^{\perp}_{\ell}(\boldsymbol{H}_0),\boldsymbol{X})^{1/2}+ \sqrt{2k}\delta\\
  & \stackrel{(d)}{\leq} 2{D}(\boldsymbol{H}_0,\boldsymbol{X})^{1/2}+ 2\sqrt{k} \max_{j\in[k]}\|P^{\perp}_{\ell}(\boldsymbol{H}_0)_{j,.}\|_2+\sqrt{2k}\delta \\
  & \stackrel{}{\leq}  2{D}(\boldsymbol{H}_0,\boldsymbol{X}_0+\B{Z})^{1/2}+2\sqrt{k} \|P^{\perp}_{\ell}(\boldsymbol{H}_0)\|_F+\sqrt{2k}\delta\\
  & \stackrel{(e)}{\leq} 2\sqrt{2}{D}(\boldsymbol{H}_0,\boldsymbol{X}_0)^{1/2}+ 2\sqrt{k} \|P^{\perp}_{\ell}(\boldsymbol{H}_0)\|_F+3\sqrt{2k}\delta\\
  & \stackrel{(f)}{=}2\sqrt{k} \|P^{\perp}_{\ell}(\boldsymbol{H}_0)\|_F+3\sqrt{2k}\delta
\end{align*}
where $(a)$ is true as $\text{Conv}(\B{X}_0)\subseteq \text{Conv}(\B{H}_0)$, $(b)$ is by Lemma~\ref{perturb2}, $(c)$ is by optimality of $\hat{\B{H}}$, $(d)$ and $(e)$ are by Lemma~\ref{perturb2} and $(f)$ is true as by separability, $D(\B{H}_0,\B{X}_0)=0$. Next, without loss of generality assume that for $i\in[k]$, we have that $\B{X}^{0}_{i,.}=\B{H}^{0}_{i,.}$ (note that each row of $\B{H}_0$ is also a row of $\B{X}_0$ by separability) or in other words, $\B{H}_0=\tilde{\B{X}}_0=\B{X}^0_{1:k,.}$. Then,
\begin{align}
    D(\boldsymbol{H}_0,\hat{\boldsymbol{H}})^{1/2} & = D(\tilde{\boldsymbol{X}}_0,\hat{\boldsymbol{H}})^{1/2} \nonumber \\
    & = D(\tilde{\boldsymbol{X}}-\tilde{\B{Z}},\hat{\boldsymbol{H}})^{1/2}\nonumber\\
    & \stackrel{(a)}{\leq } \sqrt{2}D(\tilde{\boldsymbol{X}},\hat{\boldsymbol{H}})^{1/2}+\sqrt{2k}\delta \nonumber \\
    & \stackrel{}{=} \sqrt{2} \sqrt{\sum_{i=1}^k D(\tilde{\B{X}}_{i,.},\hat{\B{H}})}+\sqrt{2k}\delta \nonumber\\
    & \stackrel{(b)}{\leq} \sqrt{2} \sqrt{\sum_{i=1}^k(\delta + \beta)^2}+\sqrt{2k}\delta \nonumber\\
    & = \sqrt{2k}\beta + 2\sqrt{2k}\delta
\end{align}
where $(a)$ is by Lemma~\ref{perturb2}, and $(b)$ is by feasibility of $\hat{\B{H}}$. Rest of the proof follows the proof of Theorem~\ref{robustnessthm}.
\end{proof}

\subsection{Proof of Proposition \ref{penalizedpro}}\label{sec:proof-of-prop1}
The constants in this proposition are listed below:
\begin{equation}
    \begin{aligned}
    c^1_{\lambda}&=\left(2\kappa(\boldsymbol{H}_0)\left[k\sqrt{m}\sqrt{m+\lambda k^2}+mk\right]+(1+\sqrt{2})\sqrt{k}\left[\sqrt{mk/\lambda + k^3}+k\right]\right)\\
    c^2_{\lambda}&=\left(7\kappa(\boldsymbol{H}_0)\left[\sqrt{mk/\lambda +k^3}+k\right]+(1+\sqrt{2})\sqrt{k}\left[k\sqrt{m}\sqrt{m+\lambda k^2}+mk\right]\right)\\
    c^3_{\lambda}&=\left[(1+m)k +k\sqrt{m}\sqrt{m+\lambda k^2 }+\sqrt{mk/\lambda +k^3}\right].
    \end{aligned}
\end{equation}

\begin{lem}\label{proptechlem}
Under the assumptions of Proposition \ref{penalizedpro}, one has
\begin{align}
        D(\hat{\boldsymbol{H}}_{\lambda},\boldsymbol{H}_0)^{1/2} & \leq \sqrt{k}\sqrt{\frac{m\delta^2}{\lambda}+k^2\delta^2} + k\delta,\label{propineq5} \\
        D(\boldsymbol{H}_0,\hat{\boldsymbol{H}}_{\lambda})^{1/2} & \leq  k\sqrt{m}\sqrt{m\delta^2+\lambda k^2\delta^2} + mk\delta \label{propineq7}.
\end{align}
\end{lem}
\begin{proof}
Recall that in this proposition, we assume $P_{\ell}(\boldsymbol{H}_0)=\boldsymbol{H}_0$ ($\beta=0$) and 
\begin{equation}\label{separab}
    D(\boldsymbol{H}_0,\tilde{\boldsymbol{X}}_0)=\tilde{D}(\boldsymbol{H}_0,{\boldsymbol{X}}_0)=0.
\end{equation}
From \eqref{toprop}, we have 
\begin{equation} \label{propineq1}
D(\boldsymbol{X},\boldsymbol{H}_0)=\sum_{i=1}^m D(\boldsymbol{X}_{i,.},\boldsymbol{H}_0)\leq m\delta^2.\end{equation}
In addition, we have
\begin{align}
   D(\boldsymbol{H}_0,\boldsymbol{X})^{1/2}\stackrel{(a)}{\leq} \tilde{D}(\boldsymbol{H}_0,\boldsymbol{X})  = \tilde{D}(\boldsymbol{H}_0,\boldsymbol{X}_0+\boldsymbol{Z}) \stackrel{(b)}{\leq} \tilde{D}(\boldsymbol{H}_0 ,\boldsymbol{X}_0) + k\delta  \stackrel{(c)}{=}  k\delta \label{propineq2}
\end{align}
where $(a)$ is due to Lemma \ref{firstlem}, $(b)$ is due to Lemma \ref{perturb} and $(c)$ is true because of \eqref{separab}. Therefore, from~\eqref{propineq2} we have:
$$D(\boldsymbol{H}_0,\boldsymbol{X})\leq [\tilde{D}(\boldsymbol{H}_0,\boldsymbol{X})]^2\leq k^2\delta^2. $$
Let $u=m\delta^2$ and $v=k^2 \delta^2$. Note that as $\hat{\boldsymbol{H}}_{\lambda}$ is the optimal solution of the penalized problem \eqref{archl0dual}, by \eqref{propineq1} and \eqref{propineq2} we have
\begin{equation}\label{propineq9}
   D(\boldsymbol{X},\hat{\boldsymbol{H}}_{\lambda}) + \lambda D(\hat{\boldsymbol{H}}_{\lambda},\boldsymbol{X})\leq u+\lambda v. 
\end{equation}

Note that we have the following:
\begin{align}
    D(\hat{\boldsymbol{H}}_{\lambda},\boldsymbol{H}_0)^{1/2} & \stackrel{(a)}{\leq} D(\hat{\boldsymbol{H}}_{\lambda},\boldsymbol{X}_0)^{1/2}\stackrel{(b)}{\leq}  \tilde{D}(\hat{\boldsymbol{H}}_{\lambda},\boldsymbol{X}_0)\nonumber\\
    & = \tilde{D}(\hat{\boldsymbol{H}}_{\lambda},\boldsymbol{X}-\boldsymbol{Z}) \stackrel{(c)}{\leq} \tilde{D}(\hat{\boldsymbol{H}}_{\lambda},\boldsymbol{X}) + k\delta \nonumber\\
    & \stackrel{(d)}{\leq} \sqrt{k} D(\hat{\boldsymbol{H}}_{\lambda},\boldsymbol{X})^{1/2} + k\delta 
    \stackrel{(e)}{\leq} \sqrt{k}\sqrt{\frac{u}{\lambda}+v} + k\delta,
\end{align}
where $(a)$ is because $\text{Conv}(\B{X}_0)\subseteq \text{Conv}(\B{H}_0)$, $(b),(d)$ are due to Lemma \ref{firstlem}, $(c)$ is due to Lemma \ref{perturb};
and in $(e)$ we use the observation $D(\hat{\boldsymbol{H}}_{\lambda},\boldsymbol{X})\leq {u}/{\lambda} + v$ (which follows from~\eqref{propineq9}). 
This proves \eqref{propineq5}.\\
We will now prove~\eqref{propineq7}. We obtain the following set of inequalities:
\begin{align}
    D(\boldsymbol{H}_0,\hat{\boldsymbol{H}}_{\lambda})^{1/2} & \stackrel{(a)}{\leq} \tilde{\mathcal{L}}(\boldsymbol{H}_0,\boldsymbol{X}_0)+k\tilde{D}(\boldsymbol{X}_0,\hat{\boldsymbol{H}}) \nonumber\\
    & \stackrel{(b)}{\leq} 2\sqrt{k^3}\kappa(\boldsymbol{H}_0)D(\boldsymbol{H}_0,\tilde{\boldsymbol{X}}_0)^{1/2} + k\tilde{D}(\boldsymbol{X}_0,\hat{\boldsymbol{H}})\nonumber\\
    & \stackrel{(c)}{=} k\tilde{D}(\boldsymbol{X}_0,\hat{\boldsymbol{H}}_{\lambda})\nonumber \\
    & \stackrel{(d)}{\leq} k\tilde{D}(\boldsymbol{X},\hat{\boldsymbol{H}}_{\lambda}) + mk\delta \nonumber\\
    & \stackrel{(e)}{\leq}  k\sqrt{m}D(\boldsymbol{X},\hat{\boldsymbol{H}}_{\lambda})^{1/2} + mk\delta \nonumber\\
    & \stackrel{(f)}{\leq}   k\sqrt{m}\sqrt{u+\lambda v} + mk\delta
\end{align}
where $(a)$ is a result of  (\ref{needed1}), $(b)$ is due to (\ref{needed2}), $(c)$ is due to \eqref{separab}, $(d)$ is due to Lemma \ref{perturb} with $\B{X}_0=\B{X}-\B{Z}$, $(e)$ is true by 
Lemma~\ref{firstlem}; and $(f)$ is a result of \eqref{propineq9}. This establishes~\eqref{propineq7}. 
\end{proof}
\begin{proof}[{\bf Proof of Proposition~\ref{penalizedpro}}]

Part 1) If $\hat{\B{H}}_{\lambda}$ has linearly independent rows, this part of the proposition is a direct result of \eqref{propineq5} and \eqref{propineq7} together with Lemma \ref{hamiddist} with $\B{A}=\B{H}_0$ and $\B{B}=\hat{\B{H}}_{\lambda}$. If $\hat{\B{H}}_{\lambda}$ does not have linearly independent rows, a perturbation argument similar to the proof of Theorem~\ref{robustnessthm} Part 1 suffices.\\
Part 2) Similar to the proof of Theorem \ref{robustnessthm}, condition \eqref{propcond} guarantees $\kappa(\hat{\B{H}}_{\lambda})\leq (7/2)\kappa(\B{H}_0)$. The rest of the proof follows from \eqref{propineq5} and \eqref{propineq7} together with Lemma \ref{hamiddist} with $\B{A}=\hat{\B{H}}_{\lambda}$ and $\B{B}=\B{H}_0$.
\end{proof}

\subsection{Proof of Theorem \ref{convthm}}
\begin{proof}
Part 1) The proof of convergence is based on Theorem 2 of \citet{xu2017globally}. Following \citet{xu2017globally}, we define the maximum and minimum of Lipschitz constants across three blocks $(\B{H},\B{W},\tilde{\B{W}})$ at iteration $j$ as 
\begin{align}
    L_j= \max\{L_1(\B{W}_j),L_2(\B{H}_j),L_3(\B{X})\}~~~\text{and}~~~
    \ell_j= \min\{L_1(\B{W}_j),L_2(\B{H}_j),L_3(\B{X})\}\nonumber
\end{align}
respectively. By substituting the values of $L_1(\B{W}_j),L_2(\B{H}_j),L_3(\B{X})$, 
\begin{align}
    L_j&=\max\{2(\|\boldsymbol{W}_j^T\boldsymbol{W}_j\|_2+\lambda),2\max\{\|\boldsymbol{H}_j\boldsymbol{H}_j^T\|_2,\varepsilon\},2\lambda\|\boldsymbol{X}\boldsymbol{X}^T\|_2\}\\
    \ell_j&=\min\{2(\|\boldsymbol{W}_j^T\boldsymbol{W}_j\|_2+\lambda),2\max\{\|\boldsymbol{H}_j\boldsymbol{H}_j^T\|_2,\varepsilon\},2\lambda\|\boldsymbol{X}\boldsymbol{X}^T\|_2\}.
\end{align}
As $\B{W}_j,\tilde{\B{W}}_j$ are simplex matrices and bounded, and considering the cost function $\Psi(.,.,.)$ is bounded from above, $\B{H}_j$ needs to be bounded. Consequently, $L_j$ is uniformly bounded from above. In addition, by the assumption $\lambda>0$, $\ell_j$ is uniformly bounded away from zero. As a result, the Lipschitz constants across three blocks are uniformly bounded from above and bounded away from zero from below---a condition of Theorem 2 of \citet{xu2017globally}. Other conditions of Theorem 2 of \citet{xu2017globally} are satisfied and therefore, this implies the convergence of Algorithm~\ref{algorithm1}.\\
Part 2) Let 
\begin{equation}
    \B{T}_j=\max\{0,\boldsymbol{H}_j-[1/L_1(\B{W}_j)](-{\boldsymbol{W}_j}^T[\boldsymbol{X}-\boldsymbol{W}_j\boldsymbol{H}_j]+\lambda[\boldsymbol{H}_j-\tilde{\boldsymbol{W}}_j\boldsymbol{X}])\}.
\end{equation}
Note that $\B{T}_j\to \B{T}$ where $\B{T}$ is defined in \eqref{Tcond}. In addition by the assumption of the second part of the theorem on $\B{T}$, $P_{\ell}(\B{T})$ is unique. First, we show that $P_{\ell}(\B{T}_j)\to P_{\ell}(\B{T})$. Let us consider two cases:
\begin{enumerate}
    \item $\|\boldsymbol{T}\|_0>\ell$: In this case, there exists $i^*\geq 1$ such that the support of $P_{\ell}(\boldsymbol{T}_i)$ and $P_{\ell}(\boldsymbol{T})$ are the same for $i\geq i^*$. Therefore, for $i\geq i^*$, for $(r,u)\in S(P_{\ell}(\boldsymbol{T}_i))=S(P_{\ell}(\boldsymbol{T}))$ [recall that $S(\B{T})$ is the support of $\B{T}$], 
    $$P_{\ell}(\boldsymbol{T}_i)_{r,u}=\boldsymbol{T}^i_{r,u}\to \boldsymbol{T}_{r,u}=P_{\ell}(\boldsymbol{T})_{r,u}$$
    and for $(r,u)\in S(P_{\ell}(\boldsymbol{T}_i))^c=S(P_{\ell}(\boldsymbol{T}))^c$,
    $$P_{\ell}(\boldsymbol{T}_i)_{r,u}=0=P_{\ell}(\boldsymbol{T})_{r,u}.$$
    \item $\|\boldsymbol{T}\|_0\leq \ell$: In this case, $S(\boldsymbol{T})=S(P_{\ell}(\boldsymbol{T}))$ and there exists $i^*\geq 1$ such that for $i\geq i^*$, $S(\boldsymbol{T})\subseteq S(P_{\ell}(\boldsymbol{T}_i))$. As a result, for $i\geq i^*$, for $(r,u)\in S(\boldsymbol{T})$,
     $$P_{\ell}(\boldsymbol{T}_i)_{r,u}=\boldsymbol{T}^i_{r,u}\to \boldsymbol{T}_{r,u}=P_{\ell}(\boldsymbol{T})_{r,u}$$
     and for $(r,u)\in S(\boldsymbol{T})^c$,
     $$P_{\ell}(\boldsymbol{T}_i)_{r,u}\in\{0,\boldsymbol{T}^i_{r,u}\}$$
     and as $\boldsymbol{T}^i_{r,u}\to \boldsymbol{T}_{r,u}=0$,
     $$P_{\ell}(\boldsymbol{T}_i)_{r,u}\to P_{\ell}(\boldsymbol{T})_{r,u}=0.$$
\end{enumerate}
In addition, note that for any bounded convex set $C\subseteq \mathbb{R}^{m}$, if $\boldsymbol{x}_i\to \boldsymbol{x}^*$, $P_C(\boldsymbol{x}_i)\to P_C(\boldsymbol{x}^*)$ where $P_C$ is the projection onto $C$. This along with the fact that $P_{\ell}(\B{T}_j)\to P_{\ell}(\B{T})$, is sufficient to show stationarity: 
\begin{align*}
   \boldsymbol{H}^* & =\lim_{j\to\infty} \boldsymbol{H}_{j+1}\\
   & \stackrel{(a)}{=}\lim_{j\to\infty}P_{\ell}\left(\B{T}_j\right) \\
   & \stackrel{(b)}{=} P_{\ell}\left(\B{T}\right) \\
   & = P_{\ell}(\max\{0,\boldsymbol{H}^*-[1/L_1(\B{W}^*)](-{\boldsymbol{W}^*}^T[\boldsymbol{X}-\boldsymbol{W}^*\boldsymbol{H}^*]+\lambda[\boldsymbol{H}^*-\tilde{\boldsymbol{W}}^*\boldsymbol{X}])\}) \\
   & = \argmin_{\substack{\boldsymbol{H}\geq 0 \\ \|\B{H}\|_0\leq \ell}} \left\Vert\boldsymbol{H}-\left(\boldsymbol{H}^*-\frac{1}{2L_1(\B{W}^*)}\nabla_{\boldsymbol{H}}\Psi(\boldsymbol{H}^*,\B{W}^*,\tilde{\B{W}}^*)\right)\right\Vert_F^2
\end{align*}
where $(a)$ is by the definition of the iterate $\B{H}_{j+1}$ and $(b)$ was proved above, showing stationarity for the block $\B{H}$. 
\end{proof}

\subsection{Proof of Proposition \ref{ridgepro}}
\begin{proof}
Let $\phi(\B{H},\tilde{\B{W}})=\|\B{H}-\tilde{\B{W}}\B{X}\|_F^2$ be the objective function of \eqref{infilambda}. Suppose $\B{H}^*,\tilde{\B{W}}^*$ are optimal solutions to problem \eqref{infilambda} and let $\B{H}'=\B{0}$. Suppose $j$ is such that $\|\B{X}_{j,.}\|_2=\min_{u\in[m]}\|\B{X}_{u,.}\|_2$ and $\B{e}_j\in\R^m$ is the vector with all coordinates equal to zero except coordinate $j$ equal to one. Let $\tilde{\B{W}}'=\B{1}_k\B{e}_j^T$. Hence, $\tilde{\B{W}}'\B{X}= \B{1}_k\B{X}_{j,.}$. Note that $\B{H}',\tilde{\B{W}}'$ are feasible for \eqref{infilambda}.\\ 
We prove the statement of this proposition by the method of contradiction. 
Suppose there exists $i_0\in [k]$ such that  
\begin{equation}\label{contrad}
  \|\boldsymbol{H}^*_{i_0,.}\|_2> \max_{u\in[m]}\|\boldsymbol{X}_{u,.}\|_2+\sqrt{k}\min_{u\in[m]} \|\boldsymbol{X}_{u,.}\|_2.  
\end{equation}
Note that for any $\boldsymbol{v}\in\text{Conv}(\boldsymbol{X})$, there exists $\alpha_1,\cdots,\alpha_m\geq 0$ such that they sum to one and $\boldsymbol{v}=\sum_{i=1}^m \alpha_i \boldsymbol{X}_{i,.}$. As a result,
\begin{align*}
    \|\boldsymbol{v}\|_2= \|\sum_{i=1}^m \alpha_i \boldsymbol{X}_{i,.}\|_2 \leq \sum_{i=1}^m \alpha_i \|\boldsymbol{X}_{i,.}\|_2
     \leq \sum_{i=1}^m\alpha_i \max_{u\in [m]}\|\boldsymbol{X}_{u,.}\|_2=\max_{u\in [m]}\|\boldsymbol{X}_{u,.}\|_2
\end{align*}
which when used with~\eqref{contrad} leads to:
\begin{equation}\label{vsmall}
    \|\boldsymbol{H}^*_{i_0,.}\|_2\geq \|\boldsymbol{v}\|_2.
\end{equation}
One has
\begin{align*}
    k\min_{u\in[m]}\|\B{X}_{u,.}\|_2^2 & \stackrel{(a)}{=} \phi(\B{H}',\tilde{\B{W}}') \\
    & \stackrel{(b)}{\geq} \phi(\B{H}^*,\tilde{\B{W}}^*) \\
    & \geq \|\boldsymbol{H}^*_{i_0,.}-\tilde{\B{W}}_{i_0,.}\B{X}\|_2^2 \\
    & \geq D(\B{H}^*_{i_0,.},\B{X})=\min_{\boldsymbol{v}\in\text{Conv}(\boldsymbol{X})} \|\boldsymbol{H}^*_{i_0,.}-\boldsymbol{v}\|_2^2 \\
     & \stackrel{(c)}{\geq} \min_{\boldsymbol{v}\in\text{Conv}(\boldsymbol{X})} |\|\boldsymbol{H}^*_{i_0,.}\|_2-\|\boldsymbol{v}\|_2|^2 \stackrel{(d)}{\geq} |\|\boldsymbol{H}^*_{i_0,.}\|_2-\max_{\boldsymbol{v}\in\text{Conv}(\boldsymbol{X})}\|\boldsymbol{v}\|_2|^2 \\
    & = (\|\boldsymbol{H}^*_{i_0,.}\|_2-\max_{u\in[m]}\|\boldsymbol{X}_{u,.}\|_2)^2\\
    & \stackrel{(e)}{>} k\min_{u\in[m]}\|\B{X}_{u,.}\|_2^2,
\end{align*}
where $(a)$ is true by definition of $\B{H}',\tilde{\B{W}}'$, $(b)$ is due to the optimality of $\B{H}^*,\tilde{\B{W}}^*$, $(c)$ is true as for any two vectors $\B{a},\B{b}$, $\|\B{a}-\B{b}\|_2^2 \geq (\|\B{a}\|_2-\|\B{b}\|_2)^2$, $(d)$ is due to \eqref{vsmall} and $(e)$ is because of \eqref{contrad}. This is a contradiction. Hence, for any $i\in[k]$, $$\|\B{H}^*_i\|_2\leq\max_{u\in[m]} \|\boldsymbol{X}_{u,.}\|_2+\sqrt{k}\min_{u\in[m]} \|\boldsymbol{X}_{u,.}\|_2.$$
\end{proof}

\subsection{Proof of Proposition \ref{dualprop}}
\begin{proof}
Part 1) The cost function of \eqref{ftomin} is jointly convex in $\B{H},\tilde{\B{W}},\B{Y}$ and the feasible set of \eqref{ftomin} is convex. Therefore, $F$ is a marginal minimization of a jointly convex function (w.r.t. $\B{H},\tilde{\B{W}}$) over a convex set and is convex (see Section 3.2.5 of \citet{boyd2004convex}). \\
Part 2) For the second part, note that we can rewrite~\eqref{ftomin} as:
\begin{align}\label{ftomin2}
    F(\B{Y})=\min_{\B{H},\tilde{\B{W}},\B{U}} \quad &   \|\B{H}-\B{U}\|_F^2 \\
     \text{s.t.} \quad & \boldsymbol{H} \geq 0,~~\tilde{\boldsymbol{W}} \geq 0,~~\tilde{\boldsymbol{W}}\boldsymbol{1}_m=\boldsymbol{1}_k \nonumber \\
     & \B{H}_{i,j}\leq \sqrt{b}\B{Y}_{i,j}\quad \forall (i,j)\in[k]\times[n]\nonumber\\
     & \B{U}=\tilde{\B{W}}\B{X}.\nonumber
\end{align}
We start by obtaining the dual problem of \eqref{ftomin2}. Note that by enhanced Slater's condition \citep{boyd2004convex}, strong duality holds for this problem. The Lagrangian of \eqref{ftomin2} can be written as
\begin{align}
L(\B{H},\tilde{\B{W}},\B{U},\B{M},\B{\Lambda},\B{\mu}) =& \|\B{H}-\B{U}\|_F^2 -\langle\B{M}_1,\B{H}\rangle-\langle\B{M}_2,\tilde{\B{W}}\rangle+
    \langle\B{M}_3,\B{U}-\tilde{\B{W}}\B{X}\rangle \nonumber \\ 
    & +\langle\B{\mu},\tilde{\B{W}}\B{1}_m-\B{1}_k\rangle
    +\langle\B{\Lambda},\B{H}-\sqrt{b}\B{Y}\rangle \nonumber\\
=&\left[\|\B{H}-\B{U}\|_F^2-\langle\B{M}_1-\B{\Lambda},\B{H}\rangle+\langle\B{M}_3,\B{U}\rangle\right] \nonumber\\
&          +\left[-\langle\B{M}_2,\tilde{\B{W}}\rangle-\langle\B{M}_3,\tilde{\B{W}}\B{X}\rangle+\langle\B{\mu},\tilde{\B{W}}\B{1}_m\rangle\right] \label{lag0} \\
 &         + \left[-\langle\B{\mu},\B{1}_k\rangle-\langle\B{\Lambda},\sqrt{b}\B{Y}\rangle\right], \nonumber
         \end{align}
where $\B{M}_1,\B{M}_2,\B{M}_3,\B{\mu},\B{\Lambda}$ are the corresponding Lagrangian variables. By considering the optimality conditions wrt $\B{H},\B{U},\tilde{\B{W}}$, we achieve
\begin{align}
    2(\B{H}-\B{U})& =\B{M}_1-\B{\Lambda},\label{lag1} \\
    2(\B{H}-\B{U})& =\B{M}_3,\label{lag2}\\
 \B{M}_2+\B{M}_3\B{X}^T&=\B{\mu}\B{1}_m^T.\label{lag3}
\end{align}
Using \eqref{lag1}, \eqref{lag2}, \eqref{lag3} in~\eqref{lag0}, we get the dual of \eqref{ftomin2}:
\begin{align}\label{dualftomin}
   F(\B{Y})= \max_{\B{M}_1,\B{M}_2,\B{\Lambda}\geq0,\B{M}_3,\B{\mu}} \quad &   -\frac{1}{4}\|\B{M}_3\|_F^2 -\langle\B{\mu},\B{1}_k\rangle-\langle\B{\Lambda},\sqrt{b}\B{Y}\rangle \\
     \text{s.t.} \quad & \B{M}_2+\B{M}_3\B{X}^T=\B{\mu}\B{1}_m^T\nonumber \\
     & \B{M}_1-\B{\Lambda}=\B{M}_3. \nonumber
\end{align}
At optimality, note that for $(i,j)\in [k]\times [n]$,  $\B{\Lambda}_{i,j}=-\B{M}_{i,j}^3$ if $\B{M}_{i,j}^3<0$ and otherwise $\B{\Lambda}_{i,j}=0$ as $\Lambda_{i,j},\B{Y}_{i,j}\geq 0$ and the cost function is higher if $\Lambda_{i,j}$ is smaller. 
By using Danskin's Theorem~\citep{bertsekas1997nonlinear},
if $\B\Lambda$ is the optimal solution to~\eqref{dualftomin}, 
$-\sqrt{b}\B{\Lambda}$ is a subgradient of $F$. 
In addition, at optimality, based on KKT conditions, $\B{M}_3=2(\B{H}-\B{U})$ by \eqref{lag2} and $\B{U}=\tilde{\B{W}}\B{X}$ by feasibility, completing the proof.
\end{proof}

\subsection{Proof of Proposition~\ref{OGprop}}
\begin{proof}
Note that based on the upper/lower bounds returned by the outer approximation algorithm, for any $\B{H}\in\R^{k\times n}_{\geq 0}$ such that $\|\B{H}\|_0\leq \ell$,
$$D({\B{H}},\B{X})\geq \text{LB}~~~\text{and}~~~ D({\B{H}}^*_{\infty},\B{X})=\text{UB}.$$
As a result,
$$\frac{D({\B{H}}^*_{\infty},\B{X})-D(\B{H},\B{X})}{D({\B{H}}^*_{\infty},\B{X})}\leq \frac{\text{UB}-\text{LB}}{\text{UB}}=\text{OG}$$
or
\begin{equation}\label{ogineq}
  D(\B{H},\B{X})  \geq (1-\text{OG})D({\B{H}}^*_{\infty},\B{X}).
\end{equation}
Thus,
\begin{align}
    D(\B{X},\B{H}^*_{\lambda}) + \lambda D(\B{H}^*_{\lambda},\B{X})& \stackrel{(a)}{\geq} D(\B{X},\hat{\B{H}}_{\lambda}) + \lambda D(\hat{\B{H}}_{\lambda},\B{X}) \nonumber\\
    & \geq \lambda D(\hat{\B{H}}_{\lambda},\B{X}) \nonumber\\
    & \stackrel{(b)}{\geq} \lambda(1-\text{OG})D({\B{H}}^*_{\infty},\B{X})\label{prop-init-1}
\end{align}
where $(a)$ is due to optimality of $\hat{\B{H}}_{\lambda}$ and $(b)$ is due to~\eqref{ogineq}. Moreover, as Algorithm~\ref{algorithm1} is a descent algorithm initialized with $\B{H}_{\infty}^*$,
\begin{equation}
   D(\B{X},\B{H}^*_{\lambda}) + \lambda D(\B{H}^*_{\lambda},\B{X}) \leq D(\B{X},\B{H}^*_{\infty}) + \lambda D(\B{H}^*_{\infty},\B{X}).\label{prop-init-2}
\end{equation}
Hence,
\begin{align}
    \Phi_{\lambda}(\hat{\B{H}}_{\lambda}) - \Phi_{\lambda}(\B{H}^*_{\lambda}) &\stackrel{(a)}{\geq} -D(\B{X},\B{H}^*_{\lambda})-\lambda D(\B{H}^*_{\lambda},\B{X})+\lambda(1-\text{OG})D({\B{H}}^*_{\infty},\B{X})\nonumber\\
    & \stackrel{(b)}{\geq} -D(\B{X},\B{H}^*_{\infty})-\lambda(\text{OG})D({\B{H}}^*_{\infty},\B{X})\label{prop-init-3}
\end{align}
where $(a)$ is a result of~\eqref{prop-init-1} and $(b)$ is a result of~\eqref{prop-init-2}. Additionally, by~\eqref{ogineq},
\begin{equation}
    \lambda(1-\text{OG})D({\B{H}}^*_{\infty},\B{X}) \leq \lambda D(\B{H}^*_{\lambda},\B{X})\leq D(\B{X},\B{H}^*_{\lambda}) + \lambda D(\B{H}^*_{\lambda},\B{X})
\end{equation}
so by~\eqref{prop-init-3},
\begin{align}
    \frac{\Phi_{\lambda}(\hat{\B{H}}_{\lambda}) - \Phi_{\lambda}(\B{H}^*_{\lambda}) }{\Phi_{\lambda}(\B{H}^*_{\lambda})}\geq \frac{-D(\B{X},\B{H}^*_{\infty})-\lambda(\text{OG})D({\B{H}}^*_{\infty},\B{X})}{\lambda(1-\text{OG})D({\B{H}}^*_{\infty},\B{X})}.
\end{align}
\end{proof}

\subsection{Proof of Proposition \ref{loothm}}
\begin{proof}
Similar to proof of Theorem~\ref{noiseremove}, we can consider the matrix $\boldsymbol{U}_1\in \{0,1\}^{k\times k}$ with exactly one $1$ in each row, such that  $\mathcal{L}(\boldsymbol{H}^{-\mathcal{S}},\boldsymbol{H}_0)=\|\boldsymbol{H}^{-\mathcal{S}}-\boldsymbol{U}_1\boldsymbol{H}_0\|_F^2$. By the definition of $d(\boldsymbol{X})$, for any $i\in\mathcal{S}$ there exists $j_i\in[m]\setminus \mathcal{S}$ such that $\|\boldsymbol{X}_{i,.}-\boldsymbol{X}_{j_i,.}\|_2\leq d(\boldsymbol{X})$. Let $\bar{\B{X}}_{0}\in\R^{|\mathcal{S}|\times n}$ be such that for $i\in\mathcal{S}$, $\bar{\B{X}}^{0}_{i,.}=\B{X}^0_{j_i,.}$. Moreover, let $\bar{\B{W}}_0\in \R^{|\mathcal{S}|\times k}$ be such that for $i\in\mathcal{S}$, $\bar{\B{W}}^{0}_{i,.}=\B{W}^0_{j_i,.}$. This implies $\bar{\B{X}}_0=\bar{\B{W}}_0\B{H}_0$. We define $\bar{\B{X}}$ similar to $\bar{\B{X}}_0$. We also let $\bar{\B{Z}}=\bar{\B{X}}-\bar{\B{X}}_0$. One has
\begin{align*}
    D(\boldsymbol{X}_{\mathcal{S}},\boldsymbol{H}^{-\mathcal{S}})^{1/2}&=\min_{\substack{\B{W}\geq 0\\ \boldsymbol{W1}=\B{1}}}\|\B{X}_{\mathcal{S}}-\boldsymbol{W}\boldsymbol{H}^{-\mathcal{S}}\|_F\\
    & \leq \|\B{X}_{\mathcal{S}}-\bar{\boldsymbol{W}}_0\boldsymbol{H}^{-\mathcal{S}}\|_F \\
    & = \|\B{X}_{\mathcal{S}}+\bar{\B{X}}-\bar{\B{X}}-\bar{\boldsymbol{W}}_0\boldsymbol{H}^{-\mathcal{S}}\|_F\\
        & \stackrel{(a)}{\leq} \|\B{X}_{\mathcal{S}}-\bar{\B{X}}\|_F+\|\bar{\B{X}}_0+\bar{\B{Z}}-\bar{\boldsymbol{W}}_0(\boldsymbol{H}^{-\mathcal{S}}-\boldsymbol{U}_1\boldsymbol{H}_0+\boldsymbol{U}_1\boldsymbol{H}_0)\|_F\\
    & \stackrel{(b)}{\leq} \sqrt{|\mathcal{S}|}d(\boldsymbol{X}) + \|\bar{\boldsymbol{W}}_0(\mathbb{I}_k-\boldsymbol{U}_1)\boldsymbol{H}_0\|_F + \|\bar{\boldsymbol{W}}_0(\boldsymbol{H}^{-\mathcal{S}}-\boldsymbol{U}_1\boldsymbol{H}_0)\|_F + \|\bar{\B{Z}}\|_F\\
     & \stackrel{(c)}{\leq} \sqrt{|\mathcal{S}|}\left[d(\boldsymbol{X}) + \|\mathbb{I}_k-\boldsymbol{U}_1\|_F\|\boldsymbol{H}_0\|_F + \|\boldsymbol{H}^{-\mathcal{S}}-\boldsymbol{U}_1\boldsymbol{H}_0\|_F+ \max_{j\in[m]}\|\boldsymbol{Z}_{j,.}\|_2\right]\\
     & \leq \sqrt{|\mathcal{S}|}\left[d(\boldsymbol{X}) + k\|\boldsymbol{H}_0\|_F + \max_{j\in[m]}\|\boldsymbol{Z}_{j,.}\|_2 + \mathcal{L}(\boldsymbol{H}^{-\mathcal{S}},\boldsymbol{H}_0)^{1/2}\right]
\end{align*}
where $(a)$ is due to $\bar{\B{Z}}=\bar{\B{X}}-\bar{\B{X}}_0$, $(b)$ is due to the definition of $\bar{\B{X}}$ and considering $\bar{\B{X}}$ has $|\mathcal{S}|$-many rows, and $(c)$ is true because $\bar{\B{W}}_0$ has $|\mathcal{S}|$-many rows that each has $\ell_2$ norm bounded by 1 (see~\eqref{wnorm}). \\
We now turn our attention to the second part of the proof. Similar to the first part of the proof, there exists $\B{U}_2\in\{0,1\}^{k\times k}$ such that $\mathcal{L}(\boldsymbol{H}_0,\boldsymbol{H}^{-\mathcal{S}})=\|\boldsymbol{U}_2\boldsymbol{H}^{-\mathcal{S}}-\boldsymbol{H}_0\|_F^2$. Moreover, similar to the proof of Theorem~\ref{noiseremove},
$$\bar{\B{W}}_0\B{U}_2\in\{\B{W}:\B{W}\geq 0, \B{W1}_k=\B{1}_{|\mathcal{S}|}\}.$$
Consequently, 
\begin{align*}
     D(\boldsymbol{X}_{\mathcal{S}},\boldsymbol{H}^{-\mathcal{S}})^{1/2}&=\min_{\substack{\B{W}\geq 0\\ \boldsymbol{W1}=\B{1}}}\|\B{X}_{\mathcal{S}}-\boldsymbol{W}\boldsymbol{H}^{-\mathcal{S}}\|_F\\
    & \leq \|\B{X}_{\mathcal{S}}+\bar{\B{X}}-\bar{\B{X}}-\bar{\boldsymbol{W}}_0\B{U}_2\boldsymbol{H}^{-\mathcal{S}}\|_F \\
    & \leq \|\B{X}_{\mathcal{S}}-\bar{\B{X}}\|_F +\|\bar{\B{X}}-\bar{\boldsymbol{W}}_0\B{U}_2\boldsymbol{H}^{-\mathcal{S}}\|_F \\
    &\leq \|\B{X}_{\mathcal{S}}-\bar{\B{X}}\|_F +\|\bar{\B{X}}_0-\bar{\boldsymbol{W}}_0\B{U}_2\boldsymbol{H}^{-\mathcal{S}}\|_F +\|\B{Z}\|_F \\
    & \leq \sqrt{|\mathcal{S}|}\left[d(\boldsymbol{X}) +\max_{j\in m}\|\B{Z}_{j,.}\|_2\right] + \|\bar{\B{W}}_0\|_F\|\B{H}_0-\B{U}_2\B{H}^{-\mathcal{S}}\|_F \\
    & \leq \sqrt{|\mathcal{S}|}\left[d(\boldsymbol{X}) +\max_{j\in m}\|\B{Z}_{j,.}\|_2+\mathcal{L}(\boldsymbol{H}_0,\boldsymbol{H}^{-\mathcal{S}})^{1/2}\right]. 
\end{align*}

\end{proof}

\section{Additional Numerical Experiments}
\subsection{Synthetic Data}\label{syntheticadd}
\textbf{Robustness and Initialization:} In Section~\ref{initvsobj} we showed that using the MIP-based initialization and local search helps to improve the objective function achieved by the solution. Here, we show that our proposed framework improves the robustness of the resulting model compared to the baseline. To this end, we consider the same setup as in Section~\ref{initvsobj} and report the strong robustness quantity $\mathcal{L}(\hat{\B{H}},\B{H}_0)$ (where $\B{H}_0$ is the underlying archetypes matrix and $\hat{\B{H}}$ is the estimated one) in Table~\ref{robusobj}. As it can be seen, our framework of SAA (and SAA+LS) not only lead to lower objective values (as discussed in Section~\ref{initvsobj}), but also leads to more robust solutions compared to a simple initialization in most case, specially when the noise is low. 
\begin{table}[h!]

\centering
\begin{tabular}{ |c|c|c|c|c| } 
\hline
 $\sigma_z$ & Method &  $\ell/nk=0.5$ & $\ell/nk=0.65$ & $\ell/nk=0.8$ \\
\hline
\multirow{4}{*}{0.01}& Zero &  0.7010 & 0.5865 & 0.4914 \\
& Projections & 0.1398 & 0.0191 & 0.0021  \\
& SAA &  0.0736 & 0.0082 & 0.0002  \\
&SAA+LS &    \textbf{0.0735} & \textbf{0.0081} & \textbf{0.0002 } \\
\hline

\multirow{4}{*}{0.1}&Zero &   0.7038 & 0.5877 & 0.4952  \\
& Projections & 0.1728 & 0.0520 & 0.0298   \\
& SAA &   0.1055 & \textbf{0.0370} & 0.0280\\
&SAA+LS  &  \textbf{0.1054} & \textbf{0.0370} & \textbf{0.0279} \\
\hline

\multirow{4}{*}{0.5}&Zero & 0.7927 & 0.6691 & 0.5849  \\
& Projections &  0.7956 & 0.6306 & \textbf{0.5680} \\
& SAA &   0.7771 & 0.6147& 0.5780  \\
&SAA+LS &     \textbf{0.7769} & \textbf{0.6146}& 0.5780   \\
\hline
\end{tabular}
\caption{\small  Comparison of zero initialization, Projections, SAA and SAA+LS in Appendix \ref{syntheticadd}. }
\label{robusobj}
\end{table}
 
\subsection{Cancer Gene Expression Example}\label{geneapp}

 {\textbf{The effect of initialization:} To show the effectiveness of our proposed initialization and local search framework, we consider the same setup as discussed in Section~\ref{genedata}. However, we initialize our method as discussed in Section~\ref{initvsobj} (we consider three cases Zero, SAA and SAA+LS as discussed). We compare the objective value achieved by these three approaches, in addition to their clustering performance captured by purity and entropy, in Table~\ref{cluster-app}. As it can be seen and in agreement with our observations from the synthetic examples, our initialization and local search frameworks improve the quality of the solution significantly compared to the baseline (Zero). 
} \\

\begin{table}[h!]
\centering
\begin{tabular}{ |c|c|c|c| } 
\hline
  & Zero & SAA & SAA+LS   \\
\hline
Objective & $1.7\times 10^{12}$  &  $5.2\times 10^{11}$ &  $5.1\times 10^{11}$  \\
Purity &  0.364 & 0.660 &    0.662\\
Entropy & 0.636 & 0.361 &  0.360 \\
\hline

\end{tabular}
\caption{\small   Performance of different initializations for the gene expression data set in Appendix~\ref{geneapp} }
\label{cluster-app}
\end{table}

\noindent\textbf{The effect of varying $\lambda$:} To show how the choice of $\lambda$ affects the performance of our method, we consider the following experimental setup. We choose 9 values of $\lambda$ between 0.1 and 20 and run our method on the data. We randomly select 40 data points as the validation set and use other 158 data points to estimate $\B{H}$. The left panel in Figure~\ref{fig:app} shows the validation loss (on the held-out data) for different values of $\lambda$. We also use the whole data set to investigate the clustering performance of our method (captured by purity) for different values of $\lambda$. The right panel in Figure~\ref{fig:app} shows the purity of the resulting estimator for different values of $\lambda$. As it can be seen, the best clustering performance is achieved for $\lambda\approx 0.5-1.5$. The lowest validation loss is also achieved for the value of $\lambda=1$. Overall, it can be seen choosing $\lambda$ too small or too large can negatively affect the performance of the model. However, our validation method can find a good value of $\lambda$.    \\

\begin{figure*}[t!]
     \centering
\begin{tabular}{lclc}
&  Validation Loss & & Purity \\
     \rotatebox{90}{~~~~~~~~~~~~~~~~~~~$\text{V}_{\lambda}$}& \includegraphics[width=0.4\linewidth,trim =0.5cm 0cm 1cm 0cm, clip = true]{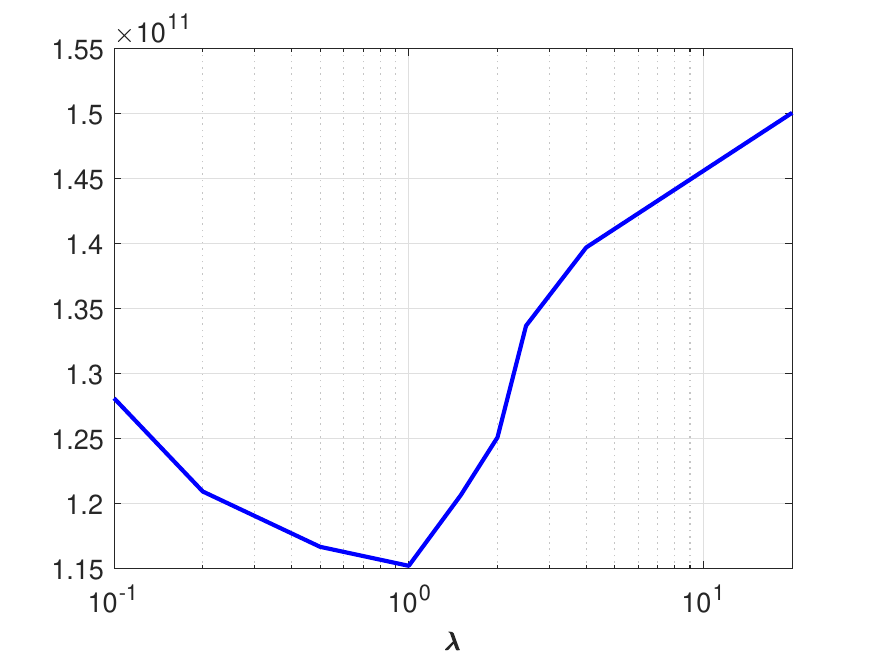}& 
 \rotatebox{90}{~~~~~~~~~~~~~~~~Purity}&        \includegraphics[width=0.4\linewidth,trim =0.5cm 0cm 1cm 0cm, clip = true]{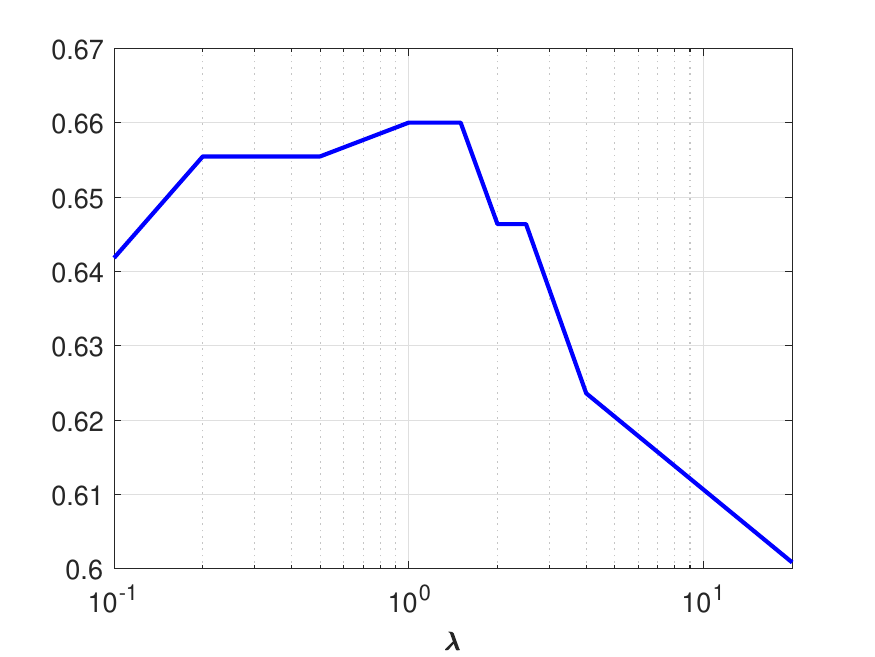}\\ 
\end{tabular}
        \caption{\small Comparison of validation loss and purity of the clustering result for gene expression data in Appendix~\ref{geneapp}}
        \label{fig:app}
\end{figure*}

\noindent \textbf{The effect of varying sparsity:} Finally, we show our algorithm outperforms other methods with different values of sparsity. We vary the value of the target sparsity $\ell$ between $0.35nk$ and $0.95nk$ and tune other methods to result in solutions with the target sparsity (as discussed in Section~\ref{synthnumundrob}). Our simulation setup is the same as in Section~\ref{genedata}. The purity results for these cases are shown in Figure~\ref{fig:app2}. We observe that our method outperforms other sparse NMF methods with different values of sparsity and even performs better than two non-sparse methods.

\begin{figure*}[t!]
     \centering
 \begin{tabular}{lc}
 & Purity \\
     \rotatebox{90}{~~~~~~~~~~~~~~~~~Purity}& \includegraphics[width=0.4\linewidth,trim =0.5cm 0cm 1cm 0cm, clip = true]{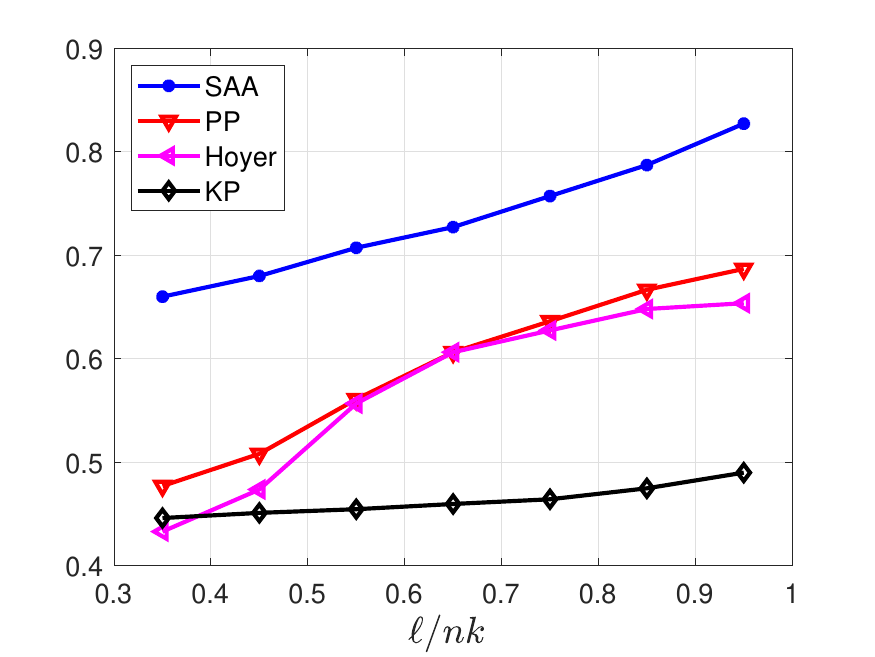}\\ 
\end{tabular}
        \caption{\small  Comparison of effect of sparsity for gene expression data in Appendix~\ref{geneapp}}
        \label{fig:app2}
\end{figure*}

\subsection{Hyperspectral Unmixing Example}\label{app:hyperdata}
In this section, we present additional results for the experiments conducted in Section~\ref{hyperdata}. First, we show how the results from SAA appear to be more interpretable. To this end, consider Figure~\ref{figpines2}. 

\begin{figure}[t!]

     \centering
\begin{tabular}{cccc}
     Ground Truth & Kmeans & AA/Chen & SAA \\
 \includegraphics[width=0.23\linewidth,trim = 1.5cm 1cm 1.5cm 1cm, clip = true]{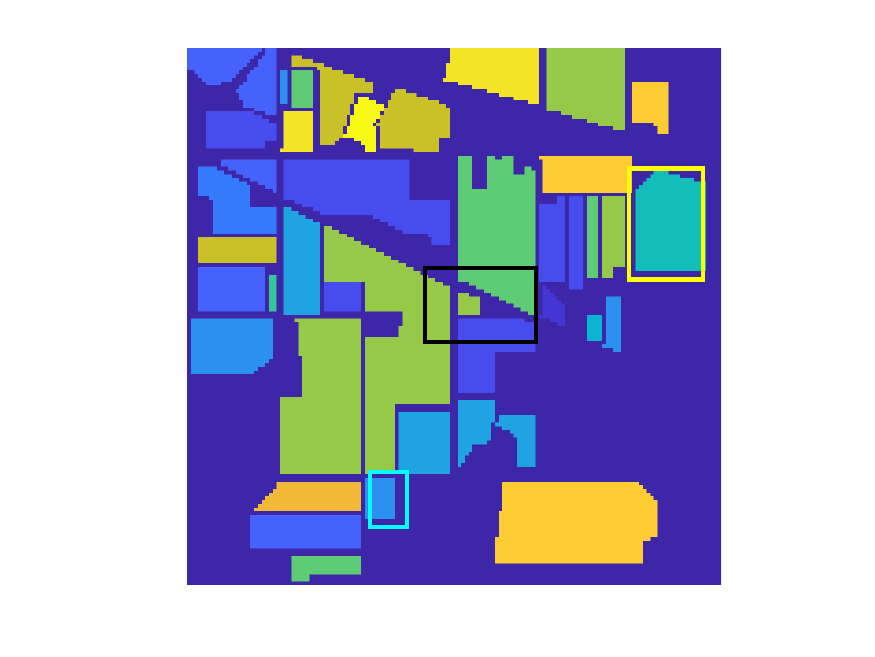} & \includegraphics[width=0.23\linewidth,trim = 1.5cm 1cm 1.5cm 1cm, clip = true]{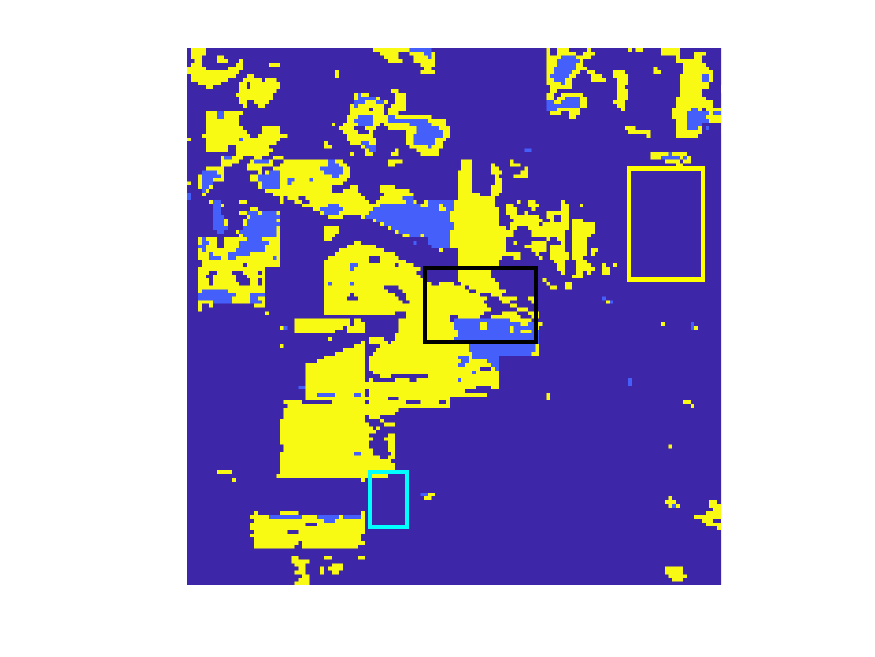} &  \includegraphics[width=0.23\linewidth,trim = 1.5cm 1cm 1.5cm 1cm, clip = true]{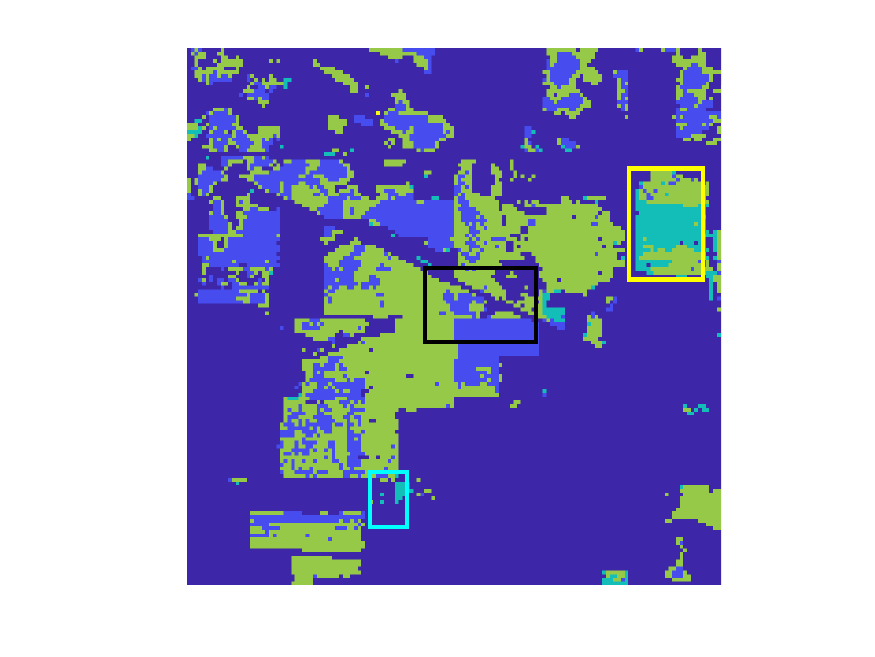}&  \includegraphics[width=0.23\linewidth,trim = 1.5cm 1cm 1.5cm 1cm, clip = true]{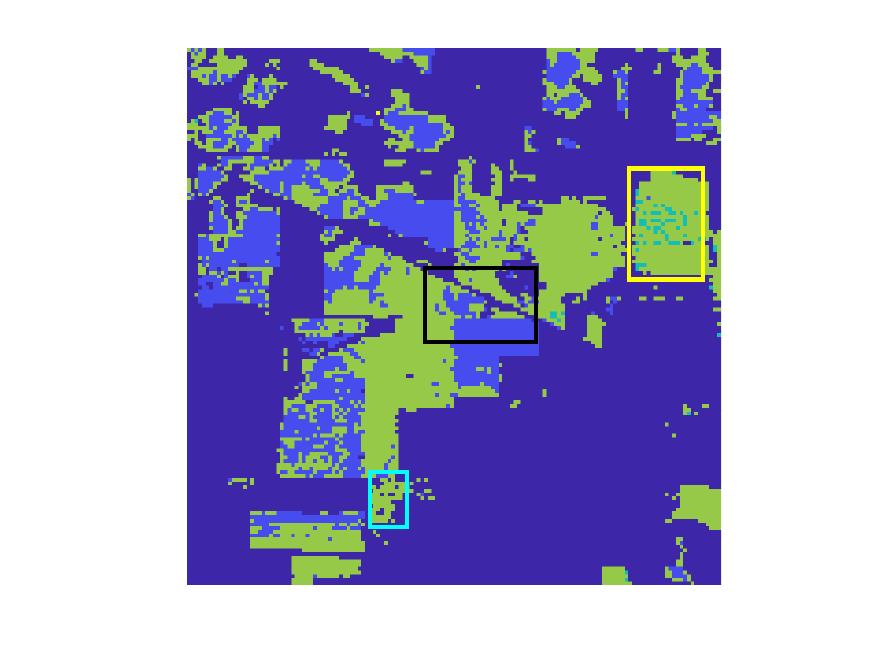}\\
         \end{tabular}
     \caption{\small Visual comparison for  hyperspectral unmixing data in Appendix~\ref{app:hyperdata}  }    \label{figpines2}
\end{figure}
The black box in figures above shows a part of the image where AA/SAA/Chen are able to properly identify a road. Moreover, the yellow and cyan boxes highlight two areas of the image, where SAA appears to have better recovery. As can be seen in Figure~\ref{figpines2}, AA/Chen miss most of the area in the cyan box, where SAA has a better recovery. Similarly, the area in the yellow box is mostly put into a single cluster by SAA, as expected from the ground truth, but this is not the case for AA/Chen (we note the color coding between the estimated clusters and the ground truth might not be unique, as it is not possible to map clusters from different examples to each other perfectly). \\
\textbf{The effect of varying sparsity:} Similar to the previous section, we explore the effect of sparsity on the quality of the solution for different methods. The results for this case are shown in Figure~\ref{fig:apppinesell}. As it can be seen, as long as the solution is not too sparse, SAA is able to deliver a good performance. This shows that sparse NMF, and particularly SAA, can obtain solutions that perform well, while compressing the solution by enforcing sparsity.
\begin{figure*}[t!]
     \centering
   \begin{tabular}{lc}
 & Purity \\
     \rotatebox{90}{~~~~~~~~~~~~~~~~~Purity}& \includegraphics[width=0.4\linewidth,trim =0.5cm 0cm 1cm 0cm, clip = true]{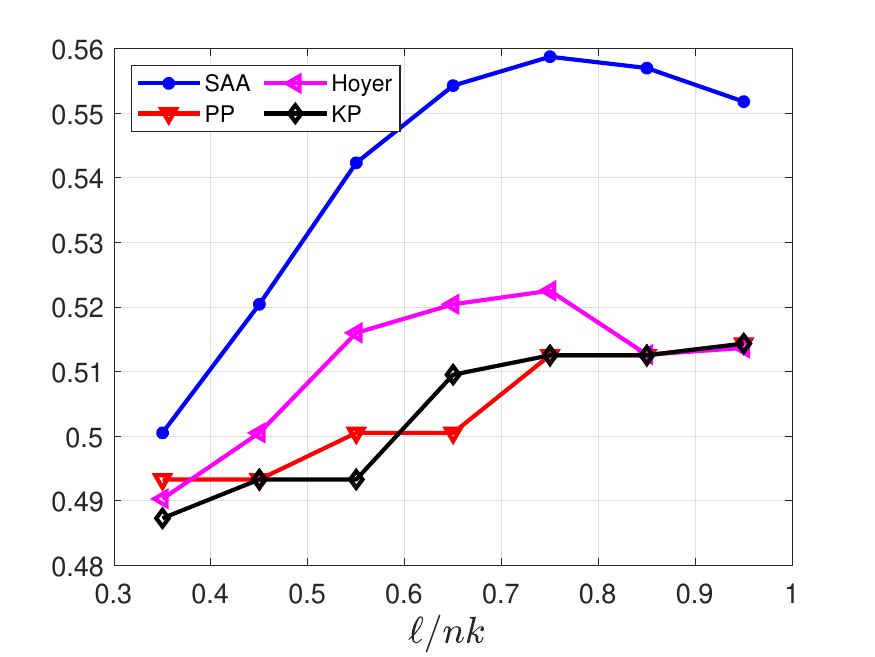}\\ 
\end{tabular}
        \caption{\small Comparison of effect of sparsity for hyperspectral unmixing data in Appendix~\ref{app:hyperdata}}
        \label{fig:apppinesell}
\end{figure*}

\end{document}